\documentclass[11pt,oneside]{article}

\usepackage{config_mjh_arxiv}

\usepackage[numbers]{natbib}
\usepackage{amsmath}
\usepackage{amsthm}
\theoremstyle{definition} \newtheorem{defn}{Definition}
\theoremstyle{plain} 
\theoremstyle{plain} \newtheorem{thm}[defn]{Theorem}
\theoremstyle{plain} \newtheorem{lem}[defn]{Lemma}
\theoremstyle{plain} 
\theoremstyle{remark} \newtheorem{rmk}[defn]{Remark}
\theoremstyle{remark} 

\usepackage{enumitem} %
\makeatletter
\def\namedlabel#1#2{\begingroup
    #2%
    \def\@currentlabel{#2}%
    \phantomsection\label{#1}\endgroup
}
\makeatother

\usepackage[pdfauthor={Matthew J. Holland},%
colorlinks=true,%
linkcolor=blue,%
citecolor=blue]{hyperref}
\hypersetup{pdftitle={Holland (2020): Better scalability under potentially heavy-tailed feedback}} %

\begin{document}

\title{\textbf{Better scalability under potentially\\heavy-tailed feedback}}
\author{
  Matthew J.~Holland\thanks{Please direct correspondence to \texttt{matthew-h@ar.sanken.osaka-u.ac.jp}.}\\
  Osaka University
}
\date{} %

\maketitle

\begin{abstract}
We study scalable alternatives to robust gradient descent (RGD) techniques that can be used when the losses and/or gradients can be heavy-tailed, though this will be unknown to the learner. The core technique is simple: instead of trying to robustly aggregate gradients at each step, which is costly and leads to sub-optimal dimension dependence in risk bounds, we instead focus computational effort on robustly choosing (or newly constructing) a strong candidate based on a collection of cheap stochastic sub-processes which can be run in parallel. The exact selection process depends on the convexity of the underlying objective, but in all cases, our selection technique amounts to a robust form of boosting the confidence of weak learners. In addition to formal guarantees, we also provide empirical analysis of robustness to perturbations to experimental conditions, under both sub-Gaussian and heavy-tailed data, along with applications to a variety of benchmark datasets. The overall take-away is an extensible procedure that is simple to implement, trivial to parallelize, which keeps the formal merits of RGD methods but scales much better to large learning problems.
\end{abstract}

\tableofcontents

\section{Introduction}\label{sec:intro}

Obtaining ``strong contracts'' for the performance of machine learning algorithms is difficult.\footnote{This notion was described lucidly in a keynote lecture by L.~Bottou \citep{bottou2015oral}.} Classical tasks in computer science, such as sorting integers or simple matrix operations, come with lucid worst-case guarantees. With enough resources, the job \emph{can} be done correctly and completely. In machine learning, things are less simple. Since we only have access to highly impoverished information regarding the phenomena or goal of interest, inevitably the learning task is uncertain, and any meaningful performance guarantee can only be stated with some degree of confidence, typically over the random draw of the data used for training. This uncertainty is reflected in the standard formulation of machine learning tasks as ``risk minimization'' problems \citep{vapnik1982EDBED,haussler1992a}. Here we consider risk minimization over some set of candidates $\WW \subseteq \RR^{d}$, where the \term{risk} of $w$ is defined as the expected loss to be incurred by $w$, namely
\begin{align*}
\risk_{\ddist}(w) \defeq \exx_{\ddist}\loss(w;Z) = \int_{\ZZ} \loss(w;z) \, \ddist(\dif z), \qquad w \in \WW.
\end{align*}
Here we have a loss function $\loss: \WW \times \ZZ \to \RR_{+}$, and random data $Z \sim \ddist$ takes values in a set $\ZZ$. At most, any learning algorithm will have access to $n$ data points sampled from $\ddist$, denoted $Z_{1},\ldots,Z_{n}$. Write $(Z_{1},\ldots,Z_{n}) \mapsto \what_{n}$ to denote the output of an arbitrary learning algorithm. The usual starting point for analyzing algorithm performance is the \term{estimation error} $\risk_{\ddist}(\what_{n})-\risk_{\ddist}^{\ast}$, where $\risk_{\ddist}^{\ast} \defeq \inf\{ \risk_{\ddist}(w): w \in \WW \}$, or more precisely, the distribution of this error. Since we never know much about the underlying data-generating process, typically all we can assume is that $\ddist$ belongs to some class $\PP$ of probability measures on $\ZZ$, and typical guarantees are given in the form of
\begin{align*}
\prr\left\{ \risk_{\ddist}(\what_{n})-\risk_{\ddist}^{\ast} > \varepsilon\left(n,\delta,\ddist,\WW\right)  \right\} \leq \delta, \qquad \forall \, \ddist \in \PP.
\end{align*}
Flipping the inequalities around, this says that the algorithm generating $\what_{n}$ enjoys $\varepsilon$-good performance with $(1-\delta)$-high confidence over the draw of the sample, where the error level depends on the sample size $n$, the desired confidence level $\delta$, the underlying data distribution $\ddist$, and any constraints encoded in $\WW$, not to mention the nature of loss $\loss$. Ideally, we would like formal guarantees to align as closely as possible with performance observed in the real world by machine learning practitioners. With this in mind, the following properties are important to consider.
\begin{enumerate}
\item \textbf{Transparency:} can we actually compute the output $\what_{n}$ that we study in theory?

\item \textbf{Strength:} what form do bounds on $\varepsilon(n,\delta,\ddist,\WW)$ take? How rich is the class $\PP$?

\item \textbf{Scalability:} how do computational costs scale with the above-mentioned factors?
\end{enumerate}
Balancing these three points is critical to developing guarantees for \emph{algorithms that will actually be used in practice}. If strong assumptions are made on the data distribution (i.e., $\PP$ is a ``small'' class), then most of the data any practitioner runs into will fall out of scope. If the error bound grows too quickly with $1/\delta$ or shrinks too slowly with $n$, then either the guarantees are vacuous, or the procedure is truly sub-optimal. If the procedure outputting $\what_{n}$ cannot be implemented, then we run into a gap between what we code, and what we study formally.

\paragraph{Our problem setting}

In this work, we consider the setup of \term{potentially heavy-tailed} data. More concretely, all the learner can know is that for some $m < \infty$,
\begin{align}\label{eqn:potentially_heavy}
\PP \subseteq \left\{ \ddist: \sup_{w \in \WW} \exx_{\ddist}|\loss(w;Z)|^{m} < \infty \right\},
\end{align}
where typically $m = 2$. Thus, it is unknown whether the losses (or partial derivatives, etc.) are congenial in a sub-Gaussian sense (where (\ref{eqn:potentially_heavy}) holds for all $m$), or heavy-tailed in the sense that all higher-order moments could be infinite or undefined. The goal then comes down to obtaining the strongest possible guarantees for a tractable learning algorithm, given (\ref{eqn:potentially_heavy}). We next review the related technical literature, and give an overview of our contributions.

\section{Context and contributions}\label{sec:context_contribs}

With the three properties of transparency, strength, and scalability highlighted in the previous section in mind, for the next few paragraphs we look at the characteristics of several important families of learning algorithms.

\paragraph{ERM: can scale well, but lacks robustness}

Classical learning theory is primarily centered around \term{empirical risk minimization} (ERM) \citep{vapnik1998SLT,anthony1999NNTheory}, and studies the statistical properties that hold for \emph{any} minimizer of the empirical risk, namely
\begin{align}
\what_{n} \in \argmin_{w \in \WW} \, \frac{1}{n} \sum_{i=1}^{n} \loss(w;Z_{i}).
\end{align}
Clearly, this leaves all algorithmic aspects of the problem totally abstract, and opens up the possibility for substantial gaps between the performance of ``good'' and ``bad'' ERM solutions, as studied by \citet{feldman2017a}. Furthermore, the empirical mean is sensitive to outliers, and formally speaking is sub-optimal in the sense that it cannot achieve sub-Gaussian error bounds under potentially heavy tails, while other practical procedures can; see \citet{catoni2012a} and \citet{devroye2016a} for comprehensive studies. Roughly speaking, the empirical mean cannot guarantee better error bounds than those which scale as $\Omega(1/\sqrt{\delta n})$. In the context of machine learning, these statistical limitations provide an important implication about the \emph{feedback} available to any learner which tries to directly minimize the empirical risk, effectively lower-bounding the statistical error (in contrast to the optimization error) incurred by any such procedure.

\paragraph{Robust risk minimizers: strong in theory, but lacking transparency}

To deal with the statistical weaknesses of ERM, it is natural to consider algorithms based on more ``robust'' feedback, i.e., minimizers of estimators of the risk which provide stronger guarantees than the empirical mean under potentially heavy tails. A seminal example of this is the work of \citet{brownlees2015a}, who consider learning algorithms of the form
\begin{align}
\what_{n} \in \argmin_{w \in \WW} \widehat{\risk}(w), \text{ where } \sum_{i=1}^{n}\psi\left(\frac{\widehat{\risk}(w)-\loss(w;Z_{i})}{s}\right) = 0.
\end{align}
That is, they consider minimizers of an M-estimator of the risk, using influence function $\psi$ of the type studied by \citet{catoni2012a}. Under weak moment bounds like (\ref{eqn:potentially_heavy}), their minimizers enjoy $\bigO(1/\sqrt{n})$ rates with $\bigO(\log(\delta^{-1}))$ dependence on the confidence. This provides a significant improvement in terms of the strength of guarantees compared with ERM, but unfortunately the issue of transparency remains. Like ERM, the algorithmic side of the problem is left abstract here, and in general may even be a much more difficult computational task. Observe that the new objective $\widehat{\risk}(\cdot)$ cannot be written in closed form, and even if $\loss(\cdot;Z)$ is convex, this $\widehat{\risk}(\cdot)$ need not preserve such convexity. Direct optimization is hard, but verifying improvement in the function value is easy, and some researchers have utilized a guess-and-check strategy to make the approach viable in practice \citep{holland2017a}. However, these methods are inexact, and due to optimization error, strictly speaking the algorithm being run does not enjoy the full guarantees given by \citet{brownlees2015a} for the ideal case.

\paragraph{Robust gradient descent: transparent, but scales poorly}

To try and address the issue of transparency without sacrificing the strength of formal guarantees, several new families of algorithms have been designed in the past few years to tackle the potentially heavy-tailed setting using a tractable procedure. Such algorithms may naturally be called \term{robust gradient descent} (RGD), the naming being appropriate since their core updates all take the form
\begin{align}\label{eqn:rgd_defn}
\what_{t+1} = \what_{t} - \alpha_{t} \, \widehat{G}_{n}(\what_{t}),
\end{align}
and they are ``robust'' in the sense that the estimate $\widehat{G}_{n}(w) \approx \nabla\risk_{\ddist}(w)$ has deviations with near-optimal confidence intervals under potentially heavy-tailed data (i.e., both the loss and partial gradients are potentially heavy-tailed). These strategies typically use biased estimators of the mean, in sharp contrast with traditional first-order oracle assumptions for stochastic gradient descent. Since we will be interested in making a direct comparison with these procedures in this work, we give a more detailed introduction to representative RGD methods in the next three paragraphs.

The most common strategy is a ``median of means'' approach, studied first by \citet{chen2017arxiv,chen2017a} under a distributed learning setup with outliers, and subsequently by \citet{prasad2018a} in the context of potentially heavy-tailed data.\footnote{In the context of distributed machine learning under outliers, there have been many variations on how to do the ``aggregation'' of gradients in a robust fashion \citep{blanchard2018a,xie2018a,elmhamdi2019a,rajput2020a}. These amount to different special cases of doing the RGD update (\ref{eqn:rgd_defn}), within a different problem setting.} The basic strategy is simple: at each step, set the update direction $\widehat{G}_{n}$ to be the median of means estimator of the risk gradient. The sample is partitioned $\{1,\ldots,n\} = \II_{1} \cup \cdots \cup \II_{k}$ into $k$ subsets with $\lfloor n/k \rfloor$ elements each. From each subset $\II_{j}$, one computes an empirical mean of gradients, and then merge these $k$ independent estimates by taking their geometric median; we denote this $\geomed$ (see Algorithm \ref{algo:merge_geomed} for details). More explicitly, we have
\begin{align}
\label{eqn:rgd_mom}
\widehat{G}_{n}(w) & = \geomed\left[ \{\widehat{G}^{(1)}(w),\ldots,\widehat{G}^{(k)}(w) \}; \|\cdot\| \right],\\
\nonumber
\widehat{G}^{(j)}(w) & = \frac{1}{|\II_{j}|}\sum_{i \in \II_{j}} \nabla\loss(w;Z_{i}), \quad j = 1,\ldots,k.
\end{align}
We refer to (\ref{eqn:rgd_defn}) implemented using (\ref{eqn:rgd_mom}) as \term{RGD-by-MoM}.

Another approach, first introduced by \citet{holland2017arxiv,holland2019c}, does not use a sample-splitting mechanism, but rather takes a dimension-wise robustification strategy, updating using
\begin{align}
\label{eqn:rgd_cat}
\widehat{G}_{n}(w) & = \left(\widehat{\theta}_{1}(w),\ldots,\widehat{\theta}_{d}(w)\right),\\
\nonumber
\widehat{\theta}_{j}(w) & = \argmin_{\theta \in \RR} \sum_{i=1}^{n}\rho\left( \frac{\nabla_{j}\loss(w;Z_{i})-\theta}{s} \right), \quad j = 1,\ldots,d.
\end{align}
Here $s$ is a scaling parameter, and $\rho$ is a convex, even function that is approximately quadratic near zero, but grows linearly in the limit of $\pm \infty$. Since M-estimation is the key sub-routine, we refer to (\ref{eqn:rgd_defn}) implemented by (\ref{eqn:rgd_cat}) as \term{RGD-M}. Note that both RGD-by-MoM and RGD-M enjoy error bounds with optimal dependence on $n$ and $1/\delta$ under potentially heavy-tailed data, with the significant merit that the computational procedures are transparent and easy to implement as-is. Unfortunately, instead of a simple one-dimensional robust mean estimate as in \citet{brownlees2015a}, all RGD methods rely on sub-routines that work in $d$-dimensions. This makes the procedures much more expensive computationally for ``big'' learning tasks, and leads to an undesirable dependence on the ambient dimension $d$ in the statistical guarantees as well, hampering their overall scalability. Furthermore, the analysis of these procedures requires strong convexity of the underlying risk; as we shall discuss shortly, error bounds which depend on strong convexity parameters tend to grow without bound and become vacuous in high dimensions, which further damages the scalability of the traditional RGD methodology.

An alternative approach to doing robust gradient descent comes from \citet{lecue2018a}, who utilize a neat mixture of the core ideas of robust risk minimizers and the robust gradient descent procedures. Taking a $k$-partition of the data as just described, the update direction is set as
\begin{align}
\label{eqn:rgd_lecue}
\widehat{G}_{n}(w) & = \frac{1}{|\II_{\star}|} \sum_{i \in \II_{\star}} \nabla\loss(w;Z_{i}),\\
\nonumber
\widehat{\loss}_{\star}(w) & \defeq \med\left\{\widehat{\loss}_{1}(w),\ldots,\widehat{\loss}_{k}(w)\right\},\\
\nonumber
\widehat{\loss}_{j}(w) & \defeq \frac{1}{|\II_{j}|} \sum_{i \in \II_{j}} \loss(w;Z_{i}), \quad j = 1,\ldots,k.
\end{align}
That is, a robust estimator of the risk (median-of-means) is used to determine which subset to use for computing an empirical estimate of the risk gradient. This approach is meant to approximately achieve the minimization of $\widehat{\loss}_{\star}(w)$, the median-of-means risk estimator; we refer to it as \term{MoM-by-GD}. \citet{lecue2018a} prove strong statistical guarantees for the \emph{true} minimizer of $\widehat{\loss}_{\star}(\cdot)$, specialized to the binary classification task, without requiring bounded inputs; this is an appealing technical improvement over what can be guaranteed using the machinery of \citet{brownlees2015a}. Under some technical conditions (their Sec.~4.2), they prove convergence of their algorithm, which scales well computationally since the only high-dimensional operation required is summation over a small subset. Unfortunately, since the rate of convergence is unclear, there may exist a substantial gap between the statistical error guaranteed for the median-of-means risk minimizer and the output of this procedure.

As a final important point, the formal performance analysis of standard RGD methods (e.g., \citep{chen2017a,prasad2018a,holland2019d}) relies heavily upon special properties of gradient-based minimizers when the objective function (here, the risk $\risk_{\ddist}$) is \emph{strongly} convex. As the number of parameters to be determined grows, it is typical that the strong convexity parameter shrinks rapidly, making existing error bounds essentially vacuous in the high-dimensional setting. Furthermore, if one attempts to perform the analysis without assuming strong convexity, the resulting excess risk bounds become extremely sensitive to the number of iterations and misspecified hyperparameters (see section \ref{sec:comparison_nonsc} for more details). From the standpoint of trying to develop scalable, general-purpose learning algorithms with guarantees, this reliance on strong convexity severely hampers the effective scalability.

\paragraph{Our contributions}

To briefly summarize the issues highlighted above, ERM and robust risk minimizers leave the potential for a severe gap between what is guaranteed on paper and what is done in practice. On the other hand, both formal guarantees and computational requirements for RGD methods do not scale well to high-dimensional learning tasks.\footnote{We compare and discuss RGD error bounds under strong convexity in Table \ref{table:sc_compare} and section \ref{sec:comparison_sc}, and without strong convexity in section \ref{sec:comparison_nonsc}.} The key issues are clear: even when working with the Euclidean geometry, a quick glance at the proofs in the cited works on RGD shows that direct dependence on $d$ in the error bounds is unavoidable. Furthermore, the extra computational overhead, scaling at least linearly in $d$, must be incurred at \emph{every} step in the iterative procedure, which severely hurts scalability.

Considering these issues, here we investigate a different algorithmic approach of equal generality, with the goal of achieving as-good or better dependence on $n$, $d$, and $1/\delta$, under the same assumptions, and in provably less time for larger problems. The core technique uses distance-based rules to select among independent weak candidates when convexity is available, and a robust confidence-boosting sub-routine in the more general case. To make our analysis sufficiently concrete, the weak learners are implemented using inexpensive stochastic gradient-based updates, which can be easily run in parallel. Our main contributions:
\begin{itemize}
\item We analyze a general-purpose learning procedure (Algorithm \ref{algo:DandC_SGD}), and obtain sharp high-probability error bounds (Theorems \ref{thm:sc_lip_smooth_SGDlast_roboost} and \ref{thm:sc_smooth_SGDlast_roboost}), which improve upon the poor dimension dependence of existing RGD routines under strongly convex risks when both the losses and gradients can be heavy-tailed (comparison in section \ref{sec:comparison_sc}).

\item We further extend this analysis to the case without strong convexity, with a concrete computational procedure (Algorithm \ref{algo:DandC_valid}) for which sharp risk bounds (Theorem \ref{thm:smooth_SGDave_roboost}) are obtained (details in section \ref{sec:comparison_nonsc}).

\item The procedures outlined in Algorithms \ref{algo:DandC_SGD} and \ref{algo:DandC_valid} are simple to implement and amenable to distributed computation, providing superior computational scalability over existing serial RGD procedures, without sacrificing theoretical guarantees.

\item Empirically, we study the efficiency and robustness of the proposed approach against key benchmarks (section \ref{sec:empirical}). This is done using both tightly controlled simulations and a variety of real-world benchmark datasets. We verify a substantial improvement in the cost-performance tradeoff, robustness to heavy-tailed data, and performance that scales well to higher dimensions.
\end{itemize}
Taken together, our results suggest a promising class of learning algorithms for general-purpose risk minimization, which achieve an appealing balance between transparency, strength and scalability.

\section{Theoretical analysis}\label{sec:theory}

\subsection{Preliminaries}

\paragraph{Notation}
For any positive integer $k$, write $[k] \defeq \{1,\ldots,k\}$. For any index $\II \subseteq [n]$, write $\Z_{\II} \defeq (Z_{i})_{i \in \II}$, defined analogously for independent copy $\Z_{\II}^{\prime}$. To keep the notation simple, in the special case of $\II=[n]$, we write $\Z_{n} \defeq \Z_{[n]} = (Z_{1},\ldots,Z_{n})$. We shall use $\prr$ as a generic symbol to denote computing probability; in most cases this will be the product measure induced by the sample $\Z_{n}$ or $\Z_{n}^{\prime}$. For any function $f: \RR^{d} \to \RR$, denote by $\partial f(u)$ the sub-differential of $f$ evaluated at $u$. Variance of the loss is denoted by $\sigma_{\ddist}^{2}(w) \defeq \vaa_{\ddist}\loss(w;Z) = \exx_{\ddist}(\loss(w;Z)-\risk_{\ddist}(w))^{2}$ for each $w \in \WW$. When we write $I\{\texttt{event}\}$, this refers to the indicator function which returns $1$ when $\texttt{event}$ is true, and $0$ otherwise. We use a white square ($\qedsymbol$) to mark the end of proofs, and a black square ($\blacksquare$) to mark the end of remarks, indicating the resumption of the main text.

\paragraph{Technical conditions}
The two key running assumptions that we make are related to independence and convexity. First, we assume that all the observed data are independent, i.e., the random variables $Z_{i}$ and $Z_{i}^{\prime}$ taken over all $i \in [n]$ are independent copies of $Z \sim \ddist$. Second, for each $z \in \ZZ$, we assume the map $w \mapsto \loss(w;z)$ is a real-valued convex function over $\RR^{d}$, and that the parameter set $\WW \subseteq \RR^{d}$ is non-empty, convex, and compact. All results derived in the next sub-section will be for an arbitrary choice of $\ddist \in \PP$, where $\PP$ satisfies (\ref{eqn:potentially_heavy}) with $m=2$. Finally, to make formal statements technically simpler, we assume that $\risk_{\ddist}(\cdot)$ achieves its minimum on the interior of $\WW$.

Several special properties of the underlying feedback provided to the learner will be of interest at different points in this paper; for convenience, we organize them all here.
\begin{itemize}
\item[\namedlabel{asmp:lip_loss}{A1$^{\ast}$}.] \textbf{$\parasm_{0}$-Lipschitz loss:} there exists $0 < \parasm_{0} < \infty$ such that for all $z \in \ZZ$ and all $u,v \in \WW$, we have $|\loss(u;z) - \loss(v;z)| \leq \parasm_{0}\|u-v\|$.

\item[\namedlabel{asmp:sc_risk}{A2}.] \textbf{$\parasc$-SC risk:} There exists a $0 < \parasc < \infty$ such that the map $w \mapsto \risk(w)$ is $\parasc$-strongly convex on $\WW$ in norm $\|\cdot\|$.

\item[\namedlabel{asmp:sm_risk}{A3}.] \textbf{$\parasm_{1}$-smooth risk:} The map $w \mapsto \risk_{\ddist}(w)$ is differentiable over $\RR^{d}$, and $\parasm_{1}$-smooth on $\WW$ in norm $\|\cdot\|$ with $0 < \parasm_{1} < \infty$.

\item[\namedlabel{asmp:sm_loss}{A3$^{\ast}$}.] \textbf{$\parasm_{1}$-smooth loss:} The map $w \mapsto \loss(w;z)$ is differentiable over $\RR^{d}$, and $\parasm_{1}$-smooth on $\WW$ in norm $\|\cdot\|$ with $0 < \parasm_{1} < \infty$, for all $z \in \ZZ$.
\end{itemize}
Definitions of \term{strong convexity} and \term{smoothness} are given in the technical appendix. To keep the statement of technical results as succinct as possible, we shall refer directly to these conditions as required. For example, the assumption of a $\parasc$-strongly convex risk will be written $\text{\ref{asmp:sc_risk}}(\parasc)$, the assumption of losses with a $\parasm_{1}$-Lipschitz gradient will be written $\text{\ref{asmp:sm_loss}}(\parasm_{1})$, and so forth. Observe that for clarity, we differentiate between properties which hold for the \emph{risk} and those which hold for the \emph{loss} by using an asterisk (e.g., note $\text{\ref{asmp:sm_loss}}(\parasm_{1}) \implies \text{\ref{asmp:sm_risk}}(\parasm_{1})$), since depending on the setting we make use of both the weak and strong versions of this condition. We will never have need for strongly convex losses.

\paragraph{Choice of sub-process}

In this work, we shall study two general-purpose learning algorithms, both of which utilize a divide-and-conquer strategy, in which inexpensive sub-processes are run in parallel, and then ``integrated'' in a robust fashion. Detailed exposition of the main learning algorithms will be given respectively in sections \ref{sec:theory_sc} (strongly convex case) and \ref{sec:theory_nonsc} (general case). To wrap up this section of preliminary material, we specify a concrete form for the sub-process that will be used throughout this work. We elect to use traditional (projected) stochastic gradient descent, denoted $\SGD$. The core update of arbitrary point $w$ given data $Z \sim \ddist$ is given by
\begin{align}\label{eqn:sgd_defn}
\SGD\left[w;Z,\alpha,\WW\right] \defeq \proj_{\WW}\left(w-\alpha\,G(w;Z)\right).
\end{align}
Here $\alpha \geq 0$ denotes a step-size parameter, $\proj_{\WW}$ denotes projection to $\WW$ with respect to the $\ell_{2}$ norm, and the standard assumption is that the random vector $G(w;Z)$ satisfies $\exx_{\ddist}G(w;Z) \in \partial \risk_{\ddist}(w)$, for each $w \in \WW$. That is, we assume access to an unbiased estimate of some sub-gradient of the true risk. For an arbitrary sequence $(Z_{1},Z_{2},\ldots,Z_{t})$ of length $t \geq 1$, let $\SGD[\what_{0};(Z_{1},\ldots,Z_{t}),\WW] \defeq \SGD[\what_{t-1};Z_{t},\alpha_{t-1},\WW]$. Note that using (\ref{eqn:sgd_defn}), the right-hand side is defined recursively, and bottoms out at $t=0$, using pre-fixed initial value $\what_{0}$. Note that we suppress the step sizes $(\alpha_{0},\ldots,\alpha_{t-1})$ from this notation for readability. For any arbitrary sub-index $\II \subseteq [n]$, sequence $\SGD[\what_{0};\Z_{\II},\WW]$ is defined analogously; since the $Z_{i}$ are iid, the sequence order does not matter.

\subsection{Under strong convexity}\label{sec:theory_sc}

We begin with a general-purpose learning algorithm that splits the data into $k$ disjoint subsets, runs the sub-routine $\SGD$ on each of these subsets to generate $k$ candidates, and from these candidates a final output is determined by another sub-routine labeled $\merge$. The general procedure is given in Algorithm \ref{algo:DandC_SGD}, and three concrete examples of $\merge$ are specified in Algorithms \ref{algo:merge_geomed}--\ref{algo:merge_median}, with more detailed discussion to follow shortly.

\subsubsection{Illustrative theorem for heavy-tailed losses}

Our analysis starts with a statement of a theorem that holds for potentially heavy-tailed losses, but with bounded gradients. This result is simple to state and effectively illustrates how Algorithm \ref{algo:DandC_SGD} can be used to obtain strong learning guarantees under potentially heavy-tailed data. After sketching out the proof of this theorem, in the following sub-section we will extend this to the case where the sub-gradients can also be heavy-tailed.
\begin{thm}\label{thm:sc_lip_smooth_SGDlast_roboost}
Let $\text{\ref{asmp:lip_loss}}(\parasm_{0})$, $\text{\ref{asmp:sc_risk}}(\parasc)$, and $\text{\ref{asmp:sm_risk}}(\parasm_{1})$ hold in the $\ell_{2}$ norm. Run Algorithm \ref{algo:DandC_SGD} with $n \geq k = \lceil 8\log(\delta^{-1}) \rceil$, and step-size $\alpha_{t}=1/(\parasc \max\{1,t\})$. Then, with probability no less than $1-\delta$, we have
\begin{align*}
\risk_{\ddist}(\what_{\DC})-\risk_{\ddist}^{\ast} \leq \left(\frac{\parasm_{0}^{2}\parasm_{1}^{2}}{\parasc^{3}}\right) \frac{c\log(\delta^{-1})}{n}
\end{align*}
where the constant $c \leq 288$ when $\merge = \smball$, $c \leq 1536$ when $\merge = \geomed$, and $c \leq 1536d$ when $\merge = \median$ (see Lemma \ref{lem:merge_requirement} for details).
\end{thm}

\begin{algorithm}[t!]
\caption{Robust divide and conquer archetype; $\displaystyle \texttt{DC-SGD}\left[\Z_{n},\what_{0}; k\right]$.}
\label{algo:DandC_SGD}
\begin{algorithmic}
\State \textbf{inputs:} sample $\Z_{n}$, initial value $\what_{0} \in \WW$, parameter $1 \leq k \leq n$.
\medskip
\State $\displaystyle \bigcup_{j = 1}^{k}\II_{j} = [n]$, with $|\II_{j}| \geq \lfloor n/k \rfloor$, and $\II_{j} \cap \II_{l} = \emptyset$ when $j \neq l$.
\medskip
\State $\displaystyle \what^{(j)} = \SGD\left[\what_{0}; \Z_{\II_{j}}, \WW\right]$, for each $j \in [k]$.
\medskip
\State \textbf{return:} $\displaystyle \what_{\DC} = \merge\left[\{\what^{(1)},\ldots,\what^{(k)}\}; \|\cdot\|_{2}\right]$.
\end{algorithmic}
\end{algorithm}

\noindent Proving such a theorem is straightforward using the quadratic growth property of strongly convex functions in conjunction with $\parasm_{1}$-smoothness. When the risk is $\parasc$-strongly convex, we have the critical property that points which are $\varepsilon$-far away from the minimum $\wstar$ must be $(\varepsilon^{2}\parasc/2)$-bad in terms of excess risk. As such, simple distance-based robust aggregation metrics can be used to efficiently boost the confidence. To start, we need a few basic facts which will be used to characterize a valid $\merge$ operation.\footnote{Procedures with this property are called ``robust distance approximation'' by \citet{hsu2016a}.} Given $k$ points $u_{1},\ldots,u_{k} \in \RR^{d}$, the basic requirement here is that we want the output of $\merge$ to be close to the majority of these points, in the appropriate norm. To make this concrete, define
\begin{align}
\diameter(u;\gamma,\{u_{1},\ldots,u_{k}\},\|\cdot\|) \defeq \inf \left\{ r \geq 0: |\{j: \|u_{j}-u\| \leq r\}| > k \left(\frac{1}{2}+\gamma\right) \right\}.
\end{align}
When the other parameters are obvious from the context, we shall write simply $\diameter(u;\gamma)=\diameter(u;\gamma,\{u_{1},\ldots,u_{k}\},\|\cdot\|)$. In words, $\diameter(u;\gamma)$ is the radius of the smallest ball centered at $u$ which contains a $\gamma$-majority of the points $u_{1},\ldots,u_{k}$. Using this quantity, our requirement on $\merge$ is that for any $0 \leq \gamma < 1/2$ and $u \in \RR^{d}$, we have
\begin{align}
\label{eqn:merge_requirement}
\left\|\widehat{u} - u \right\| \leq c_{\gamma} \diameter(u;\gamma, \{u_{1},\ldots,u_{k}\}), \text{ where } \widehat{u} = \merge\left[\{u_{1},\ldots,u_{k}\};\|\cdot\|\right].
\end{align}
Here $c_{\gamma}$ is a factor that is independent of the choice of $u$ or the points $u_{1},\ldots,u_{k}$ given, which depends only on the choice of $\gamma$. In the following lemma, we summarize how different sub-routines provide different guarantees (proofs for lemmas are given in the appendix).
\begin{lem}\label{lem:merge_requirement}
The following implementations of $\merge[\{u_{1},\ldots,u_{k}\};\|\cdot\|]$ satisfy (\ref{eqn:merge_requirement}):
\begin{itemize}
\item $\geomed$ (Algorithm \ref{algo:merge_geomed}), with $\displaystyle c_{\gamma} \leq \left(1 + \frac{1}{2\gamma}\right)$.

\item $\smball$ (Algorithm \ref{algo:merge_smball}), with $\displaystyle c_{\gamma} \leq 3$.

\item $\median$ (Algorithm \ref{algo:merge_median}) for $\|\cdot\|=\|\cdot\|_{2}$ case, we have $\displaystyle c_{\gamma} \leq \sqrt{d}\left(1 + \frac{1}{2\gamma}\right)$.
\end{itemize}
\end{lem}
\noindent Considering the partitioning scheme of Algorithm \ref{algo:DandC_SGD}, the ideal case is of course where, given some desired performance level $\risk_{\ddist}(\what^{(j)})-\risk_{\ddist}^{\ast} \leq \varepsilon$, the $\SGD$ sub-routine returns an $\varepsilon$-good candidate for all $j \in [k]$ subsets. In practice, we will not always be so lucky, but the following lemma shows that with enough candidates, most of them will be $\varepsilon$-good with high confidence.
\begin{lem}\label{lem:boost_basic_prop}
Let $(S,\|\cdot\|)$ be any normed linear space. Let $X_{1},\ldots,X_{k}$ be iid random entities taking values in $S$, and fix $x^{\ast} \in S$. For $\varepsilon > 0$, write $a_{i}(\varepsilon) \defeq I\{\|X_{i}-x^{\ast}\| \leq \varepsilon\}$, $\delta_{\varepsilon} \defeq 1-\exx a(\varepsilon)$. For any $0 \leq \gamma<(1/2-\delta_{\varepsilon})$, it follows that
\begin{align*}
\prr\left\{ \sum_{i=1}^{k} a_{i}(\varepsilon) > k\left(\frac{1}{2}+\gamma\right) \right\} \geq 1 - \exp\left(-2k\left(\gamma+\delta_{\varepsilon}-\frac{1}{2}\right)^{2}\right).
\end{align*}
\end{lem}
\noindent Applying Lemma \ref{lem:boost_basic_prop} using the event $a_{j}(\varepsilon) = I\{ \risk_{\ddist}(\what^{(j)})-\risk_{\ddist}^{\ast} \leq \varepsilon \}$, we see that when $k$ scales with $\log(\delta^{-1})$, we can guarantee that there is a $1-\delta$ probability good event in which at least a $\gamma$-majority of the candidates are $\varepsilon$-good. On this good event, via strong convexity it follows that a $\gamma$-majority of the candidates are $\sqrt{2\varepsilon/\parasc}$-close to $\wstar$, which means $\diameter(\wstar;\gamma,\{\what^{(1)},\ldots,\what^{(k)}\}) \leq \sqrt{2\varepsilon/\parasc}$. Leveraging the requirement (\ref{eqn:merge_requirement}) on $\merge$, one obtains the following general-purpose boosting procedure.
\begin{lem}[Boosting the confidence, under strong convexity]\label{lem:boost_conf_sc}
Assume $\text{\ref{asmp:sc_risk}}(\parasc)$, and $\text{\ref{asmp:sm_risk}}(\parasm_{1})$ hold. Assume that we have a learning algorithm $\what_{\textsc{old}}$ which for $n \geq 1$ and $\delta_{0} \in (0,1)$ achieves
\begin{align*}
\prr\left\{ \risk_{\ddist}(\what_{\textsc{old}})-\risk_{\ddist}^{\ast} > \frac{\varepsilon_{\ddist}(n)}{\delta_{0}} \right\} \leq \delta_{0}.
\end{align*}
For desired confidence level $\delta$, split $\Z_{n}$ into $k = \lceil 8\log(\delta^{-1})/(1-\gamma)^{2} \rceil$ disjoint subsets, and let $\what_{\textsc{old}}^{(1)},\ldots,\what_{\textsc{old}}^{(k)}$ be the outputs of $\what_{\textsc{old}}$ run on these subsets. Then setting
\begin{align*}
\what_{\textsc{new}} = \merge\left[\{\what_{\textsc{old}}^{(1)},\ldots,\what_{\textsc{old}}^{(k)}\};\|\cdot\|\right],
\end{align*}
if $\merge$ is any of the sub-routines given in Lemma \ref{lem:merge_requirement}, then for any $0 \leq \gamma < 1/4$ and $n \geq k$, we have that
\begin{align*}
\risk_{\ddist}(\what_{\textsc{new}})-\risk_{\ddist}^{\ast} \leq \frac{4c_{\gamma}^{2}\parasm_{1}}{\parasc}\varepsilon_{\ddist}\left( \frac{(1-\gamma)^{2}n}{8\log(\delta^{-1})} \right)
\end{align*}
with probability no less than $1-\delta$.
\end{lem}
\noindent With these basic results in place, we can readily prove Theorem \ref{thm:sc_lip_smooth_SGDlast_roboost}.
\begin{proof}[Proof of Theorem \ref{thm:sc_lip_smooth_SGDlast_roboost}]
Using the assumptions provided in the hypothesis, we can obviously leverage Lemma \ref{lem:boost_conf_sc}. The key remaining point is to fill in the $\varepsilon_{\ddist}(\cdot)$ bound for last-iterate SGD as specified. Standard arguments yield a $1-\delta$ probability event on which
\begin{align*}
\risk_{\ddist}(\what^{(j)})-\risk_{\ddist}^{\ast} \leq \frac{\parasm_{1}}{(n/k)} \left(\frac{\parasm_{0}}{\parasc}\right)^{2} \left(\frac{1}{\delta}\right),
\end{align*}
and this holds for each $j \in [k]$. See Theorem \ref{thm:learn_sc_lip_smooth_SGDlast} in the appendix for a more detailed statement and complete proof. We then plug this into Lemma \ref{lem:boost_conf_sc}, where the correspondence with Algorithm \ref{algo:DandC_SGD} is $\what_{\textsc{old}}^{(j)} \leftrightarrow \what^{(j)}$ and $\what_{\textsc{new}} \leftrightarrow \what_{\DC}$. It follows that on the high-probability good event, we have
\begin{align*}
\risk_{\ddist}(\what_{\DC})-\risk_{\ddist}^{\ast} \leq \left(\frac{4c_{\gamma}^{2}\parasm_{0}^{2}\parasm_{1}^{2}}{\parasc^{3}}\right) \frac{8\log(\delta^{-1})}{n(1-\gamma)^{2}}.
\end{align*}
This holds for any $\merge$ routine satisfying (\ref{eqn:merge_requirement}). In the case of $\merge=\smball$, by Lemma \ref{lem:merge_requirement}, we have $c_{\gamma} \leq 3$ for all $\gamma \geq 0$. Note that $\gamma$ is a free parameter that does not impact the algorithm being executed, and thus we can set $\gamma = 0$. For $\merge=\geomed$ case, we end up with a factor of the form $(1/(1-\gamma) + 1/(2\gamma(1-\gamma)))^{2}$. Direct computation shows that this is minimized at a value between $1/4$ and $1/2$. The bounds hold for all $\gamma < 1/4$, and thus taking $\gamma \to 1/4$ the factor equals $16$. Basic arithmetic in each case then immediately yields the desired bound.
\end{proof}
\begin{rmk}[Additional related literature]
The excess risk bounds given by Theorem \ref{thm:sc_lip_smooth_SGDlast_roboost} give us an example of the guarantees that are possible under potentially heavy-tailed data, for arguably the simplest divide-and-conquer strategy one could conceive of. Here we remark that the core idea of using robust aggregation methods to boost the confidence of independent candidates under potentially heavy-tailed data can be seen in various special cases throughout the literature. For example, influential work from \citet[Sec.~4]{minsker2015a} applies the geometric median (here, $\geomed$) to robustify both PCA and high-dimensional linear regression procedures, under potentially heavy-tailed observations. \citet[Sec.~4.2]{hsu2016a} look at merging ERM solutions when the \emph{empirical} risk is strongly convex, using a smallest-ball strategy (here, $\smball$). In contrast, we do not require the losses to be strongly convex, and our computational procedure is explicit, yielding bounds which incorporate error of both a statistical and computational nature, unlike ERM-type guarantees.

Broadening our viewpoint slightly, it is worth noting that the general approach seen in the preceding works actually dates back to at least \citet[Sec.~6.6.4, p.~243--246]{nemirovsky1983a}, albeit in a slightly different algorithmic form. The smallest-ball strategy was adopted in interesting recent work by \citet{davis2019a}, who investigate a generic strategy to give stochastic algorithms high-probability error bounds, by solving an additional proximal sub-problem at each iteration, in which the new candidate is within a small-enough ball of the previous candidate. Also quite recently, new work on stochastic convex optimization under potentially heavy-tailed data has appeared from \citet{juditsky2019a}, who study a robust stochastic mirror descent procedure, which fixes an ``anchor'' direction, and only updates using the stochastic gradient oracle if that vector is close enough to the anchor. Under the setting of Theorem \ref{thm:sc_smooth_SGDlast_roboost} to follow shortly, we remark that the error bounds for their procedure are similar to ours (e.g., their Section 6, Thm.~3), but rely critically on the quality of the anchor direction and the threshold level; when such quantities are unknown, the anchor is just set to zero, with the norm threshold being modulated by the size of the entire parameter space, which propagates into the error bounds.\hfill$\blacksquare$
\end{rmk}

\begin{algorithm}[t!]
\caption{Geometric median; $\displaystyle \geomed\left[\{u_{1},\ldots,u_{k}\};\|\cdot\|\right]$}
\label{algo:merge_geomed}
\begin{algorithmic}
\State \textbf{inputs:} points $\{u_{1},\ldots,u_{k}\} \subset \RR^{d}$, norm $\|\cdot\|$.
\medskip
\State \textbf{return:} $\displaystyle \argmin_{v \in \RR^{d}} \sum_{j=1}^{k}\|v-u_{j}\|$.
\end{algorithmic}
\end{algorithm}
\begin{algorithm}[t!]
\caption{Smallest-ball algorithm; $\displaystyle \smball\left[\{u_{1},\ldots,u_{k}\};\|\cdot\|\right]$}
\label{algo:merge_smball}
\begin{algorithmic}
\State \textbf{inputs:}  points $\{u_{1},\ldots,u_{k}\} \subset \RR^{d}$, parameter $0 < \beta < 1/2$, norm $\|\cdot\|$.
\medskip
\State $\displaystyle \diameter_{j} = \inf\left\{r \geq 0: |\{u_{l}: \|u_{j}-u_{l}\| \leq r\}| \geq k(\beta+1/2) \right\}$, for $j \in [k]$.
\medskip
\State $\displaystyle \star \defeq \argmin_{j \in [k]} \diameter_{j}$.
\medskip
\State \textbf{return:} $\displaystyle u_{\star}$.
\end{algorithmic}
\end{algorithm}
\begin{algorithm}[t!]
\caption{Coordinate-wise median; $\displaystyle \median\left[\{u_{1},\ldots,u_{k}\};\|\cdot\|\right]$}
\label{algo:merge_median}
\begin{algorithmic}
\State \textbf{inputs:}  points $\{u_{1},\ldots,u_{k}\} \subset \RR^{d}$.
\medskip
\State $\displaystyle \widehat{u}_{j} = \med\left\{ u_{1,j},\ldots,u_{k,j}\right\}$, for $j \in [k]$.
\medskip
\State \textbf{return:} $\displaystyle (\widehat{u}_{1},\ldots,\widehat{u}_{d})$.
\end{algorithmic}
\end{algorithm}

\subsubsection{Extension to allow heavy-tailed gradients}

Note that the assumptions in Theorem \ref{thm:sc_lip_smooth_SGDlast_roboost} clearly allow for potentially heavy-tailed \emph{losses}, but the Lipschitz condition $\text{\ref{asmp:lip_loss}}(\parasm_{0})$ is equivalent to requiring bounded partial derivatives, meaning that heavy-tailed \emph{gradients} are ruled out, which is not meaningful for algorithms based entirely on first-order information. This is the only assumption made in Theorem \ref{thm:sc_lip_smooth_SGDlast_roboost} that does not appear in the existing RGD literature. On the other hand, existing RGD arguments use a $\parasm_{1}$-smoothness requirement on the loss (e.g., \citet[Sec.~3.2]{holland2019c}), which is a stronger requirement than we have made in Theorem \ref{thm:sc_lip_smooth_SGDlast_roboost}. Here we show that when we align our assumptions to that of the existing RGD theory, it only requires a minor adjustment to the sub-routine used in Algorithm \ref{algo:DandC_SGD} to obtain analogous results, now allowing for \emph{both} the loss and gradient to be potentially heavy-tailed. This is summarized in the following result; the statement is slightly more complicated than the preceding illustrative theorem, but the proof follows using a perfectly analogous argument.
\begin{thm}\label{thm:sc_smooth_SGDlast_roboost}
Let $\text{\ref{asmp:sc_risk}}(\parasc)$, and $\text{\ref{asmp:sm_loss}}(\parasm_{1})$ hold in the $\ell_{2}$ norm. Run Algorithm \ref{algo:DandC_SGD} with a sample size at least $n \geq \max\{k, M^{\ast}\}$, where
\begin{align*}
k = \lceil 8\log(\delta^{-1}) \rceil, \quad M^{\ast} \defeq \frac{4\parasm_{1}}{\parasc}\left(\max\left\{ \frac{\parasm_{1}\parasc\|\what_{0}-\wstar\|^{2}}{\exx_{\ddist}\|G(\wstar;Z)\|^{2}}, 1\right\} - 1 \right).
\end{align*}
For the initial update set $\alpha_{0}=1/(2\parasm_{1})$, and subsequent step sizes $\alpha_{t} = a/(\parasc n + b)$ for $t > 0$, with $b = 2a\parasm_{1}$, and $a>0$ set such that $\alpha_{t} \leq \alpha_{0}$ for all $t$. Then, with probability no less than $1-\delta$, we have
\begin{align*}
\risk_{\ddist}(\what_{\DC})-\risk_{\ddist}^{\ast} \leq \exx_{\ddist}\|G(\wstar;Z)\|^{2}\left(\frac{a\parasm_{1}}{\parasc}\right)^{2} \frac{2c\log(\delta^{-1})}{n-M_{\delta}}
\end{align*}
where $M_{\delta} \leq 16\log(\delta^{-1})(M^{\ast}-b)$, and $c$ is exactly as in Theorem \ref{thm:sc_lip_smooth_SGDlast_roboost}.
\end{thm}
\begin{proof}[Proof of Theorem \ref{thm:sc_smooth_SGDlast_roboost}]
As with the preceding illustrative proof, the key is to fill in $\varepsilon_{\ddist}(\cdot)$ for the final iterate of standard SGD, using the prescribed step sizes. It is well-known that for \emph{averaged} SGD, one does not need to require that the losses be Lipschitz. On the other hand, for last-iterate SGD, it was only quite recently that \citet{nguyen2018a}, in a nice argument building upon \citet{bottou2016a}, showed that the Lipschitz condition is not required if we have $\parasm_{1}$-smooth losses. For our purposes, this implies that for each of the $\what^{(j)}$ candidates in Algorithm \ref{algo:DandC_SGD}, we get 
\begin{align*}
\risk_{\ddist}(\what^{(j)})-\risk_{\ddist}^{\ast} \leq \frac{\exx_{\ddist}\|G(\wstar;Z)\|^{2}}{(n/k)-M^{\ast}+b} \left(\frac{1}{\delta}\right) \left(\frac{2a^{2}\parasm_{1}}{\parasc}\right)
\end{align*}
with probability no less than $1-\delta$. A detailed statement of this property is given in Theorem \ref{thm:learn_sc_smooth_SGDlast} in the appendix. The rest of the argument goes through exactly as in the proof of Theorem \ref{thm:sc_lip_smooth_SGDlast_roboost}, noting that we shall end up with $n-8\log(\delta^{-1})(M^{\ast}-b)/(1-\gamma)^{2}$ in the denominator, for arbitrary choice of $0 \leq \gamma < 1/4$. To cover all choices of $\merge$ and thus $\gamma$, we simply use the rough upper bound $8\log(\delta^{-1})(M^{\ast}-b)/(1-\gamma)^{2} \leq 16\log(\delta^{-1})(M^{\ast}-b)$ in the stated result.
\end{proof}

\subsubsection{Comparison with RGD}\label{sec:comparison_sc}

\begin{table}[t!]
\begin{center}
\resizebox{\textwidth}{!}{ 
\begin{tabular}{|l|c|c|}
\hline
Method & Error & Cost \\
\hline\hline
\texttt{DC-SGD} (Algorithm \ref{algo:DandC_SGD}) & $\displaystyle \bigO\left(\frac{\log(\delta^{-1})}{n}\right)$ & $\displaystyle \bigO\left(\frac{dn}{\log(\delta^{-1})}\right) + \cost\left(\merge\right)$ \\
\hline
RGD-by-MoM \citep{chen2017a,prasad2018a} & $\displaystyle \bigO\left((1-c)^{2T}\right) + \bigO\left(\frac{k(d+\log(\delta^{-1}))}{n}\right)$ & $\displaystyle \bigO\left( \frac{Tdn}{k} \right) + T \cost(\geomed)$ \\
\hline
RGD-M \citep{holland2019a} & $\displaystyle \bigO\left((1-c)^{2T}\right) + \bigO\left( \frac{d(\log(d\delta^{-1})+\log(n))}{n} \right)$ & $\displaystyle \bigO\left( Tdn \right)$ \\
\hline\hline
MoM-by-GD \citep{lecue2018a} & $\displaystyle \bigO\left(\|\what_{T}-\what_{\star}\|\right) + \bigO\left( \sqrt{\frac{\max\{d,\log(\delta^{-1})\}}{n}}\right)$ & $\displaystyle \bigO\left( \frac{Tdn}{k} \right) + \bigO\left( Tk\log(k) \right)$ \\
\hline
\end{tabular}
}
\end{center}
\vspace{-0.5cm}
\caption{Here we compare performance guarantees for different learning algorithms. Error refers to $1-\delta$ confidence intervals for $\risk_{\ddist}(\cdot)-\risk_{\ddist}^{\ast}$, evaluated at the output of each algorithm after $T$ iterations (with \texttt{DC-SGD} using $T=n$ by definition). The first three rows are all under the assumptions of Theorem \ref{thm:sc_smooth_SGDlast_roboost}. The final row is just for reference, specialized to the binary classification problem. Cost estimates assume the availability of $k$ cores for parallel computations.}
\label{table:sc_compare}
\end{table}

Considering the three points of interest highlighted in section \ref{sec:intro} (transparency, strength, and scalability), let us compare Algorithm \ref{algo:DandC_SGD} with the existing RGD algorithms introduced in section \ref{sec:context_contribs}. In Table \ref{table:sc_compare}, we summarize some concrete metrics on the statistical and computational side. Let us unpack and discuss this here. First, the technical assumptions being made here are precisely that of our Theorem \ref{thm:sc_smooth_SGDlast_roboost} for the first three rows of the table. The guarantees follow from \citet[Thm.~4]{chen2017arxiv} for RGD-by-MoM, \citet[Thm.~5]{holland2019c} for RGD-M, and Theorem \ref{thm:sc_smooth_SGDlast_roboost} for \texttt{DC-SGD}. While the factors concealed by the $\bigO(\cdot)$ notation (chiefly $\parasm_{1}$ and $\parasc$) certainly cannot be ignored, in terms of explicit dependence on $n$, $d$, and $\delta$, we see that the dependence is as good or better in all respects, in particular the direct dependence on $d$ is removed. As for MoM-by-GD, the result holds for binary classification using a Lipschitz convex surrogate of the 0-1 loss, from \citet[Thm.~2]{lecue2018a}. Here $\what_{\star}$ denotes the true minimizer of $\widehat{\loss}_{\star}$ given in (\ref{eqn:rgd_lecue}), and $\what_{T}$ the output of MoM-by-GD after $T$ steps. Since convergence rates for $\|\what_{T}-\what_{\star}\|$ are not available, the overall guarantees are weaker than the above-cited RGD procedures.

Regarding computational costs, let us consider basic estimates for the temporal cost in terms of arithmetic operations required. Starting with Algorithm \ref{algo:DandC_SGD}, for each subset $\II_{j}$, we need $\bigO(dn/k)$ operations to obtain candidate $\what^{(j)}$, and these computations can be done independently on processors running in parallel over the entire learning process, until the final merge. With this in mind, the time cost to obtain all $k$ candidates will be $\bigO(dn/k)$, and then all that remains is \emph{one} call to $\merge$, yielding a total cost of $\bigO(dn/k) + \cost(\merge)$. The table above reflects a setting of $k \propto \log(\delta^{-1})$ to match Theorem \ref{thm:sc_smooth_SGDlast_roboost}. For comparison, RGD-by-MoM (\ref{eqn:rgd_mom}) requires $\bigO(dn/k)$ operations to compute one subset mean, and again assuming the computations for each $\II_{j}$, $j \in [k]$, are done across $k$ cores in parallel, then if $T$ iterations are done, the total cost is $\bigO(Tdn/k) + T \cost(\geomed)$, since the $\geomed$-based merging must be done $T$ times. Regarding $\cost(\geomed)$, the geometric median is a convex program, and can efficiently be solved to arbitrary accuracy; \citet{cohen2016a} give an implementation such that the $\geomed$ objective is $(1+\varepsilon)$-good (relative value), with time complexity of $\cost(\geomed) = \bigO(dk\log^{3}(\varepsilon^{-1}))$ for $k$ points. This cost is incurred at each step, in contrast with $\DC$, which in the case of $\merge=\geomed$, only incurs such a cost once. For RGD-M (\ref{eqn:rgd_cat}), note that solving for $\widehat{\theta}_{j}(w)$ can be done readily using a fixed-point update, and in practice the number of iterations is $\bigO(1)$, fixed independently of $n$ and $d$, which means $\bigO(dn)$ operations will be required for each of the $T$ steps. Assuming a standard empirical estimate of the per-coordinate variance is plugged in, this will require an additional $\bigO(dn)$ arithmetic operations. Finally for MoM-by-GD, sorting can be done in $\bigO(k\log(k))$ steps, update directions require just $\bigO(dn/k)$ operations, and these costs are incurred at all $T$ steps.

All else equal, under potentially heavy-tailed losses/gradients, there appears to be fairly strong formal evidence that better statistical guarantees may be possible at substantially improved computational cost, by choosing \texttt{DC-SGD} over the existing RGD procedures in the literature. Since \texttt{DC-SGD} only requires $T=n$ iterations in total, we see the obvious potential for costs to be improved by an order of magnitude, e.g., when $T=\Omega(n)$ for other routines. That said, there are other factors that remain to be considered, such as the variance over time and across independent samples, and the impact in performance for risk functions with different $\parasm_{1}/\parasc$ ratios, in both low- and high-dimensional problem settings. To elucidate how the formal guarantees derived above play out in practice, we conduct a detailed empirical analysis in section \ref{sec:empirical_sc}.

\subsection{Without strong convexity}\label{sec:theory_nonsc}

\subsubsection{Challenges under weak convexity}

When one is lucky enough to have a $\parasc$-strongly convex risk $\risk_{\ddist}$, as illustrated throughout the previous section \ref{sec:theory_sc}, using a very simple basic idea, a wide range of \emph{distance}-based algorithmic strategies are available. Say we have $k$ candidates $\what^{(1)},\ldots,\what^{(k)}$, and we know that with high probability, a majority of the candidates are $\varepsilon$-good in terms of the risk $\risk_{\ddist}$. Since $\risk_{\ddist}$ is unknown, we can never know which candidates are the $\varepsilon$-good ones. However, this barrier can be circumvented by utilizing the fact that $\parasc$-strong convexity of $\risk_{\ddist}$ implies that any $\varepsilon$-good candidate must be at least $\sqrt{2\varepsilon/\parasc}$-close to $\wstar$, the minimizer of $\risk_{\ddist}$ on $\WW$. It follows that on the ``good event'' in which the majority of candidates are $\varepsilon$-good, it is sufficient to simply ``follow the majority.'' This can be done in various ways, but in the end all such procedures comes down to computing and comparing distances $\|w-\what^{(j)}\|$ for all $j \in [k]$. This can be done without knowing which of the $\what^{(j)}$ are $\varepsilon$-good, which made the problem tractable in the previous section.

Unfortunately, as discussed in the introduction, $\parasc$-strong convexity is a luxury that is often unavailable. In particular for high-dimensional settings, it is common for the strong convexity parameter $\parasc$ to shrink rapidly as $d$ grows, making $1/\parasc$-dependent error bounds vacuous \citep{bach2014a}. Algorithmically, if strong convexity cannot be guaranteed, then the distance-based strategy just described will fail, since for any particular minimizer $\wstar$, it is perfectly plausible to have a $\varepsilon$-good candidate which is arbitrarily far from $\wstar$ (see Figure \ref{fig:convex_hull_tricky}). Even when we assume $\parasm_{1}$-smoothness of the risk, all we can say is that $\varepsilon$-badness implies $\sqrt{2\varepsilon/\parasm_{1}}$-farness from all minimizers; the converse need not hold. The traditional approach to this problem is to set aside some additional data, and simply choose the empirical risk minimizer on this new data. More concretely, assume that from the first sample $\Z_{n}$ we obtain independent candidates $\what^{(1)},\ldots,\what^{(k)}$, and that we have a second sample $\Z_{n}^{\prime}=(Z_{1}^{\prime},\ldots,Z_{n}^{\prime})$ available for ``validation,'' as it were. With this second sample, the traditional approach has the learner return
\begin{align}\label{eqn:conf_boosting_classical}
\what = \argmin \left\{ \frac{1}{n}\sum_{i=1}^{n}\loss(w;Z_{i}^{\prime}): w \in \{\what^{(1)},\ldots,\what^{(k)}\} \right\}.
\end{align}
This technique of confidence boosting for bounded losses is well-known; see \citet[Ch.~4.2]{kearns1994CLTIntro} for a textbook introduction, and a more modern statement due to \citet[Thm.~26]{shalev2010a}. Under exp-concave distributions, \citet{mehta2016a} also recently made use of this technique. Problems arise, however, when the losses can be potentially heavy-tailed. The quality of the validated final candidate is only as good as the precision of the risk estimate, and the empirical risk is well-known to be sub-optimal under potentially heavy-tailed data \citep{devroye2016a}.

\begin{figure}[t]
\centering
\includegraphics[width=0.5\textwidth]{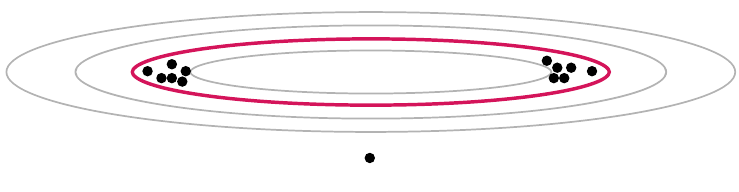}
\caption{An illustration of the difficulties of distance-based methods without strong convexity. The red oval represents the $\varepsilon$-level contour line of convex risk $\risk_{\ddist}$.}
\label{fig:convex_hull_tricky}
\end{figure}

\subsubsection{Error bounds when both losses and gradients can be heavy-tailed}

\begin{algorithm}[t!]
\caption{Divide-and-conquer with robust validation; $\displaystyle \texttt{RV-SGDAve}\left[\Z_{n},\Z_{n}^{\prime},\what_{0};k\right]$.}
\label{algo:DandC_valid}
\begin{algorithmic}
\State \textbf{inputs:} samples $\Z_{n}$ and $\Z_{n}^{\prime}$, initial value $\what_{0} \in \WW$, parameter $1 \leq k \leq n$.

\medskip

\State Split $\displaystyle \bigcup_{j = 1}^{k}\II_{j} = [n]$, with $|\II_{j}| \geq \lfloor n/k \rfloor$, and $\II_{j} \cap \II_{l} = \emptyset$ when $j \neq l$.%

\medskip

\State For each $j \in [k]$, set $\displaystyle \overbar{w}^{(j)}$ to the mean of the sequence $\displaystyle \SGD[\what_{0};\Z_{\II_{j}},\WW]$.

\medskip

\State Compute $\displaystyle \star = \argmin_{j \in [k]} \, \valid\left[\overbar{w}^{(j)};\Z_{n}^{\prime}\right]$.%

\medskip

\State \textbf{return:} $\displaystyle \what_{\RV} = \overbar{w}^{(\star)}$.
\end{algorithmic}
\end{algorithm}

Considering the challenges just described, we look at a straightforward robustification of the classical validation-based approach using robust mean estimators. The full procedure is summarized in Algorithm \ref{algo:DandC_valid}. Viewed at a high level, Algorithm \ref{algo:DandC_valid} is comprised of three extremely simple steps: partition, train, and validate. For our purposes, the key to improving on traditional ERM-style boosting techniques is to ensure the validation step is done with sufficient precision, even when the losses can be heavy-tailed. To achieve this, we shall require that there exist a constant $c > 0$ which does not depend on the distribution $\ddist$, such that for any choice of confidence level $\delta \in (0,1)$ and large enough $n$, the sub-routine $\valid$ satisfies
\begin{align}
\label{eqn:valid_requirement}
\prr\left\{ |\valid\left[w;\Z_{n}^{\prime}\right]-\risk_{\ddist}(w)| > c \sqrt{\frac{(1+\log(\delta^{-1}))\sigma_{\ddist}^{2}(w)}{n}} \right\} \leq \delta.
\end{align}
Recall that we are denoting $\sigma_{\ddist}^{2}(w) \defeq \vaa_{\ddist} \loss(w;Z)$, thus the only requirement on the class of data distributions is finite variance, readily allowing for both heavy-tailed losses and gradients. Three concrete implementations of $\valid$ which satisfy (\ref{eqn:valid_requirement}) are given in Algorithms \ref{algo:valid_mom}--\ref{algo:valid_LM} (see Lemma \ref{lem:valid_basic_prop} shortly), with upper bounds on the constant $c$. The training step can be done in any number of ways; for concreteness and clarity of the results, once again we elect to use the simple stochastic gradient descent sub-process $\SGD$, given earlier in (\ref{eqn:sgd_defn}). Under weak assumptions on the underlying loss distribution, the output $\what_{\RV}$ of this algorithm enjoys strong excess risk bounds, as the following theorem shows.
\begin{thm}\label{thm:smooth_SGDave_roboost}
Let $\risk_{\ddist}$ be $\parasm_{1}$-smooth in the $\ell_{2}$ norm,  $\exx_{\ddist}\|G(w;Z)-\nabla\risk_{\ddist}(w)\|^{2}_{2} \leq \sigma_{G,\ddist}^{2} < \infty$, and $\sigma_{\ddist}^{2}(w) \leq \sigma_{\loss,\ddist}^{2} < \infty$ for all $w \in \WW$. Run Algorithm \ref{algo:DandC_valid} with sub-routine $\valid$ satisfying (\ref{eqn:valid_requirement}), given a total sample size $n \geq 2k$ split into $\Z_{n/2}$ and $\Z_{n/2}^{\prime}$, and $\SGD$ sub-processes using step sizes $\alpha_{t} = 1/(\parasm_{1}+(1/a))$, where $a = \diameter/\sqrt{n\sigma_{G,\ddist}^{2}/2k}$. If we set $k = \lceil \log(2\lceil\log(\delta^{-1})\rceil\delta^{-1}) \rceil$, then for any confidence parameter $0 < \delta \leq 1/3$, we have
\begin{align*}
\risk_{\ddist}(\what_{\RV})-\risk_{\ddist}^{\ast} \leq 2c\sqrt{\frac{2(1+\log(2\lceil\log(\delta^{-1})\rceil\delta^{-1}))\sigma_{\loss,\ddist}^{2}}{n}} + 3\left( \frac{k\diameter^{2}\parasm_{1}}{n} + \sqrt{\frac{2k\diameter^{2}\sigma_{G,\ddist}^{2}}{n}} \right)
\end{align*}
with probability no less than $1-3\delta$.
\end{thm}

\paragraph{Proof sketch}
Here we give an overview of the proof of Theorem \ref{thm:smooth_SGDave_roboost}. We have data sequences $\Z_{n}$ and $\Z_{n}^{\prime}$. The former is used to obtain independent candidates $\overbar{w}^{(1)},\ldots,\overbar{w}^{(k)}$, and the latter is used to select among these candidates. As mentioned earlier, distance-based strategies require that the \emph{majority} of these candidates are $\varepsilon$-good, in order to ensure that points near the majority coincide with $\varepsilon$-good points. In our present setup, where strong convexity is not available, we are taking a very different approach. Now we only require that \emph{at least one} of the candidates is $\varepsilon$-good. Making this explicit,
\begin{align}
\EE_{1}(\varepsilon;k) \defeq \bigcup_{j=1}^{k} \left\{ \risk_{\ddist}(\overbar{w}^{(j)})-\risk_{\ddist}^{\ast} \leq \varepsilon \right\}
\end{align}
is our first event of interest. Note that \emph{if} for each $j \in [k]$ we have an upper bound $\varepsilon_{\ddist}(\cdot)$ depending on the sample size for the sub-process outputs $\overbar{w}^{(j)}$ such that
\begin{align*}
\exx\left[ \risk_{\ddist}(\overbar{w}^{(j)}) - \risk_{\ddist}^{\ast} \right] \leq \varepsilon_{\ddist}(\lfloor n/k \rfloor),
\end{align*}
where expectation is taken over the subset indexed by $\II_{j}$, then using Markov's inequality and taking a union bound, it follows that setting $\varepsilon = \ct{e}\,\varepsilon_{\ddist}$, we have $\prr\EE_{1}(\ct{e}\,\varepsilon_{\ddist};k) \geq 1-\ct{e}^{-k}$. Asking for one $\varepsilon$-good candidate is a much weaker requirement than asking for the majority to be $\varepsilon$-good, but we must pay the price in a different form, as we require that $\valid$ provide a good estimate of the true risk for \emph{all} of the $k$ candidates. In particular, writing $b_{\ddist}(n,\delta)$ for a confidence interval to be specified shortly, this is the following event:
\begin{align}
\EE_{2}(\delta;k) \defeq \bigcap_{j=1}^{k} \left\{ \left|\valid\left[\overbar{w}^{(j)};\Z_{n}^{\prime}\right]-\risk_{\ddist}(\overbar{w}^{(j)})\right| \leq b_{\ddist}(n,\delta) \right\}.
\end{align}
Intuitively, while we only require that at least one of the $k$ candidates be good, we must reliably know which is \emph{best} at the available precision, which requires paying the price of the intersection defining $\EE_{2}(\delta;k)$. Recalling the requirement (\ref{eqn:valid_requirement}), if we condition on $\Z_{n}$, the candidates $\overbar{w}^{(1)},\ldots,\overbar{w}^{(k)}$ become non-random elements of $\WW$, which means that setting $b_{\ddist}(n,\delta) = c\sqrt{(1+\log(\delta^{-1}))\sigma_{\loss,\ddist}^{2}/n}$, a union bound gives us $\prr\left(\EE_{2}(\delta;k);\Z_{n}\right) \geq 1 - k\delta$. This inequality holds as-is for any realization of $\Z_{n}$, so we can thus integrate to obtain
\begin{align*}
\prr\EE_{2}(\delta;k) = \int \prr\left(\EE_{2}(\delta;k);\Z_{n}\right) \, \ddist(\dif\Z_{n}) \geq 1-k\delta.
\end{align*}
The good event of interest then has probability
\begin{align*}
\prr\left[\EE_{1}\left(\ct{e}\,\varepsilon_{\ddist}\left(\left\lfloor \frac{n}{k} \right\rfloor\right);k\right) \cap \EE_{2}(\delta;k)\right] \geq 1 - \ct{e}^{-k} - k\delta.
\end{align*}
On this good event, we know that there does exist an $\varepsilon$-good candidate, even though we can never know which it is; call it $\overbar{w}_{\textsc{luck}} \in \{\overbar{w}^{(1)},\ldots,\overbar{w}^{(k)}\}$. Furthermore, even though this candidate is unknown, since we have $b_{\ddist}(n,\delta)$-good risk estimates for all $k$ candidates, the choice of $\overbar{w}^{(\star)}$, with $\star = \argmin_{j \in [k]} \valid[\overbar{w}^{(j)};\Z_{n}^{\prime}]$, cannot be much worse. More precisely, we have
\begin{align*}
\risk_{\ddist}(\overbar{w}^{(\star)}) - \risk_{\ddist}^{\ast} & = \risk_{\ddist}(\overbar{w}^{(\star)}) - \valid[\overbar{w}^{(\star)}] + \valid[\overbar{w}^{(\star)}] - \risk_{\ddist}^{\ast}\\
& \leq \risk_{\ddist}(\overbar{w}^{(\star)}) - \valid[\overbar{w}^{(\star)}] + \valid[\overbar{w}_{\textsc{luck}}] - \risk_{\ddist}^{\ast}\\
& = \left[\risk_{\ddist}(\overbar{w}^{(\star)}) - \valid[\overbar{w}^{(\star)}]\right] + \left[\valid[\overbar{w}_{\textsc{luck}}] - \risk_{\ddist}(\overbar{w}_{\textsc{luck}})\right] + \left[\risk_{\ddist}(\overbar{w}_{\textsc{luck}}) - \risk_{\ddist}^{\ast}\right]\\
& \leq 2b_{\ddist}(n,\delta) + \ct{e}\,\varepsilon_{\ddist}\left(\lfloor n/k \rfloor\right).
\end{align*}
We have effectively proved the following lemma.
\begin{lem}[Boosting the confidence under potentially heavy tails]\label{lem:boost_conf_unbounded}
Assume we have a learning algorithm $\learn$ such that for $n \geq 1$ and $\delta \in (0,1)$, we have
\begin{align*}
\prr\left\{ \risk_{\ddist}(\learn[\Z_{n}])-\risk_{\ddist}^{\ast} > \frac{\varepsilon_{\ddist}(n)}{\delta} \right\} \leq \delta.
\end{align*}
Splitting the data $\Z_{n}$ using sub-indices $\II_{1},\ldots,\II_{k}$, if we set
\begin{align*}
\star = \argmin_{j \in [k]} \valid\left[ \learn[\Z_{\II_{j}}]; \Z_{n}^{\prime} \right],
\end{align*}
then when $\valid$ satisfies (\ref{eqn:valid_requirement}), it follows that for any $\delta \in (0,1)$, we have
\begin{align*}
\risk_{\ddist}\left(\learn[\Z_{\II_{\star}}]\right)-\risk_{\ddist}^{\ast} & \leq \sup_{w \in \WW} 2c\sqrt{\frac{(1+\log(\delta^{-1}))\sigma_{\ddist}^{2}(w)}{n}} + \ct{e} \, \varepsilon_{\ddist}\left(\left\lfloor\frac{n}{k}\right\rfloor\right)
\end{align*}
with probability no less than $1-k\delta-\ct{e}^{-k}$.
\end{lem}
\noindent Note that Lemma \ref{lem:boost_conf_unbounded} here makes no direct requirements on the underlying loss or risk, beyond the need for a variance bound, which appears as $\sigma_{\ddist}^{2}(w) \leq \sigma_{\loss,\ddist}^{2} < \infty$ in the statement of Theorem \ref{thm:smooth_SGDave_roboost}. Indeed, convexity does not even make an appearance. This is in stark contrast with the distance-based confidence boosting methods used in section \ref{sec:theory_sc}, and elsewhere in the literature \citep{minsker2015a,hsu2016a}. As such, so long as we can validate in the sense of (\ref{eqn:valid_requirement}), then Lemma \ref{lem:boost_conf_unbounded} gives us a general-purpose tool from which we can construct algorithms with competitive risk bounds under potentially heavy-tailed data. The following lemma shows that validation can indeed be done in the desired way, using straightforward computational procedures.
\begin{lem}\label{lem:valid_basic_prop}
The following implementations of $\valid[w;\cdot]$ satisfy (\ref{eqn:valid_requirement}) with sample size $n$ and confidence level $0 < \delta < 1$, when passed sample $\{\loss(w;Z_{i}^{\prime}): i \in [n]\}$.
\begin{itemize}
\item $\mom[\cdot;k^{\prime}]$ (Algorithm \ref{algo:valid_mom}), with $c \leq 2\sqrt{2}\ct{e}$, when $k^{\prime} = \lceil \log(\delta^{-1})\rceil$ and $n \geq 2(1+\log(\delta^{-1}))$.

\item $\cat[\cdot;\sigma_{\ddist}^{2}(w),\delta]$ (Algorithm \ref{algo:valid_cat}), with $c \leq 2$, when $n \geq 4\log(\delta^{-1})$.

\item $\LM[\cdot;\delta]$ (Algorithm \ref{algo:valid_LM}), with $c \leq 9\sqrt{2}$, when $n \geq (16/3)\log(8\delta^{-1})$.
\end{itemize}
\end{lem}
\begin{algorithm}[t!]
\caption{Median of means estimate; $\displaystyle \mom[\{u_{1},\ldots,u_{n}\};k]$.}
\label{algo:valid_mom}
\begin{algorithmic}
\State \textbf{inputs:} sample $\{u_{1},\ldots,u_{n}\}$, parameter $1 \leq k \leq n$.

\medskip

\State $\displaystyle \bigcup_{j = 1}^{k}\II_{j} = [n]$, with $|\II_{j}| \geq \lfloor n/k \rfloor$, and $\II_{j} \cap \II_{l} = \emptyset$ when $j \neq l$.%

\medskip

\State $\displaystyle \widehat{u}_{j} = \frac{1}{|\II_{j}|} \sum_{i \in \II_{j}} u_{i}$, for each $j \in [k]$.%

\medskip

\State \textbf{return:} $\displaystyle \med\{\widehat{u}_{1},\ldots,\widehat{u}_{k}\}$.
\end{algorithmic}
\end{algorithm}

\begin{algorithm}[t!]
\caption{Catoni-type M-estimate; $\displaystyle \cat[\{u_{1},\ldots,u_{n}\};\sigma,\delta]$.}
\label{algo:valid_cat}
\begin{algorithmic}
\State \textbf{inputs:} sample $\{u_{1},\ldots,u_{n}\}$, parameters $\sigma > 0$ and $0 < \delta < 1$.

\medskip

\State Set $\displaystyle q^{2} = \frac{2\sigma^{2}\log(2\delta^{-1})}{n-2\log(2\delta^{-1})}$ and $\displaystyle s^{2} = \frac{n(\sigma^{2}+q^{2})}{2\log(2\delta^{-1})}$. %

\medskip

\State \textbf{return:} $\displaystyle \argmin_{\theta \in \RR} \sum_{i=1}^{n} \rho\left(\frac{u_{i}-\theta}{s}\right)$.
\end{algorithmic}
\end{algorithm}

\begin{algorithm}[t!]
\caption{Truncated mean estimate; $\displaystyle \LM[\{u_{1},\ldots,u_{n}\};\delta]$.}
\label{algo:valid_LM}
\begin{algorithmic}
\State \textbf{inputs:} sample $\displaystyle \{u_{1},\ldots,u_{n}\}$, parameter $\displaystyle 0 < \delta < 1$.

\medskip

\State Split the index $\displaystyle [n] = \II_{1} \cup \II_{2}$, with $\displaystyle \II_{1} \cap \II_{2} = \emptyset$ and $\displaystyle |\II_{1}| \geq |\II_{2}| \geq \lfloor n/2 \rfloor$.

\medskip

\State Set $\beta = 32\log(8\delta^{-1})/(3n)$.

\medskip

\State Set $a$ and $b$ to the $\beta$- and $(1-\beta)$-level quantiles of $\{u_{i}: i \in \II_{2}\}$.

\medskip

\State \textbf{return:} $\displaystyle \frac{1}{|\II_{1}|} \sum_{i \in \II_{1}} u_{i} I_{\{a \leq u_{i} \leq b\}}$.
\end{algorithmic}
\end{algorithm}
\noindent With these key facts in place, it is straightforward to prove Theorem \ref{thm:smooth_SGDave_roboost}.
\begin{proof}[Proof of Theorem \ref{thm:smooth_SGDave_roboost}]
Since most key facts have already been laid out, we just need to fill in a few blanks and connect these facts. To begin, consider the $\varepsilon_{\ddist}(\cdot)$-bound on $\learn$ in Lemma \ref{lem:boost_conf_unbounded}. The special case of Algorithm \ref{algo:DandC_valid} is just $\learn[\Z_{n}] = \texttt{Average}\left[\SGD[\what_{0};\Z_{n},\WW]\right]$, namely the simplest form of averaged stochastic gradient descent. Given the assumptions, we are doing averaged SGD under a $\parasm_{1}$-smooth risk, without assuming strong convexity or a Lipschitz loss, and using the step-sizes specified in the hypothesis, a standard argument gives us
\begin{align}\label{eqn:smooth_SGDave_roboost_1}
\varepsilon_{\ddist}(n) \leq \left( \frac{\diameter^{2}\parasm_{1}}{2n} + \sqrt{\frac{\diameter^{2}\sigma_{G,\ddist}^{2}}{n}} \right),
\end{align}
where $\sigma_{G,\ddist}^{2}$ is as given in the theorem statement. See Theorem \ref{thm:learn_conv_smooth_SGDave} in the appendix for a proof of the more general result that implies (\ref{eqn:smooth_SGDave_roboost_1}). This can be applied to each sub-process via the correspondence $\overbar{w}^{(j)} \leftrightarrow \learn[\Z_{\II_{j}}]$. Note that the output of Algorithm \ref{algo:DandC_valid} corresponds to $\what_{\RV} \leftrightarrow \learn[\Z_{\II_{\star}}]$. Thus leveraging Lemma \ref{lem:boost_conf_unbounded} and (\ref{eqn:smooth_SGDave_roboost_1}), and bounding $\sigma_{\ddist}^{2}(w) \leq \sigma_{\loss,\ddist}^{2}$, we have for any choice of $\delta_{0} \in (0,1)$ that
\begin{align}\label{eqn:smooth_SGDave_roboost_2}
\risk_{\ddist}(\what_{\RV})-\risk_{\ddist}^{\ast} & \leq 2c\sqrt{\frac{(1+\log(\delta_{0}^{-1}))\sigma_{\loss,\ddist}^{2}}{n}} + \ct{e} \, \left( \frac{k\diameter^{2}\parasm_{1}}{2n} + \sqrt{\frac{k\diameter^{2}\sigma_{G,\ddist}^{2}}{n}} \right)
\end{align}
with probability no less than $1-\ct{e}^{-k}-k\delta_{0}$. Note that we are assuming $k$ divides $n$ for simplicity, and using the notation $\delta_{0}$ to distinguish from $\delta$ in the theorem statement. It just remains to clean up this probability and specify $k$. To do this, given $\delta$ in the theorem statement, first set $\delta_{0} = \delta/(2\lceil\log(\delta^{-1})\rceil) < \delta$. Next, set the number of subsets to be
\begin{align*}
k = \lceil \log(1/\delta_{0}) \rceil = \lceil \log(2\lceil\log(\delta^{-1})\rceil\delta^{-1}) \rceil,
\end{align*}
and note that with this setting of $k$ and $\delta_{0}$, we have that
\begin{align*}
1-k\delta_{0} & = 1 - \lceil \log(2\lceil\log(\delta^{-1})\rceil\delta^{-1}) \rceil\left(\frac{\delta}{2\lceil\log(\delta^{-1})\rceil}\right)\\
& \geq 1 - \left( \frac{\lceil\log(2)\rceil}{\lceil\log(\delta^{-1})\rceil} + \frac{\lceil\log(\log(\delta^{-1}))\rceil}{\lceil\log(\delta^{-1})\rceil} + 1 \right)\frac{\delta}{2}\\
& \geq 1 - \left(\frac{3}{2}\right)\delta\\
& \geq 1 - 2\delta.
\end{align*}
The inequalities follow readily via the fact that for arbitrary $c_{1},c_{2} \geq 0$ we have $\lceil c_{1} + c_{2} \rceil \leq \lceil c_{1} \rceil + \lceil c_{2} \rceil$, and that $\lceil\log(2)\rceil / \lceil\log(\delta^{-1})\rceil \leq 1$ for all $\delta \leq 1/2$. As for the exponential term, note that
\begin{align*}
\ct{e}^{-k} = \exp\left( -\lceil \log(\delta_{0}^{-1}) \rceil \right) \leq \exp\left( -\log(\delta_{0}^{-1}) \right) = \delta_{0} < \delta.
\end{align*}
It thus immediately follows that the good event of (\ref{eqn:smooth_SGDave_roboost_2}) holds probability no less than
\begin{align*}
1-\ct{e}^{-k}-k\delta_{0} \geq 1-\delta-2\delta = 1-3\delta.
\end{align*}
To conclude, since we have $n$ observations split in half, we must replace $n$ with $n/2$ in (\ref{eqn:smooth_SGDave_roboost_2}). Bounding the coefficient $\ct{e} \leq 3$ for simplicity yields the desired result.
\end{proof}

\subsubsection{Comparison of error bounds}\label{sec:comparison_nonsc}

Recall that in the introduction, we highlighted properties of transparency, strength, and stability as being important to close the gap between formal guarantees and the performance achieved by the methods we actually are coding. As mentioned in the literature review in section \ref{sec:context_contribs}, the robust gradient descent algorithms cited are noteworthy in that the procedures have strong (albeit slightly sub-optimal) guarantees for a wide class of distributions, for procedures which can be implemented essentially as-stated in the cited papers, making the guarantees very transparent. Unfortunately, the best results are essentially limited to problems in which the risk $\risk_{\ddist}$ is $\parasc$-strongly convex; all of the cited papers make extensive use of this property in their analysis \citep{chen2017a,holland2017arxiv,prasad2018a}. If one is lucky enough to have $\parasc$-strong convexity (and $\parasm_{1}$-smoothness), then for any step $t$, one has
\begin{align*}
\|\what_{t} - \alpha_{t} \, \nabla\risk_{\ddist}(\what_{t}) - \wstar\|^{2} \leq \left(1 - \frac{2\alpha_{t}\parasc\parasm_{1}}{\parasc+\parasm_{1}}\right)\|\what_{t} - \wstar\|^{2}.
\end{align*}
The only difference between the left-hand side of this inequality and the general-purpose robust GD update studied in the literature is that the true risk gradient is replaced with some estimator $\widehat{G}_{n} \approx \nabla \risk_{\ddist}$. As such, one can easily control $\|\what_{t+1} - \wstar\|$ using an upper bound that depends on the right-hand side of the above inequality and the statistical estimation error $\|\widehat{G}_{n}(\what_{t}) - \nabla \risk_{\ddist}(\what_{t})\|$. After say $T$ iterations, one can then readily unfold the recursion and obtain final error bounds that can be given as a sum of an optimization error term depending on the number of iterations $T$, and a statistical error term depending on the sample size $n$ (e.g., \citet[Thm.~5]{chen2017arxiv}, \citet[Sec.~7]{prasad2018a}, \citet[Thm.~5]{holland2019c}).

\paragraph{Error bounds without strong convexity}
On the other hand, when one does not have strong convexity, such a technique fails, and one is left having to compare the difference between two sequences, the actual robust GD iterates $(\what_{t})$, and the ideal sequence $(\wstar_{t})$ of gradient descent using the true risk gradient, assuming both sequences are initialized at the same point $\what_{0}=\wstar_{0}$. This point is discussed with analysis by \citet[Sec.~A.3]{holland2019d}.\footnote{Their original bounds involve a factor $dV$, where $V$ is an upper bound on the variance of the partial derivatives of the loss taken over all coordinates. One can easily strengthen their bounds by replacing bounds stated using $dV$ with bounds stated using $\sigma_{G,\ddist}^{2}$. Analogous analysis can be done to extend the results of \citet{chen2017a} to the weak convexity case as well.} One can still unfold the recursion without much difficulty, but the propagation of the statistical error becomes much more severe. In the simple case using a fixed step-size of $\alpha > 0$, ignoring non-dominant terms, under the same technical assumptions used in our theoretical analysis, after $T$ steps, the robust RGD procedures can only obtain $(1-\delta)$-high probability bounds of the form
\begin{align*}
\risk_{\ddist}(\what_{T})-\risk_{\ddist}^{\ast} \lesssim \bigO\left((1+\parasm_{1}\alpha)^{T}\sqrt{\frac{d(\sigma_{G,\ddist}^{2}\log(d\delta^{-1})+\log(n))}{n\,T}}\right),
\end{align*}
Note that the exponential dependence on $T$ makes the maximum number of iterations one can guarantee extremely sensitive to the values of $\parasm_{1}$ and $\alpha$.

In contrast, under the same assumptions, Theorem \ref{thm:smooth_SGDave_roboost} for our Algorithm \ref{algo:DandC_valid} has no such sensitivity; it achieves the same dependence on $n$ and $1/\delta$ with just one pass over the data. Furthermore, there is no explicit dependence on the number of parameters $d$, the $\log(n)$ factor is removed, and the dependence on $\parasm_{1}$ is improved exponentially. Since typical RGD procedures do not ever use loss values, the only moment bound requirement they make is via $\sigma_{G,\ddist}^{2}$, whereas our procedure has both $\sigma_{G,\ddist}^{2}$ and $\sigma_{\loss,\ddist}^{2}$. Other minor tradeoffs exist in the form of an extra $\log(\log(\delta^{-1}))$ factor, and dependence on the diameter $\diameter$ in our bounds is linear, whereas the works of \citet{chen2017arxiv} and \citet{holland2017arxiv} have logarithmic dependence. Arguably, this is a small price to pay for the improvements that are afforded. These results are shown in the second column of Table \ref{table:nonsc_compare}.

\begin{table}[t!]
\begin{center}
\begin{tabular}{|l|c|c|}
\hline
Method & Error & Cost \\
\hline\hline
\texttt{RV-SGDAve} & $\displaystyle \bigO\left(\sqrt{\frac{\log(\delta^{-1})}{n}}\left(\sigma_{\loss,\ddist}+\sigma_{G,\ddist}\right)\right) + \bigO\left(\frac{\parasm_{1}\log(\delta^{-1})}{n}\right)$ & $\displaystyle \bigO(dn\log(\delta^{-1}))$ \\
\hline
RGD & $\displaystyle \bigO\left((1+\parasm_{1}\alpha)^{T}\sqrt{\frac{d(\sigma_{G,\ddist}^{2}\log(d\delta^{-1})+\log(n))}{n\,T}}\right)$ & $\displaystyle \bigO\left( Tdn\log(\delta^{-1}) \right)$ \\
\hline
\end{tabular}
\end{center}
\vspace{-0.5cm}
\caption{High-probability error bounds and computational cost estimates for \texttt{RV-SGDAve} (Algorithm \ref{algo:DandC_valid}), compared with modern RGD methods, \emph{without} assuming strong convexity. Error denotes confidence intervals for $\risk_{\ddist}(\what_{n})-\risk_{\ddist}^{\ast}$ with $\what$ being the output of each procedure after $T$ steps (noting \texttt{RV-SGDAve} has $T=n$ by definition).}
\label{table:nonsc_compare}
\end{table}

\paragraph{Computational cost}
Due to the ease of distributed computation and simplicity of the underlying sub-routines, Algorithm \ref{algo:DandC_valid} has significant potential to improve upon existing robust GD methods in terms of computational scalability. Using arithmetic operations as a rough estimate of time complexity, first for Algorithm \ref{algo:DandC_valid} note that for each subset $\II_{j}$, we have a fixed number of arithmetic operations that must be done for $d$ coordinates and $|\II_{j}| \geq \lfloor n/k \rfloor$ iterations. Thus one can obtain each candidate $\what^{(j)}$ with $\bigO(dn/k)$ operations, and this is done for each $j \in [k]$. These computations can trivially be done on independent cores running in parallel. It then just remains to make a single call to $\valid$ to conclude the procedure. In this final call, one evaluates $k$ candidates at $\bigO(n)$ data points; this will typically require $\bigO(dkn)$ operations, plus the cost of the final robust estimate, which will be respectively $\cost(\mom)=\bigO(k\log(k))$ and $\cost(\cat)=\bigO(n)$ for the cases described in Lemma \ref{lem:valid_basic_prop}. Adding these costs up, ignoring $\log(k)$ factors, and setting $k \propto \log(\delta^{-1})$ for simplicity yields the cost shown in the third column of Table \ref{table:nonsc_compare}. Costs for RGD with this $k$ setting follows from our discussion in section \ref{sec:comparison_sc}.

\section{Empirical analysis}\label{sec:empirical}

In this section, we carry out detailed empirical analysis of the proposed learning algorithms using both controlled simulations and performance tests on real-world benchmark datasets. Simulation-based results are given in sections \ref{sec:empirical_sc}--\ref{sec:empirical_nonsc}, while applications to real-world data are given in section \ref{sec:empirical_realdata}.

\paragraph{Online software repository}
In order to ensure the experiments to follow are easily reproducible, we provide all the necessary code at the following online repository:\\\url{https://github.com/feedbackward/sgd-roboost}.

\subsection{Controlled simulations, under strong convexity}\label{sec:empirical_sc}

In this section, we use controlled simulations to investigate how the differences in formal performance guarantees discussed in the previous section work out in practice.

\paragraph{Experimental setup}

We essentially follow the ``noisy convex minimization'' tests done by \citet{holland2019c} to compare the performance of robust gradient descent procedures with traditional ERM minimizers. For simplicity, we start with a risk function that takes a quadratic form $\risk_{\ddist}(w) = \langle \Sigma w, w \rangle + \langle w, u \rangle + a$, where $\Sigma \in \RR^{d \times d}$, $u \in \RR^{d}$, and $a \in \RR$ are constants that depend on the experimental conditions. Now in order to line the experimental setting up with the theory of section \ref{sec:theory_sc}, the idea is to construct an easily manipulated loss distribution such that the expectation aligns precisely with the quadratic $\risk_{\ddist}$ just given. To achieve this, one can naturally compute losses of the form $\loss(w;Z) = (\langle w-\wstar, X \rangle + E)^{2}/2$, where $\wstar \in \RR^{d}$ is a pre-defined vector unknown to the learner, $X$ is a $d$-dimensional random vector, $E$ is zero-mean random noise, and $X$ and $E$ are independent of each other. Note that $\ddist$ in this case corresponds to the joint distribution of $X$ and $E$, although all the learner sees is the loss value. It is readily confirmed under such a setting we have $\exx_{\ddist}\loss(w;Z) = \risk_{\ddist}(w)$ in the quadratic form given above with $\Sigma = \exx_{\ddist}XX^{\trans}/2$, $u = -2\Sigma\wstar$, and $a = \langle \Sigma\wstar,\wstar \rangle + \exx_{\ddist}E^{2}/2$ (see appendix). For any non-trivial choice of the distribution of $X$, the resulting matrix $\Sigma$ will be positive definite, implying that $\risk_{\ddist}$ is strongly convex. Furthermore, since the gradients $\nabla \loss(w;Z) = -(\langle \wstar-w, X \rangle + E)X$ are clearly Lipschitz continuous whenever $X$ has a bounded support, the two key assumptions of Theorem \ref{thm:sc_smooth_SGDlast_roboost} are satisfied.

Regarding the methods being compared, as classical baselines, empirical risk minimization using batch GD (denoted \texttt{ERM-GD}) and stochastic GD (denoted \texttt{SGD}) are used. We also implement the robust GD methods described in section \ref{sec:context_contribs}. In particular, RGD-by-MoM is denoted here as \texttt{RGD-MoM}, MoM-by-GD is denoted as \texttt{RGD-Lec}, and RGD-M is denoted as \texttt{RGD-M}. Finally, Algorithm \ref{algo:DandC_SGD} (with $\merge = \geomed$) is denoted by \texttt{DC-SGD}. Everything is implemented in Python (ver.~3.8), chiefly relying upon the \texttt{numpy} library (ver.~1.18).\footnote{Documentation for \texttt{numpy}: \url{https://numpy.org/doc/1.18/index.html}.} The basic idea of these tests is to calibrate and fix the methods to the case of ``nice'' data characterized by additive Gaussian noise, and then to see how the performance of each method changes as different experimental parameters are modified. For all algorithms that use a $k$-partition of the data, we have fixed $k=10$ throughout all tests. Partitioning is done using the \texttt{split\_array} function in \texttt{numpy}, which means each subset gets at least $\lfloor n/k \rfloor$ points. Details regarding step-size settings will be described shortly. Finally, the initial value $\what_{0}$ for all methods is determined randomly, using $\what_{0,j} = \wstar_{j}+\text{Uniform}[-c,+c]$ for each coordinate, with $c = 5.0$ unless otherwise specified.

The key performance metric that we look at in the figures to follow is ``excess risk,'' computed as $\risk_{\ddist}(\what)-\risk_{\ddist}(\wstar)$, where $\what$ is the output of any learning algorithm being studied, and $\wstar$ is the pre-fixed minimum described in the previous paragraph. Each experimental setting is characterized by the triplet $(\ddist,n,d)$, which we modify in many different ways to investigate different phenomena. For each setting, we run multiple independent trials, and compute performance statistics based on these trials. For example, when we give the average (denoted \texttt{ave}) and standard deviation (denoted \texttt{sd}) of excess risk, these statistics are computed over all trials. All box-plots are also computed based on multiple independent trials; the actual number of trials will be described in the subsequent exposition. The main points of empirical inquiry addressed in this sub-section are as follows:
\begin{itemize}
\item[\textbf{(E1)}] Error trajectories in low dimensions (fixed $n$ and $d$, many iterations).
\item[\textbf{(E2)}] Statistical error in high dimensions ($d$ grows, $n$ fixed).
\item[\textbf{(E3)}] Actual computation times ($d$ grows, $n$ fixed/grows).
\item[\textbf{(E4)}] Impact of initialization on error trajectories ($\|\what_{0}-\wstar\|$ grows).
\item[\textbf{(E5)}] Impact of noise level on error trajectories (signal/noise ratio gets worse).
\end{itemize}
We proceed to describe these experimental settings and related results one by one.

\paragraph{(E1) Error trajectories in low dimensions}

We start with a simple setting, fixing $d=2$ and $n=500$, and comparing how pre-fixed algorithms perform depending on whether $E$ is Normal or log-Normal. To be more precise, let $Y \sim \text{Normal}(0,b^{2})$. We consider the cases of $E = Y - \exx Y$ and $E = \ct{e}^{Y} - \exx \ct{e}^{Y}$, respectively with $b=2.2$ and $b=1.75$. These distinct settings of $b$ are set to keep the width of the inter-quartile range of the noise in both cases approximately equal. The case of Normal noise is used as a baseline to calibrate fixed step sizes for all methods. The four batch methods all have $\alpha_{t} = 0.1/\sqrt{d}$, the two sequential methods have $\alpha_{t} = 0.01/\sqrt{d}$, and these settings are fixed throughout the remaining experiments to evaluate the impact of different conditions on pre-fixed algorithms.\footnote{The reasoning for having the step size shrink with $d$ is because the smoothness parameter $\parasm_{1}$ scales with $\|X\|$ and thus $\sqrt{d}$, all else equal. See constant $b$ in Theorem \ref{thm:sc_smooth_SGDlast_roboost} for the motivation behind having $\alpha_{t}$ shrink with $d$ in Algorithm \ref{algo:DandC_SGD}. The same requirement is made for RGD methods \citep[Sec.~3.2]{holland2019c}.} All learning algorithms are given access to exactly $n$ data points $Z_{1},\ldots,Z_{n}$, based on which they can evaluate $\loss(w;Z_{i})$ and $\nabla \loss(w;Z_{i})$ at any $w \in \WW$ and $i \in [n]$. We examine how all algorithms behave under the settings just described, given a large time budget. More specifically, \texttt{cost} is measured in terms of gradient evaluations, so that every time an algorithm computes $\nabla \loss(w;Z_{i})$ for some $w$ and $i$, we increment $\texttt{cost}\gets\texttt{cost+1}$, until $\texttt{cost}\geq\texttt{budget}$. Here we set $\texttt{budget}=40n\sqrt{d}$. Note that this means \texttt{SGD} and \texttt{DC-SGD} will make many passes over the data; we shall look at the case of few passes over the data in short order. The number of independent trials here is $100$. Results for this experimental setting are given in Figure \ref{fig:POC_normal_sc} (Normal case) and Figure \ref{fig:POC_lognormal_sc} (log-Normal case).

\begin{figure}[t]
\centering
\includegraphics[width=0.4\textwidth]{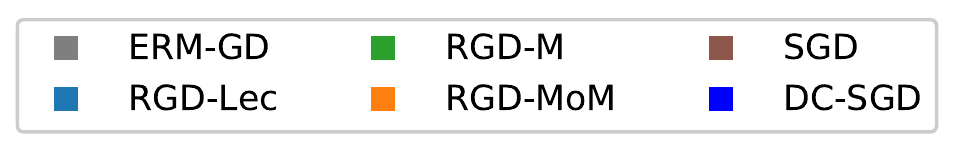}\\
\includegraphics[width=0.49\textwidth]{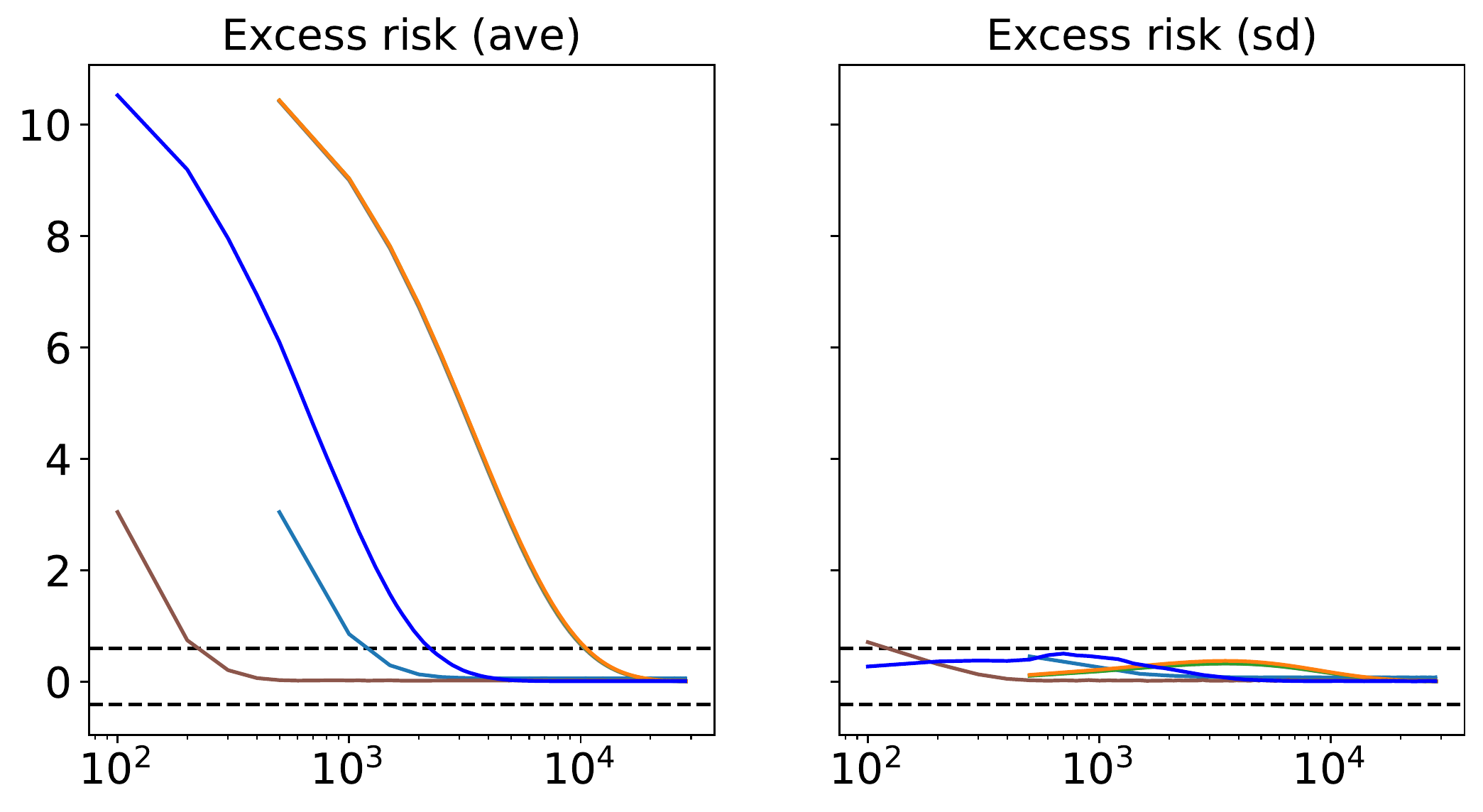}\,\includegraphics[width=0.51\textwidth]{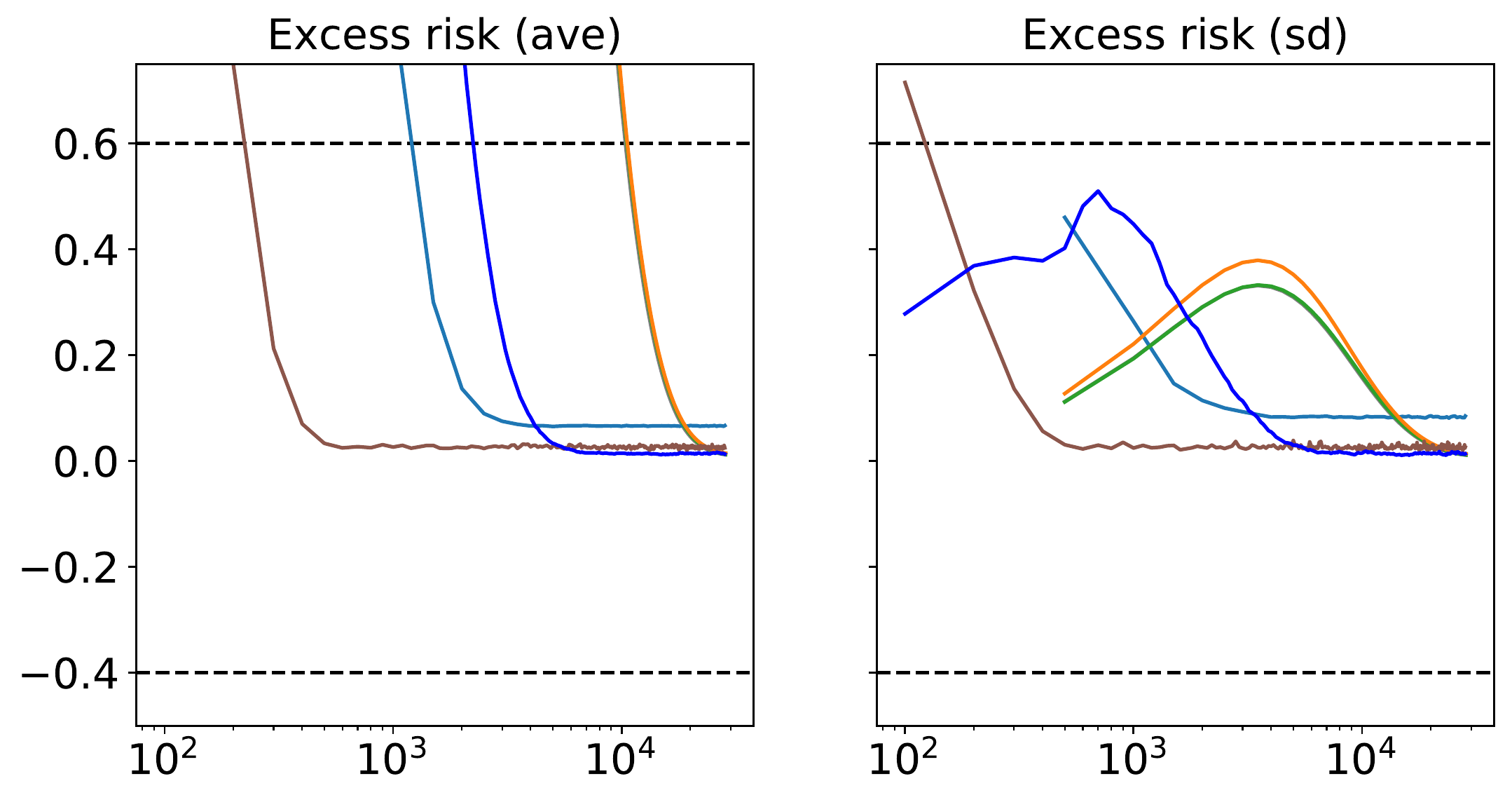}
\caption{Excess risk statistics as a function of cost in gradients (log scale, base $10$). The two right-most plots zoom in on the region between the dashed lines in the two left-most plots.}
\label{fig:POC_normal_sc}
\vspace{0.25cm}
\includegraphics[width=0.5\textwidth]{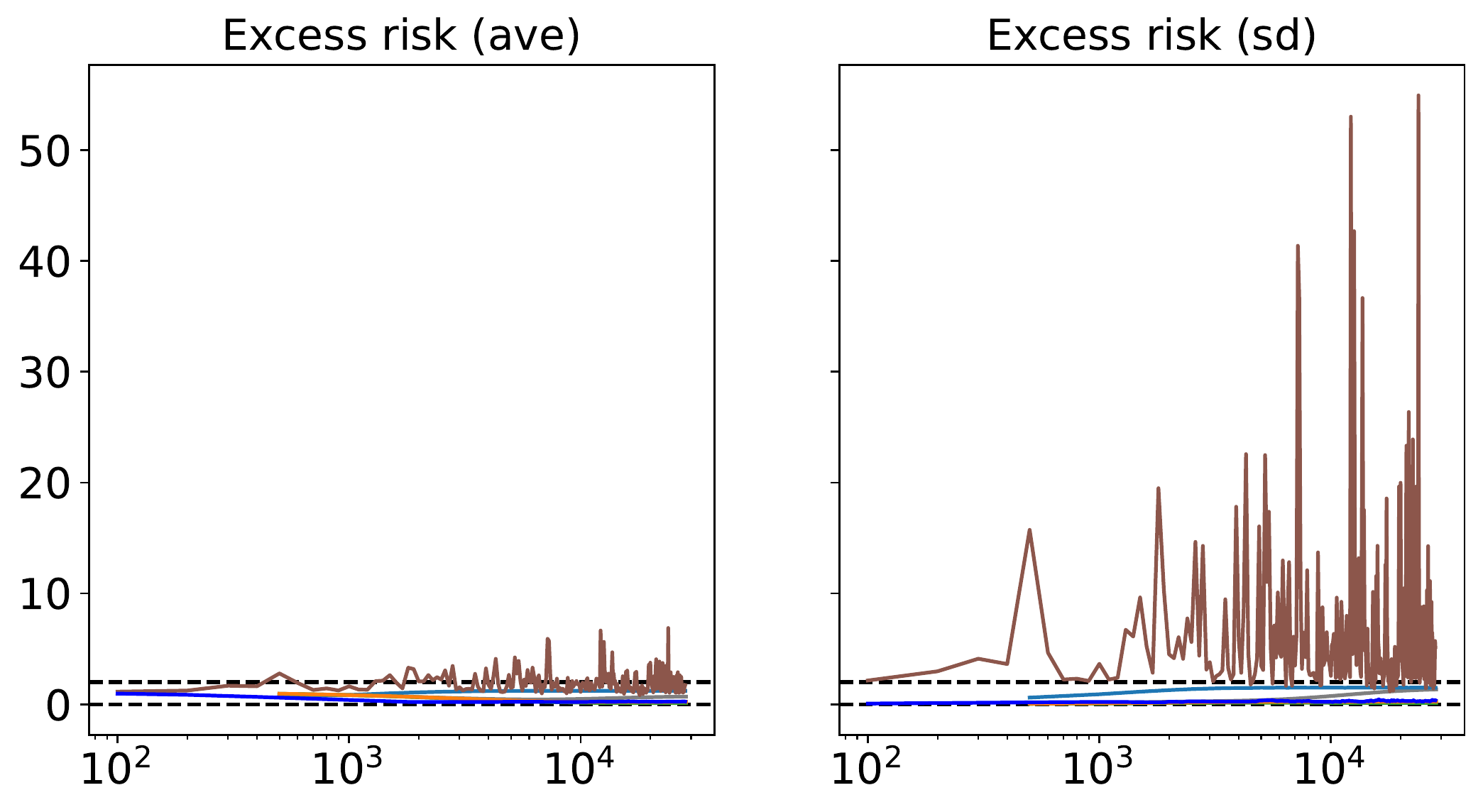}\,\includegraphics[width=0.5\textwidth]{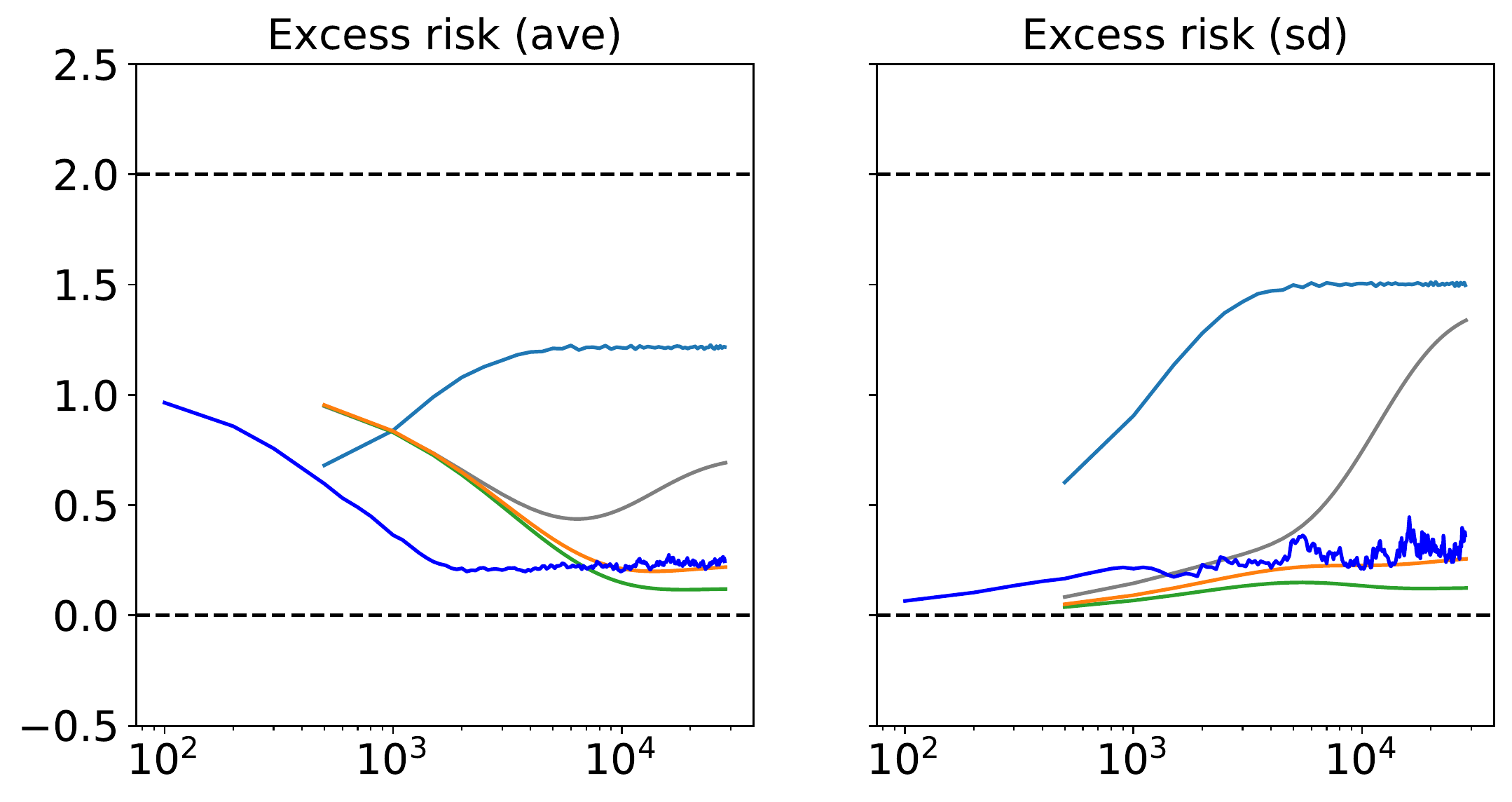}
\caption{Analogous results to Figure \ref{fig:POC_normal_sc}, for the case of log-Normal noise. Note for the zoomed-in plots on the right, we have removed the volatile SGD trajectory for visibility.}
\label{fig:POC_lognormal_sc}
\end{figure}

\paragraph{(E2) Statistical error in high dimensions}

Next, we consider how a larger number of parameters $d$ impacts the \emph{statistical} error of the methods being compared, in contrast with the \emph{computational} error (e.g., errors in Table \ref{table:sc_compare} free of $n$, shrinking as $T$ grows). To do so, we consider a range of $2 \leq d \leq 1024$ with fixed $n=2500$, and we control the initialization error for all methods such that $\|\what_{0}-\wstar\| \leq C$ for a constant $C$ that does not depend on $d$, by initializing as $\what_{0,j}=\wstar_{j}+\text{Uniform}[-c,+c]/\sqrt{d}$ for each $j \in [d]$. Furthermore, we let the batch methods run for many iterations with a gradient budget of $100n$, whereas we now restrict \texttt{DC-SGD} (Algorithm \ref{algo:DandC_SGD}) to just $2n$. This is done to ensure that the computational error terms of batch methods are sufficiently small. The number of independent trials here is $250$, and the noise distribution settings are as discussed previously. In Figure \ref{fig:ERROR_D_sc}, we give box-plots of the final excess risk achieved by each method for different $d$ sizes. As a complementary result, we can also consider fixing $d$ and taking $n$ very large to evaluate the impact of non-$n$ factors in the statistical error bounds. Such a result is given in Figure \ref{fig:ERROR_N_sc}, noting that we are still considering just single pass \texttt{DC-SGD} against the many-pass batch methods.

\begin{figure}[t]
\centering
\includegraphics[width=0.5\textwidth]{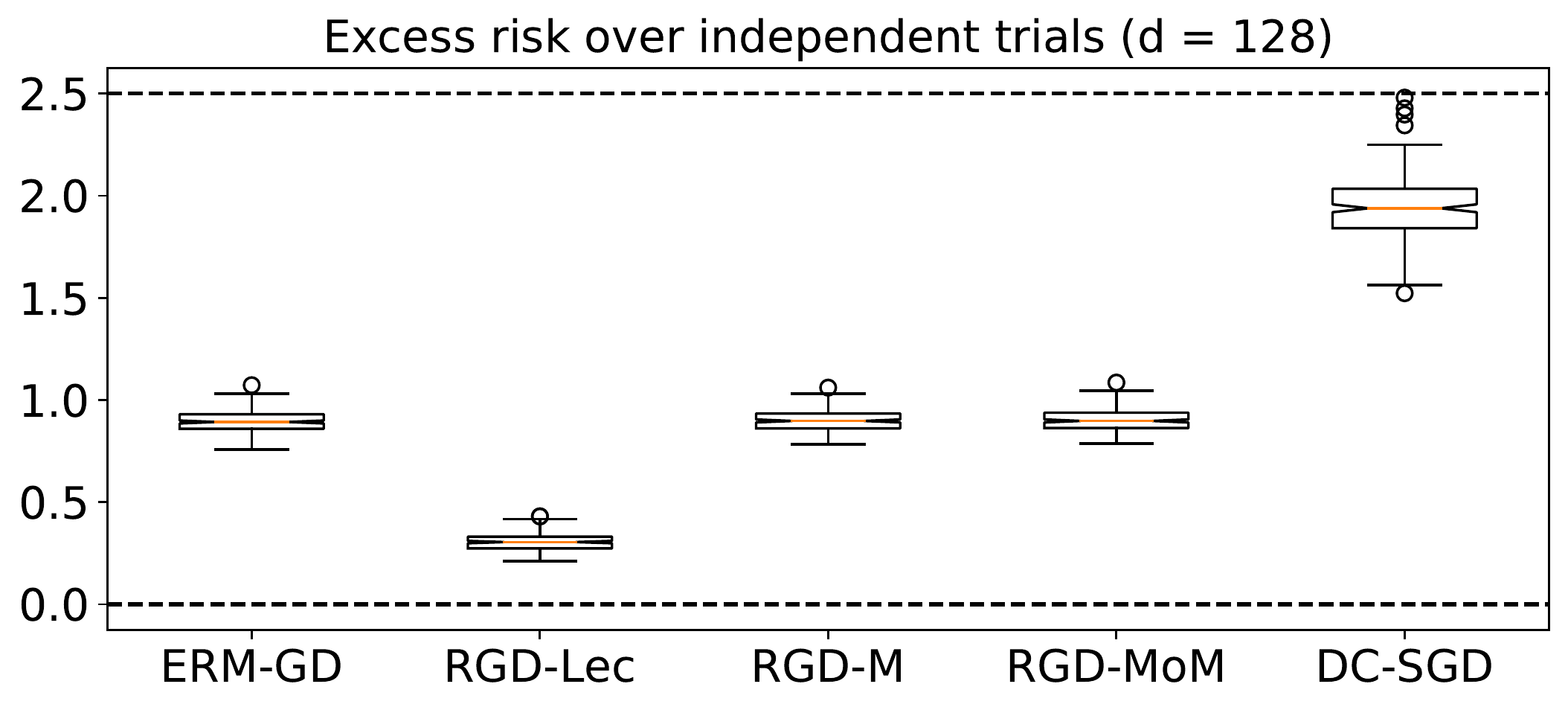}\,\includegraphics[width=0.5\textwidth]{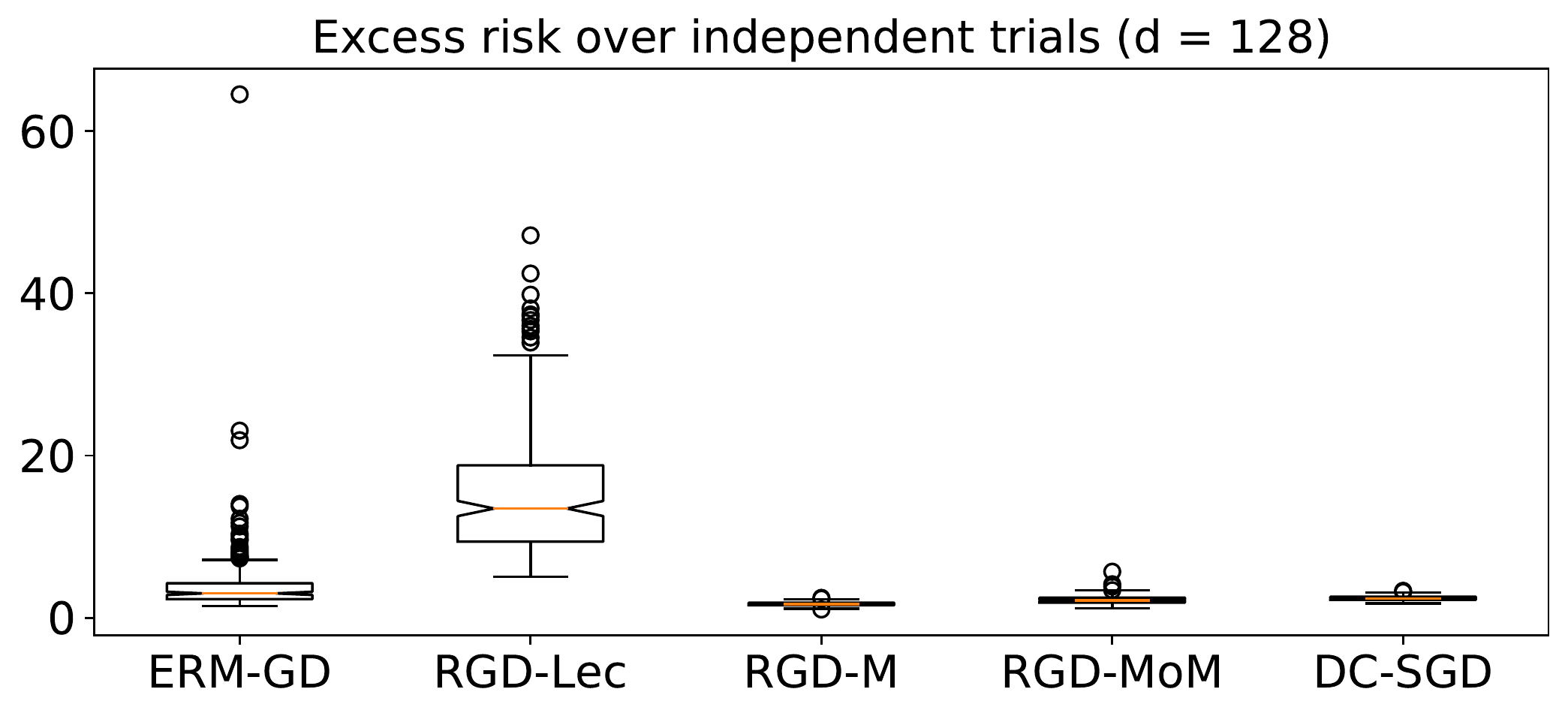}\\
\includegraphics[width=0.5\textwidth]{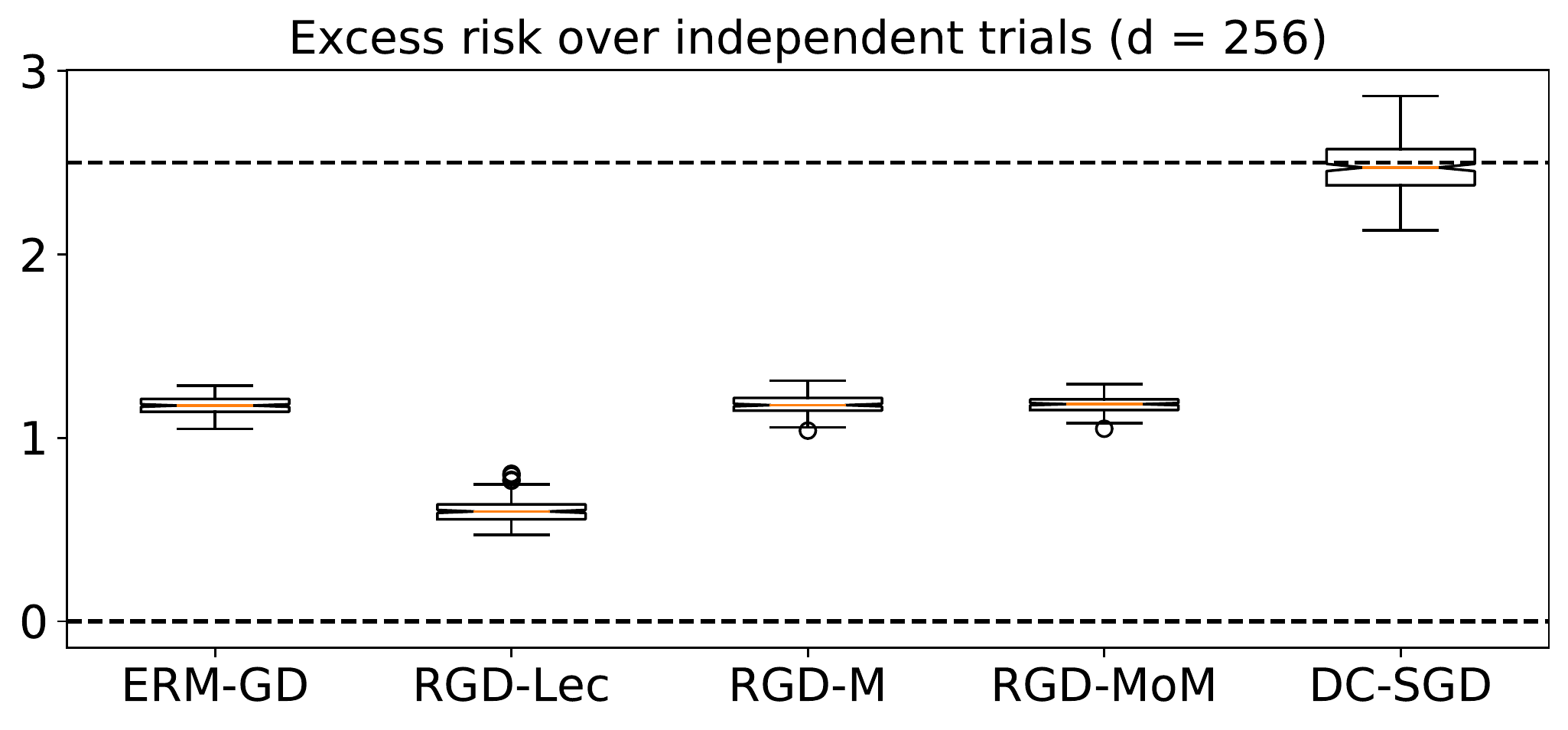}\,\includegraphics[width=0.5\textwidth]{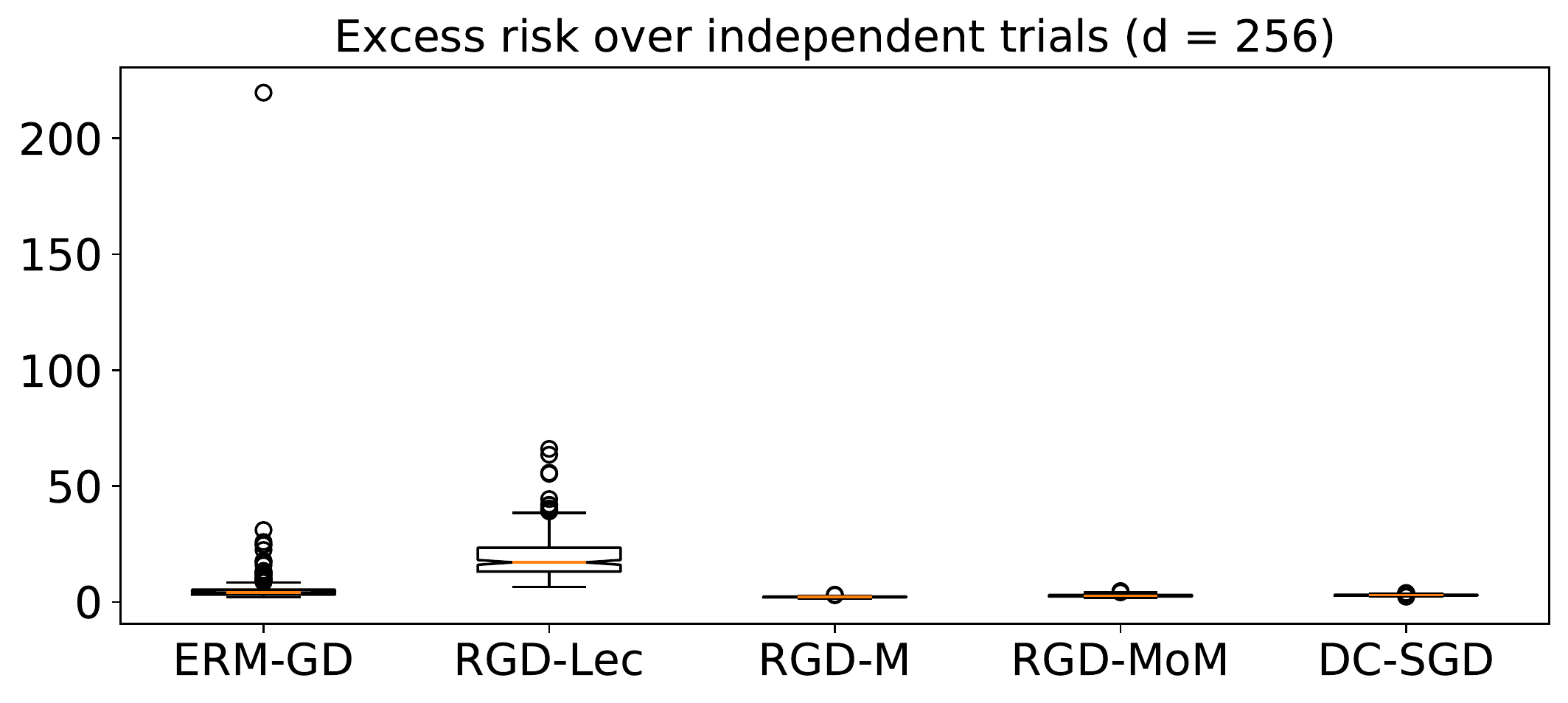}\\
\includegraphics[width=0.5\textwidth]{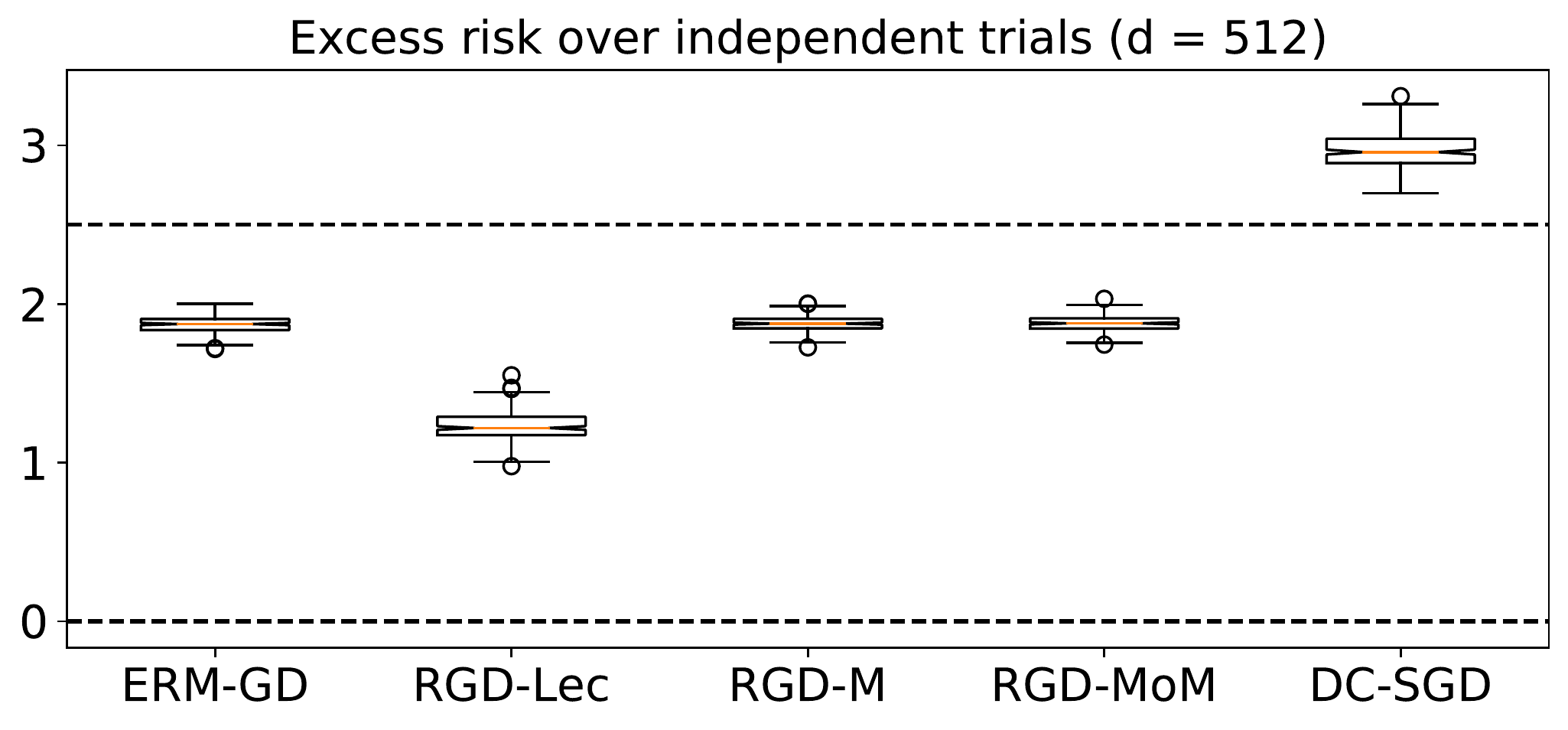}\,\includegraphics[width=0.5\textwidth]{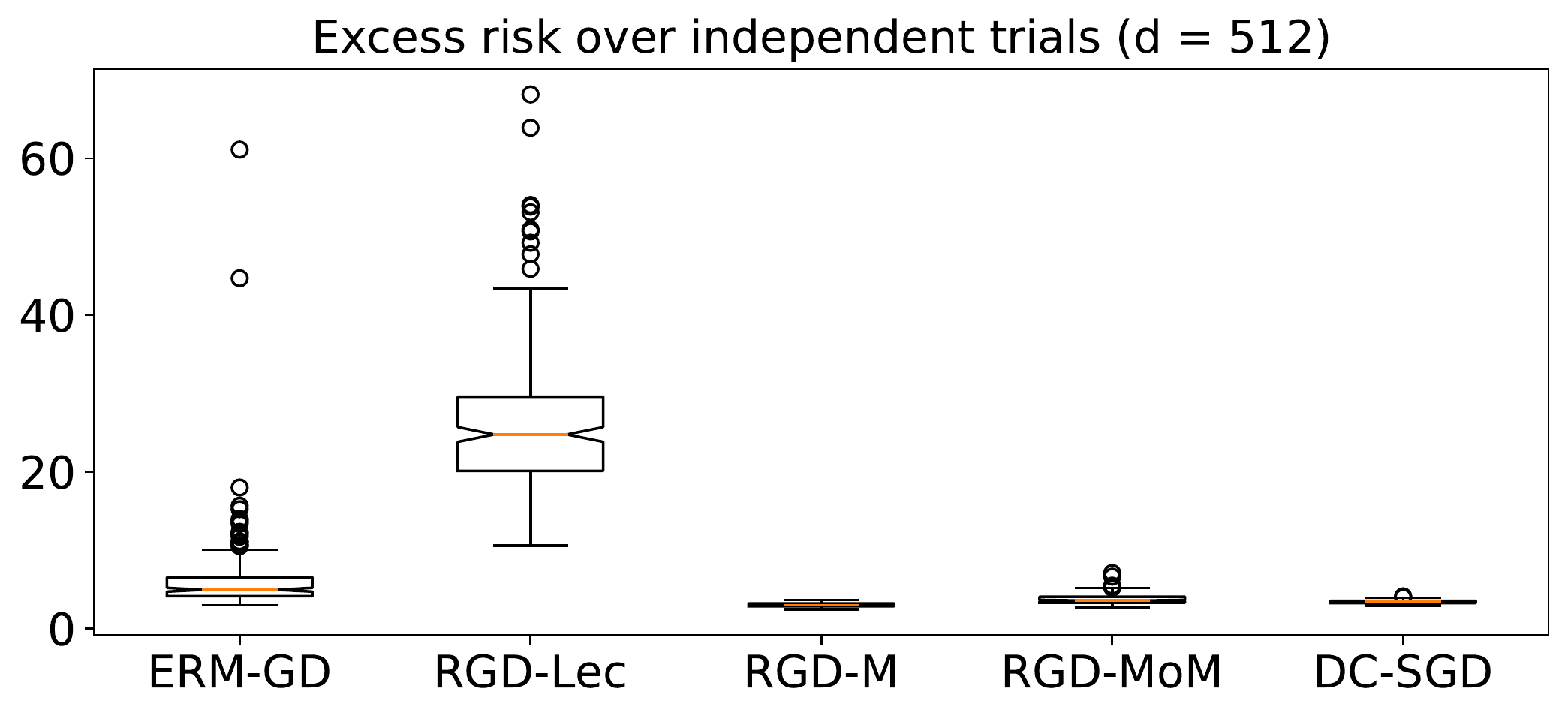}\\
\includegraphics[width=0.5\textwidth]{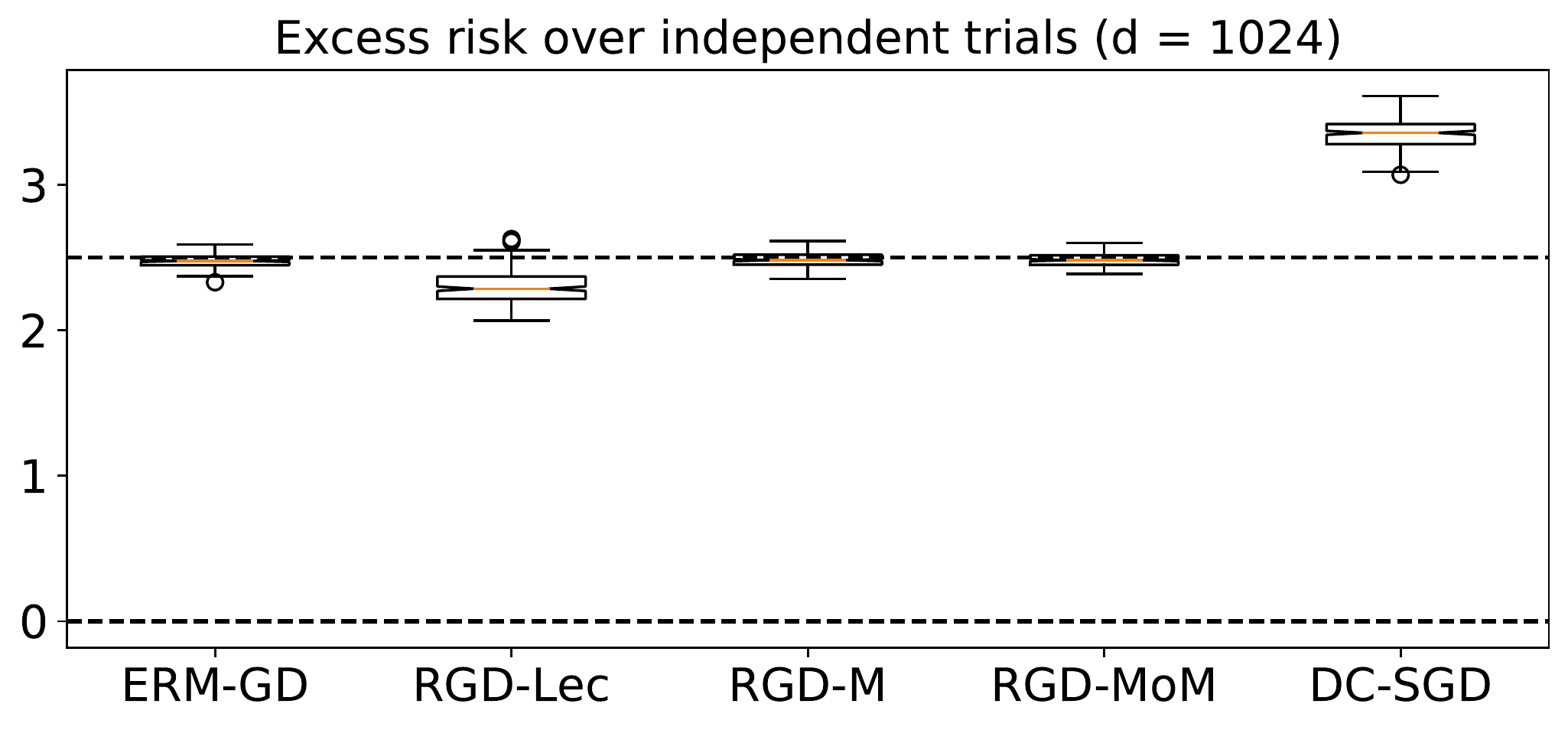}\,\includegraphics[width=0.5\textwidth]{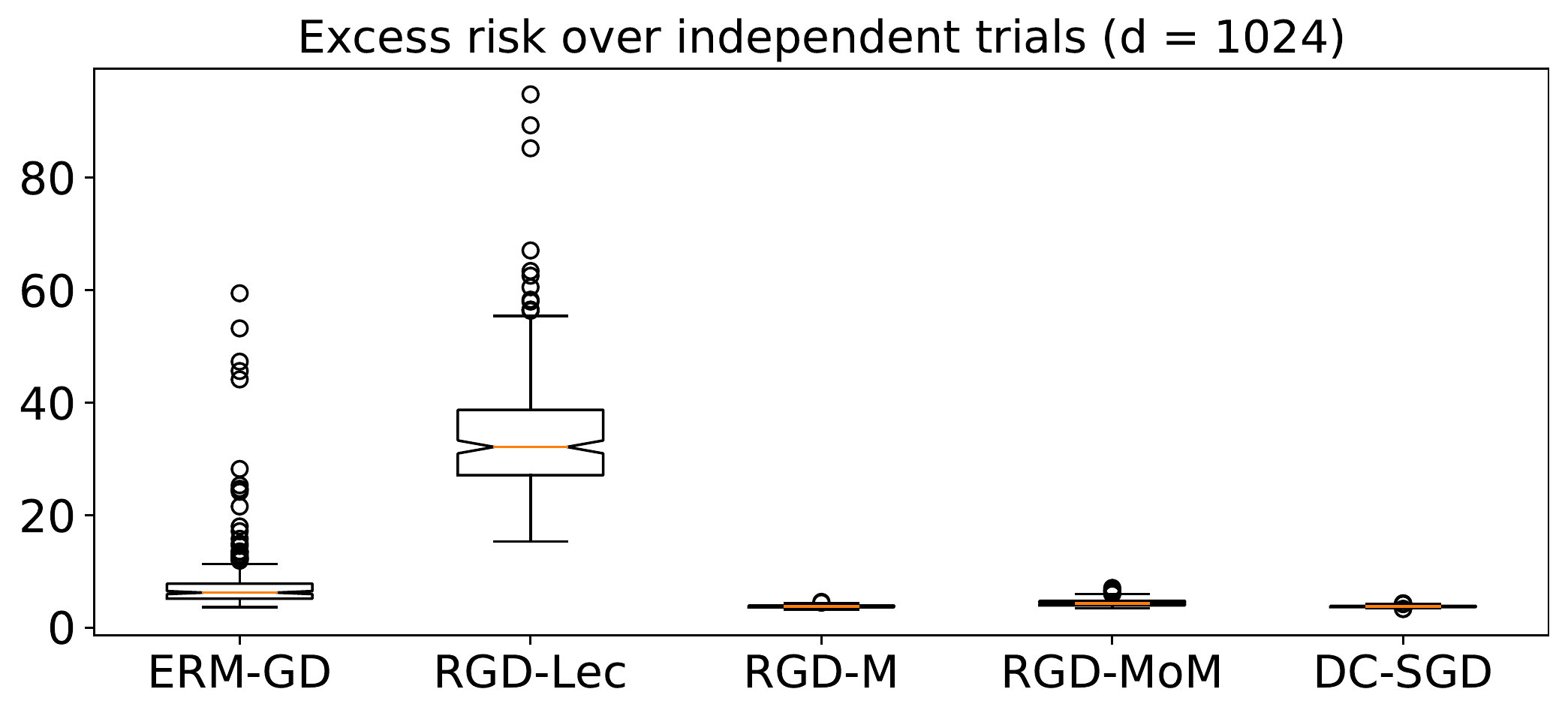}
\caption{Excess risk for many-pass batch methods and two-pass \texttt{DC-SGD}. Dimension settings shown are $d \in \{128, 256, 512, 1024\}$. Left column: Normal noise (with dashed horizontal rule fixed to highlight small changes). Right column: log-Normal noise.}
\label{fig:ERROR_D_sc}
\end{figure}

\clearpage

\paragraph{(E3) Actual computation times}

It is natural to consider how much \emph{time} is actually required to achieve the results given above, for both the many-pass batch methods and the single-pass \texttt{DC-SGD} (Algorithm \ref{algo:DandC_SGD}), in particular when the $k$ sub-processes used to compute the $\what^{(j)}$ in Algorithm \ref{algo:DandC_SGD} are run in parallel. Once again for all $k$-dependent methods we have fixed $k$ just as in previous experiments, and consider two types of tests. First, $n$ and $d$ move together, with $n=4000d$ and $2 \leq d \leq 64$. Second, $n=2500$ is fixed, and dimension ranges over $2 \leq d \leq 1024$ as before. Stopping conditions based on budget constraints are precisely as in the experiments described in the previous paragraph. We measure computation time of each experiment using the Python module \texttt{time} as follows.\footnote{Documentation: \url{https://docs.python.org/3/library/time.html}.} For each method, we record the time $\tau_{0}$ immediately after $\what_{0}$ is generated and passed to the learning algorithm, and the time $\tau_{1}$ immediately after the stopping condition \texttt{cost}$\geq$\texttt{budget} is achieved. Computation time is then defined as simply $\tau_{1}-\tau_{0}$. We run $250$ independent trials, and compute the median times for each method. For the parallel implementation of \texttt{DC-SGD} (Algorithm \ref{algo:DandC_SGD}), we use the Python module \texttt{multiprocessing} to allocate each SGD sub-process to independent worker processes that can be run in parallel.\footnote{Documentation: \url{https://docs.python.org/3.8/library/multiprocessing.html}. Specifically, we generate an instance of the \texttt{Pool} class, iterating over $k$ worker functions that return $\{\what^{(1)},\ldots,\what^{(k)}\}$.} We note that computation of ending time $\tau_{1}$ for \texttt{DC-SGD} comes after the $\merge$ step, and thus all the overhead due to \texttt{multiprocessing} is included in the computation times recorded. From the perspective of making a fair comparison, we have made every effort to ensure all algorithms are implemented as efficiently as possible. The median times for each method in both experimental settings are shown in Figure \ref{fig:lognormal_times_sc}.

\begin{figure}[h]
\centering
\includegraphics[width=0.5\textwidth]{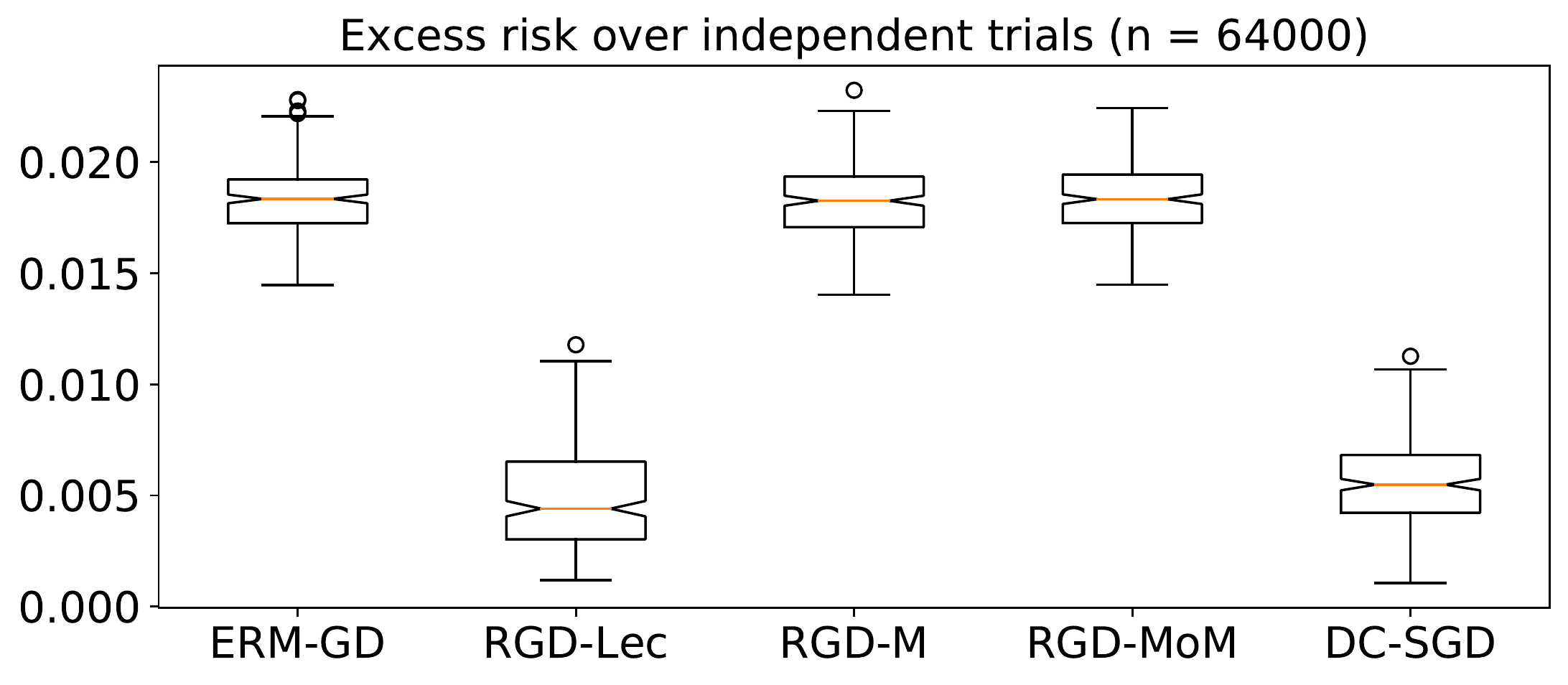}\,\includegraphics[width=0.5\textwidth]{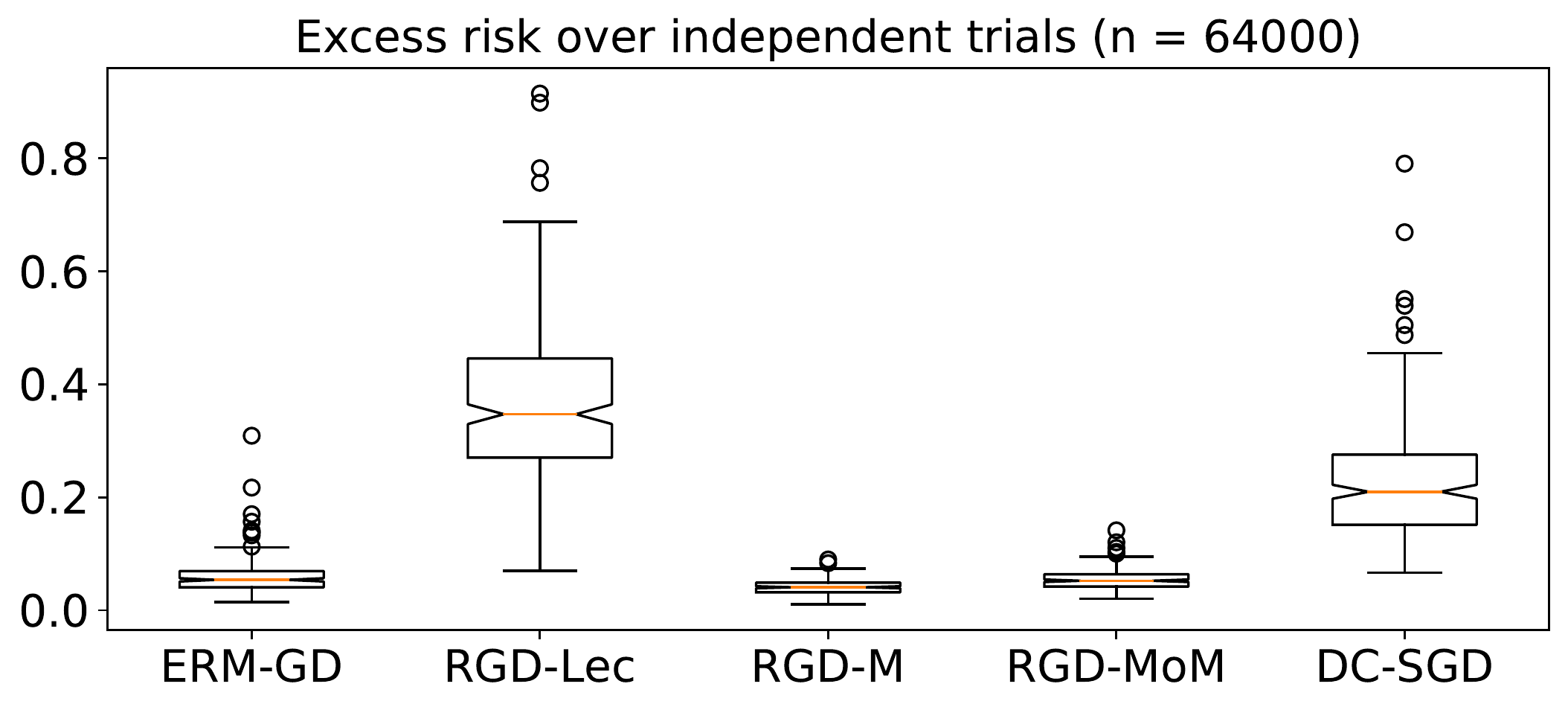}
\caption{Excess risk for many-pass batch methods and two-pass \texttt{DC-SGD} when $n \gg d$. In particular, $n=64000$ and $d=16$. Left column: Normal noise. Right column: log-Normal noise.}
\label{fig:ERROR_N_sc}
\vspace{0.25cm}
\includegraphics[width=0.7\textwidth]{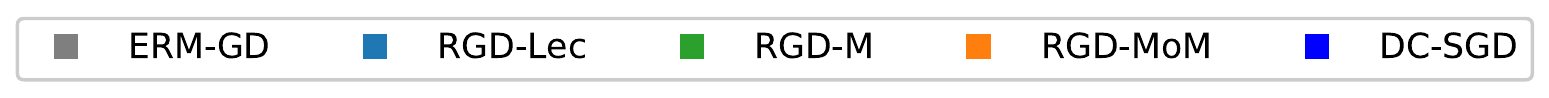}\\
\includegraphics[width=0.5\textwidth]{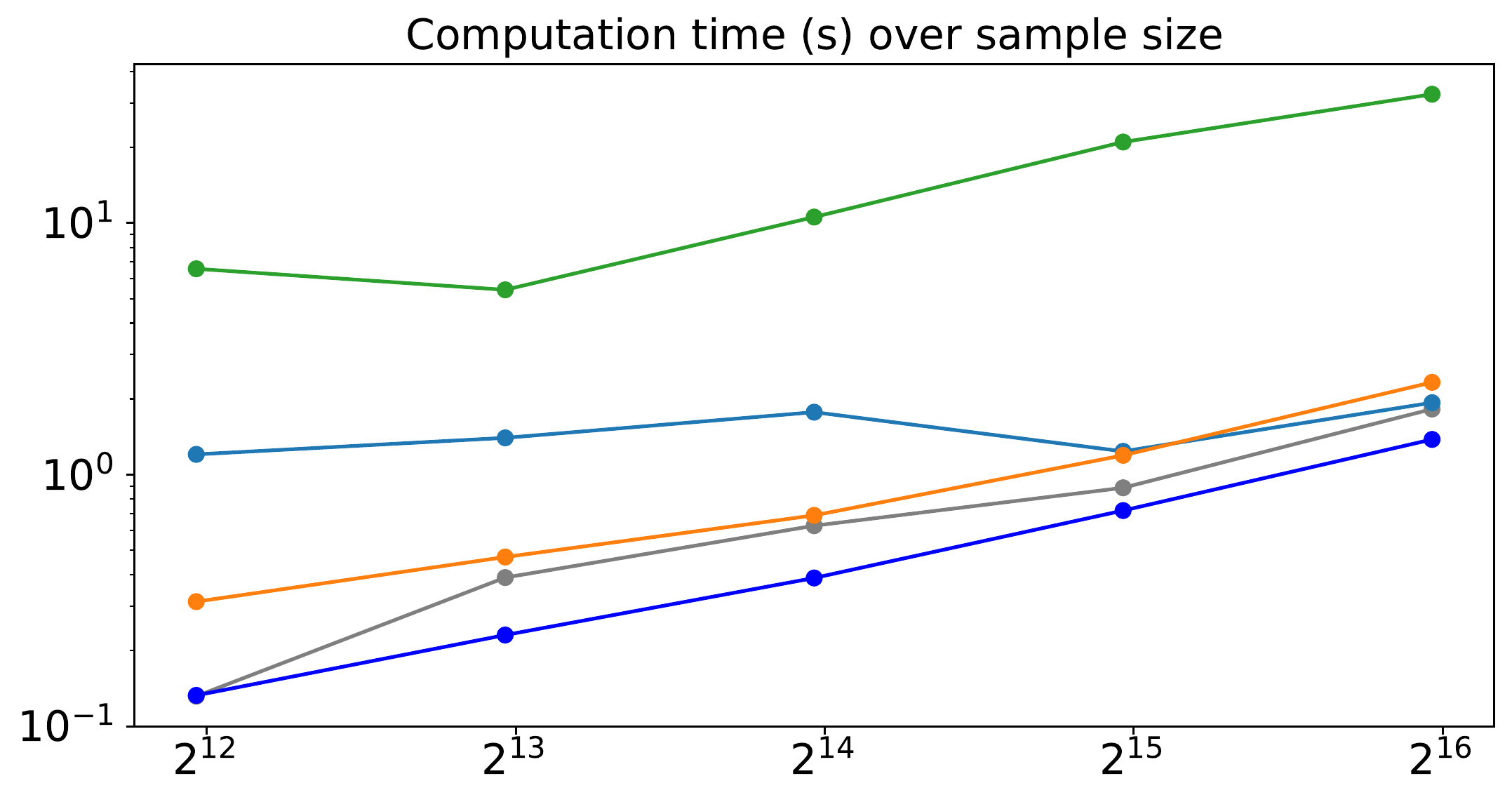}\,\includegraphics[width=0.5\textwidth]{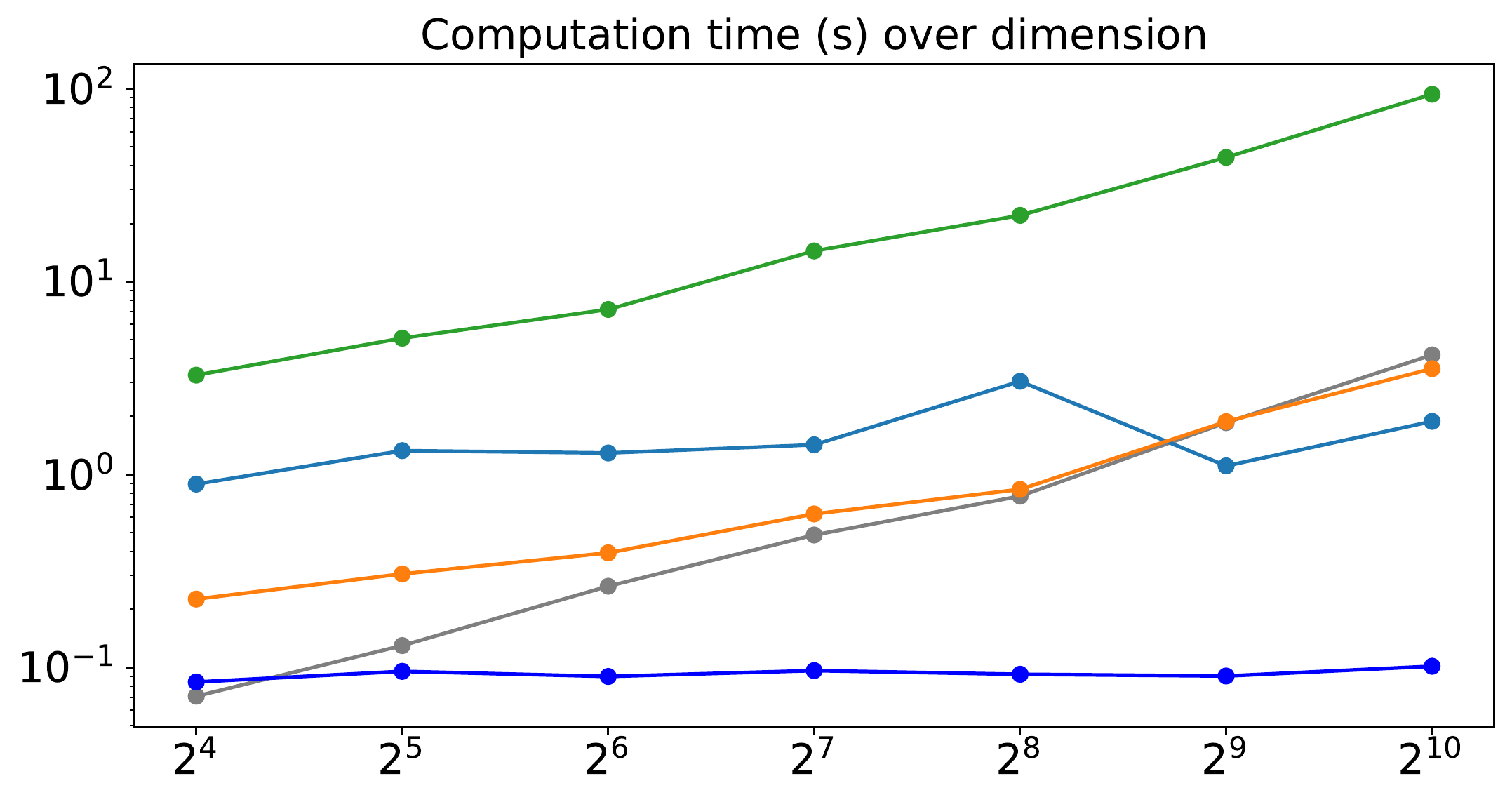}
\caption{Median computation times (log scale, base $10$) for the log-Normal noise setting, as a function of $n$  and $d$ (right; log scale, base $2$). Left: time as a function of $n$ (log scale, base $2$), with $n$ and $d$ growing together. Right: time as a function of $d$ (log scale, base $2$), with $n$ fixed.}
\label{fig:lognormal_times_sc}
\end{figure}

\clearpage

\paragraph{(E4) Impact of initialization on error trajectories}

We now return to the setting of (E1), and investigate the impact that a larger initialization error has on the resulting trajectory of each method. Recall that our baseline setup has us initializing in a coordinate-wise fashion, namely $\what_{0,j} = \wstar_{j} + \text{Uniform}[-c,+c]$. The default setting was $c=5.0$, but here we consider $c \in \{2.5, 5.0, 10.0\}$, for both Normal and log-Normal noise cases. As with (E1), the excess risk values are averaged over $100$ independent trials. To ensure the plots are legible, we choose four representative methods to highlight the key trends. Results are shown in Figure \ref{fig:INIT_sc} (we denote $c$ by \texttt{sup} in the legend).

\begin{figure}[t]
\centering
\includegraphics[width=0.5\textwidth]{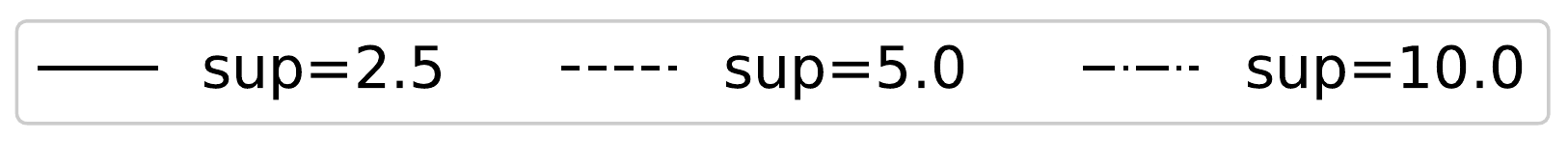}\\
\includegraphics[width=1.0\textwidth]{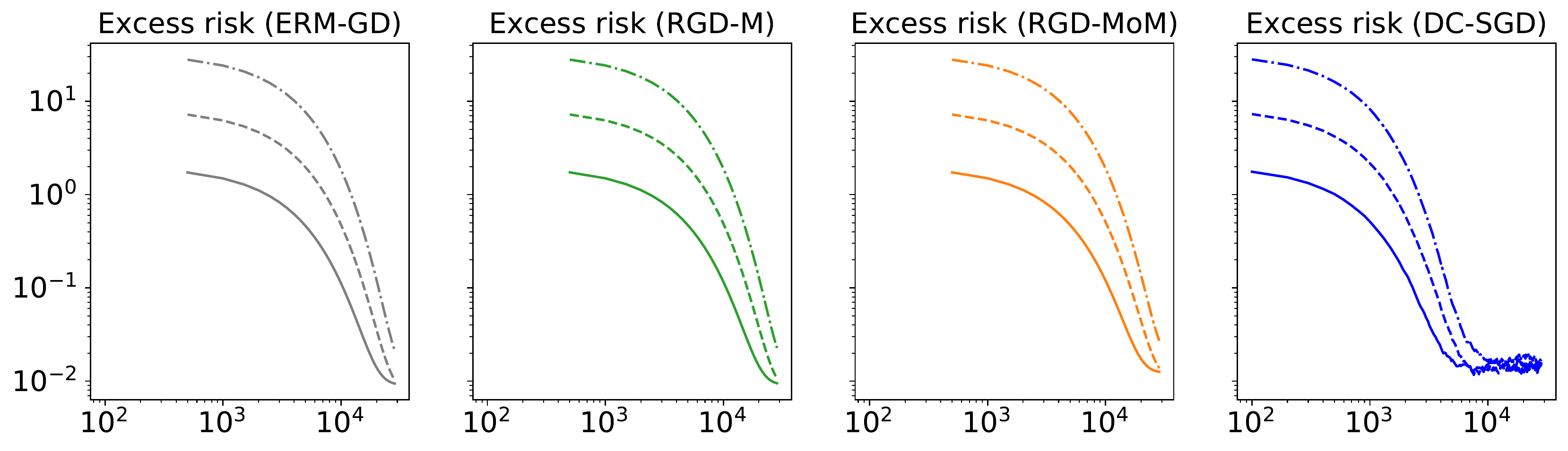}\\
\includegraphics[width=1.0\textwidth]{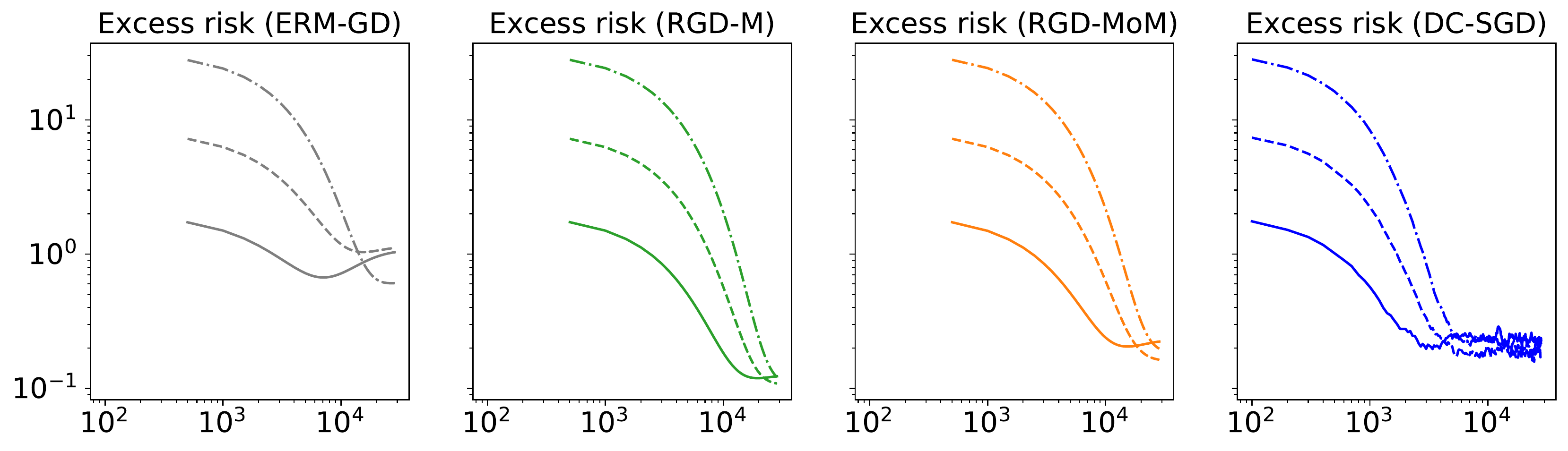}\\
\caption{Excess risk trajectories (averaged over trials) with different initialization error ranges. Top row: Normal noise. Bottom row: log-Normal noise.}
\label{fig:INIT_sc}
\end{figure}

\paragraph{(E5) Impact of noise level on error trajectories}

Continuing with the same basic setup as (E4), we keep the default initialization error range, and instead here modify the signal to noise ratio. The nature of $X$ and $\wstar$ is kept constant, so the strength of the ``signal'' $\langle \wstar,X \rangle$ does not change. Recall that setting $Y \sim \text{Normal}(0,b^{2})$, we consider two cases of additive noise, namely where $E = Y - \exx Y$ (Normal case) and $E = \ct{e}^{T} - \exx \ct{e}^{Y}$ (log-Normal case). In the Normal case, we take $b \in \{1.5, 2.2, 2.4\}$. In the log-Normal case, we take $b \in \{1.25, 1.75, 1.90\}$. Starting from small to large, we denote these levels as $\{\texttt{low},\texttt{med},\texttt{high}\}$. As with our earlier experimental settings, we have selected the parameters for these three ``noise levels'' such that at each level, the inter-quartile range of $E$ is approximately equal for both the Normal and log-Normal cases. As before, we select four representative methods and show the impact of noise level on excess risk, averaged over $100$ independent trials. Results are given in Figure \ref{fig:DIST_sc}.

\begin{figure}[t]
\centering
\includegraphics[width=0.4\textwidth]{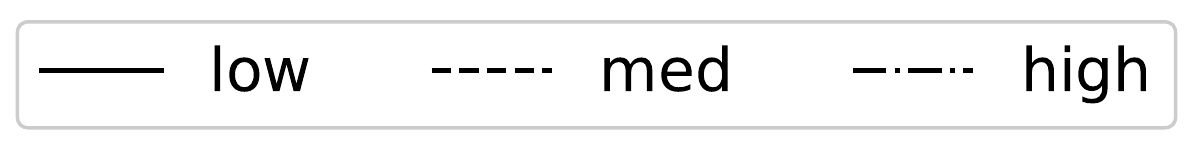}\\
\includegraphics[width=1.0\textwidth]{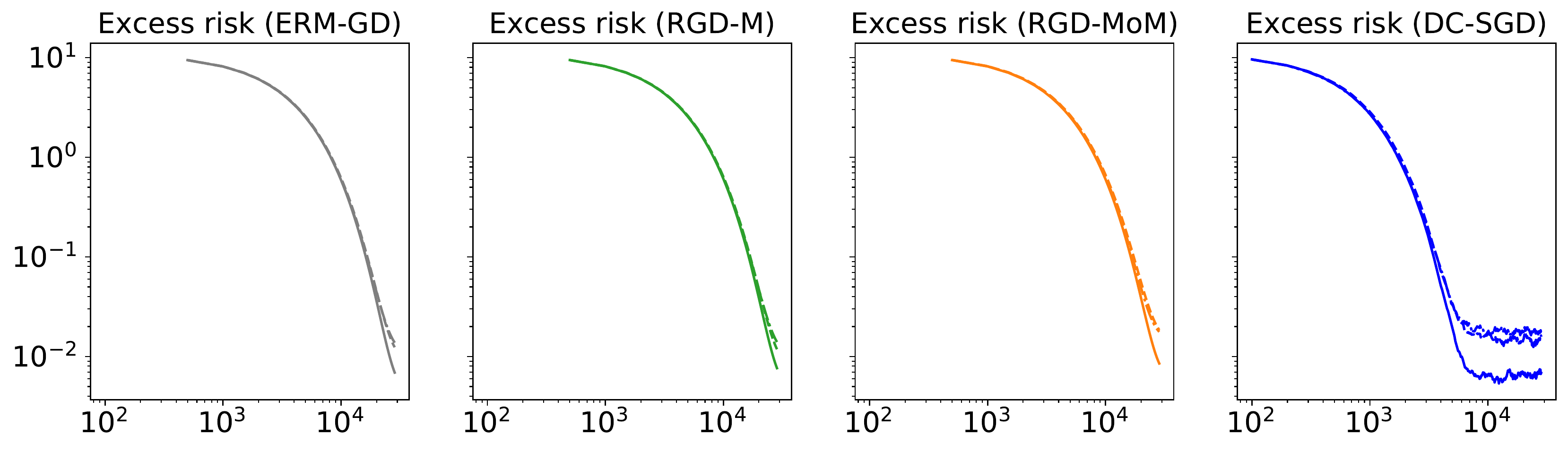}\\
\includegraphics[width=1.0\textwidth]{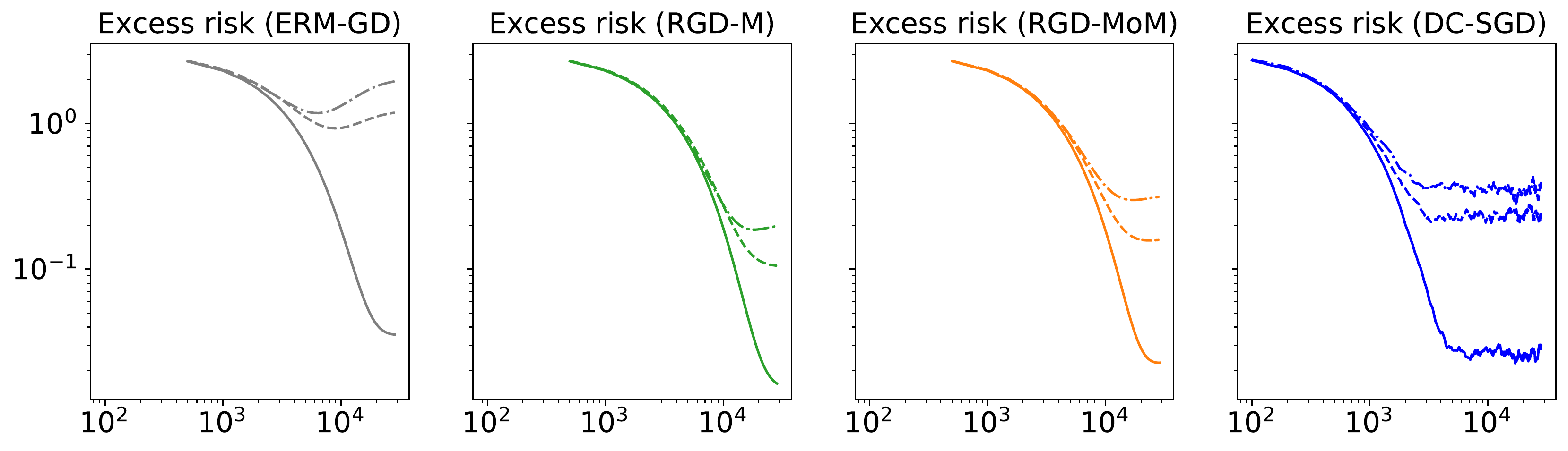}\\
\caption{Noise levels and excess risk trajectories (averaged over trials). Top row: Normal noise. Bottom row: log-Normal noise.}
\label{fig:DIST_sc}
\end{figure}

\paragraph{Discussion of results}

As an overall take-away from the preceding empirical test results, it is clear that even with no fine-tuning of algorithm parameters, it is possible for Algorithm \ref{algo:DandC_SGD} to achieve performance comparable to robust gradient descent methods using far less computational resources. Clearly, even when the underlying sub-processes used by Algorithm \ref{algo:DandC_SGD} are very noisy, only a few passes over the data are necessary to match the best-performing RGD methods, both on average and in terms of between-trial variance (Figures \ref{fig:POC_normal_sc}--\ref{fig:POC_lognormal_sc}). Furthermore, without any algorithm adjustments, this robustness holds over changes to initialization error and the signal/noise ratio (Figures \ref{fig:INIT_sc}--\ref{fig:DIST_sc}). It is clear that due to sample splitting, our procedure can take a small hit in terms of statistical error as the sample size grows very large (Figure \ref{fig:ERROR_N_sc}), but makes up for this in scalability. As the dimensionality of the task grows, under heavy-tailed data the proposed procedure establishes an even more stark advantage over competing methods (Figure \ref{fig:ERROR_D_sc}), at only a small fraction of the computational cost (Figure \ref{fig:lognormal_times_sc}). These initial results are encouraging, and additional tests looking at basic principles and data-driven strategies for optimizing the number of partitions $k$ is a natural point of interest.

\subsection{Controlled simulations, without strong convexity}\label{sec:empirical_nonsc}

To study how the theoretical insights obtained in the previous section play out in practice, we carried out a series of tightly controlled numerical tests. The basic experimental design strategy that we employ is to calibrate all the methods (learning algorithms) of interest to achieve good performance under a particular learning setup, and then we systematically modify characteristics of the learning tasks, leaving the methods fixed, to observe how performance changes in both an absolute and relative sense. Viewed from a high level, the main points we address can be categorized as follows:
\begin{itemize}
\item[\textbf{(E1)}] How do error trajectories of baseline methods change via robust validation?

\item[\textbf{(E2)}] How does relative performance change in high dimensions without strong convexity?

\item[\textbf{(E3)}] How do actual computation times compare as $n$ and/or $d$ grow?

\item[\textbf{(E4)}] Can robust validation be replaced by cross-validation?
\end{itemize}
We proceed by giving additional details on our experimental setting, before taking up the key experiments just listed one at a time.

\paragraph{Experimental setup}

Our basic setup is just as in the previous sub-section (noisy convex minimization), but with new modifications made here to control the degree of strong convexity, among other experimental parameters. With this design it is easy to allow both the losses and partial derivatives to be heavy-tailed, while still satisfying the key technical assumptions of Theorem \ref{thm:smooth_SGDave_roboost}, namely $\parasm_{1}$-smooth $\risk_{\ddist}$ and gradients with $\sigma_{G,\ddist}$-bounded variance. Furthermore, since we are interested in the case where strong convexity may not hold, this experimental design means that the strong convexity parameter $\parasc$ of $\risk_{\ddist}$ is at our control, allowing us to construct many flat directions, and observe algorithm performance as $\parasc \downarrow 0$. All tests and methods are implemented using Python (ver.~3.8), chiefly relying upon the \texttt{numpy} library (ver.~1.18).

For clarity of results, we limit our comparisons to two main families of distributions for $\ddist$, namely those in which the loss $\loss(w;Z)$ contains a Normal noise term, and those in which it contains a log-Normal noise term. In all cases, this noise is centered, and controlled to have nearly equal signal/noise ratios, where the noise level is measured by the width of the interquartile range of the additive noise term, just as in the previous sub-section. The procedure to be evaluated is Algorithm \ref{algo:DandC_valid}, denoted \texttt{RV-SGDAve}, which has been implemented with $\valid$ set to be the Catoni-type M-estimator \citep{catoni2012a} (Algorithm \ref{algo:valid_cat}). Benchmark methods against which we compare are implemented exactly as in the previous sub-section.

\paragraph{(E1) How do error trajectories of baseline methods change via robust validation?}

Before we look at the impact of $\risk_{\ddist}$ having weak convexity, we begin with a nascent investigation of the basic workings of the robust validation procedure of interest. We run $100$ independent trials for both the Normal and log-Normal settings described previously, with $d=2$, $n=500$, and $1$-strongly convex $\risk_{\ddist}$. Here we let all methods run with a fixed ``budget'' of $40n\sqrt{d}$, where the ``cost'' is measured by gradient computations, i.e., cost is incremented by one when $\nabla \loss(w;Z_{i})$ is computed at any $w$ for any $i \in [n]$. Naturally, this means Algorithm \ref{algo:DandC_valid} will be run for multiple passes over the data, meaning that the behavior after the first pass takes us, strictly speaking, beyond the scope of Theorem \ref{thm:smooth_SGDave_roboost}, a natural point of empirical interest. In Figures \ref{fig:baseline_check_normal_nonsc}--\ref{fig:baseline_check_lognormal_nonsc}, we show how the baseline stochastic methods change when being passed through a robust validation procedure such as is used in our Algorithm \ref{algo:DandC_valid}. Here \texttt{RV-SGDAve} is precisely Algorithm \ref{algo:DandC_valid}, where \texttt{RV-SGD} denotes the same procedure \emph{without} averaging the SGD sub-processes. It is natural to choose \texttt{RV-SGDAve} as a representative, and in Figure \ref{fig:proof_of_concept_nonsc}, we compare just \texttt{RV-SGDAve} against the modern RGD methods.

\begin{figure}[th!]
\centering
\includegraphics[width=0.4\textwidth]{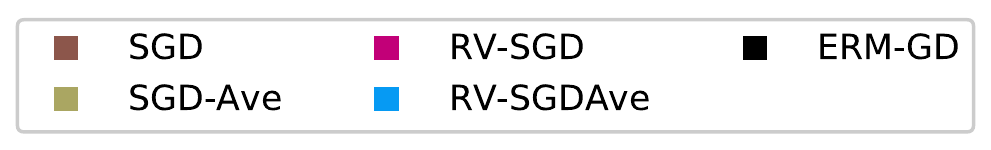}\\
\includegraphics[width=0.49\textwidth]{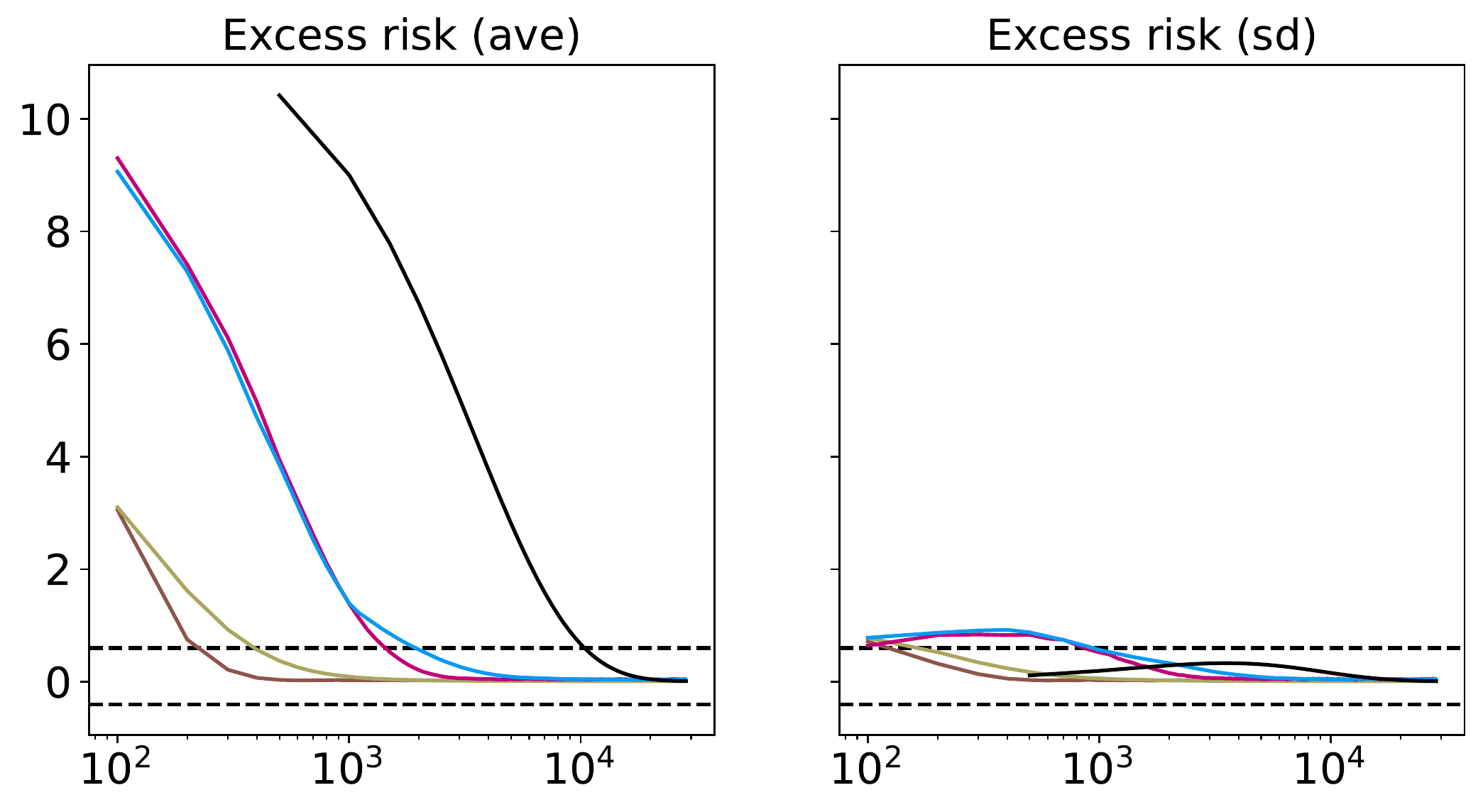}\,\includegraphics[width=0.51\textwidth]{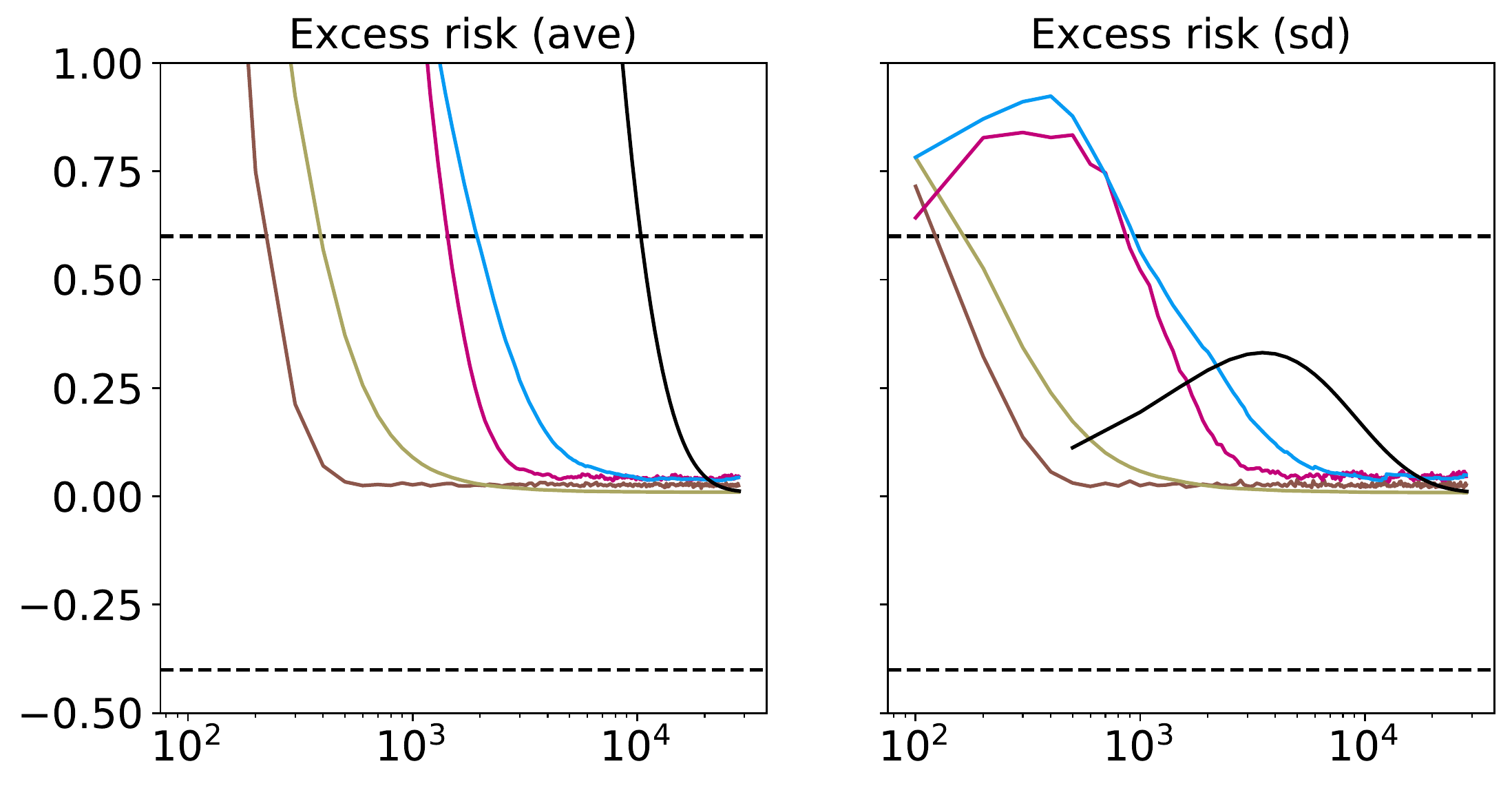}
\caption{Excess risk statistics as a function of cost in gradients (log scale, base $10$). The two right-most plots zoom in on the region between the dashed lines in the two left-most plots.}
\label{fig:baseline_check_normal_nonsc}
\vspace{0.25cm}
\includegraphics[width=0.5\textwidth]{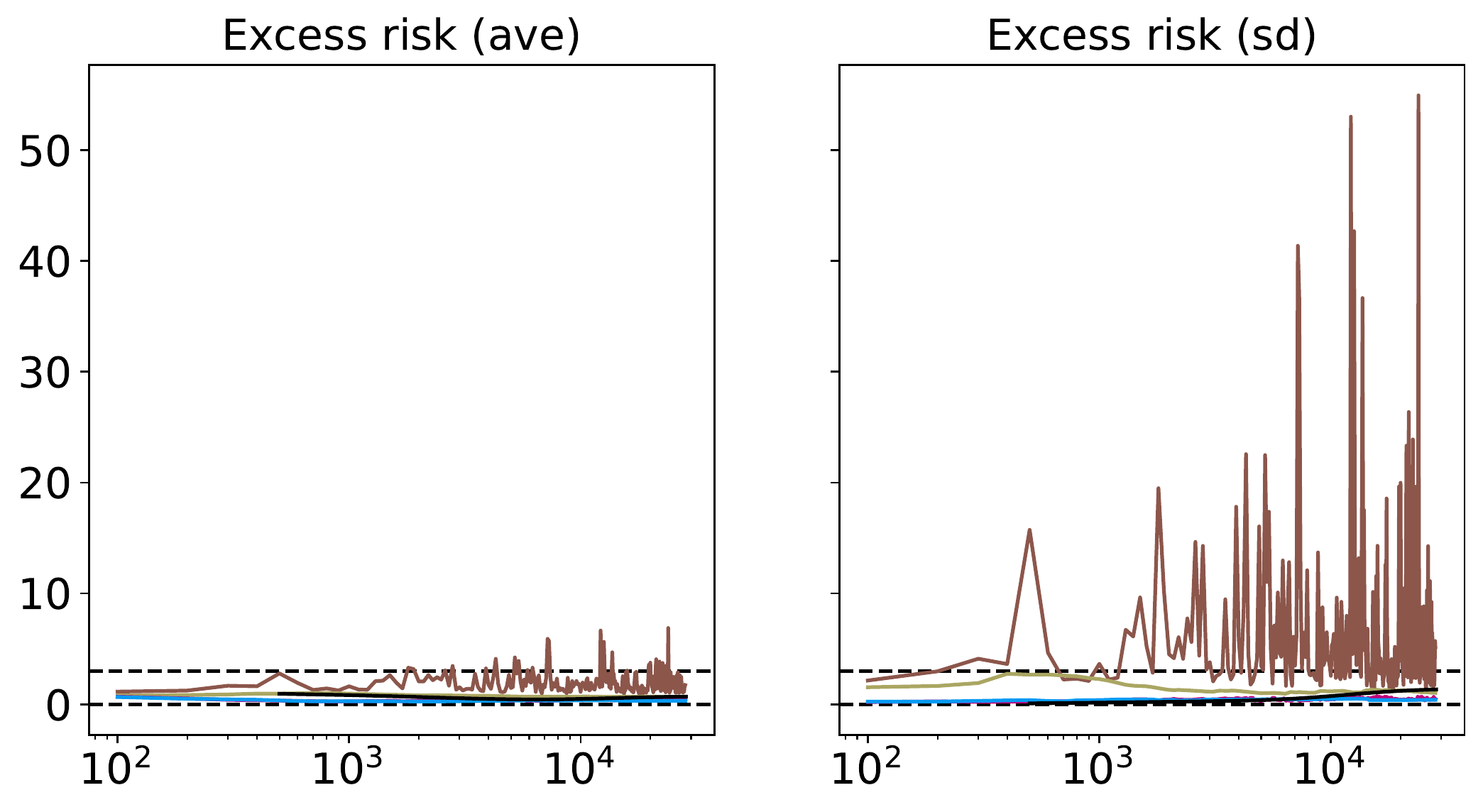}\,\includegraphics[width=0.5\textwidth]{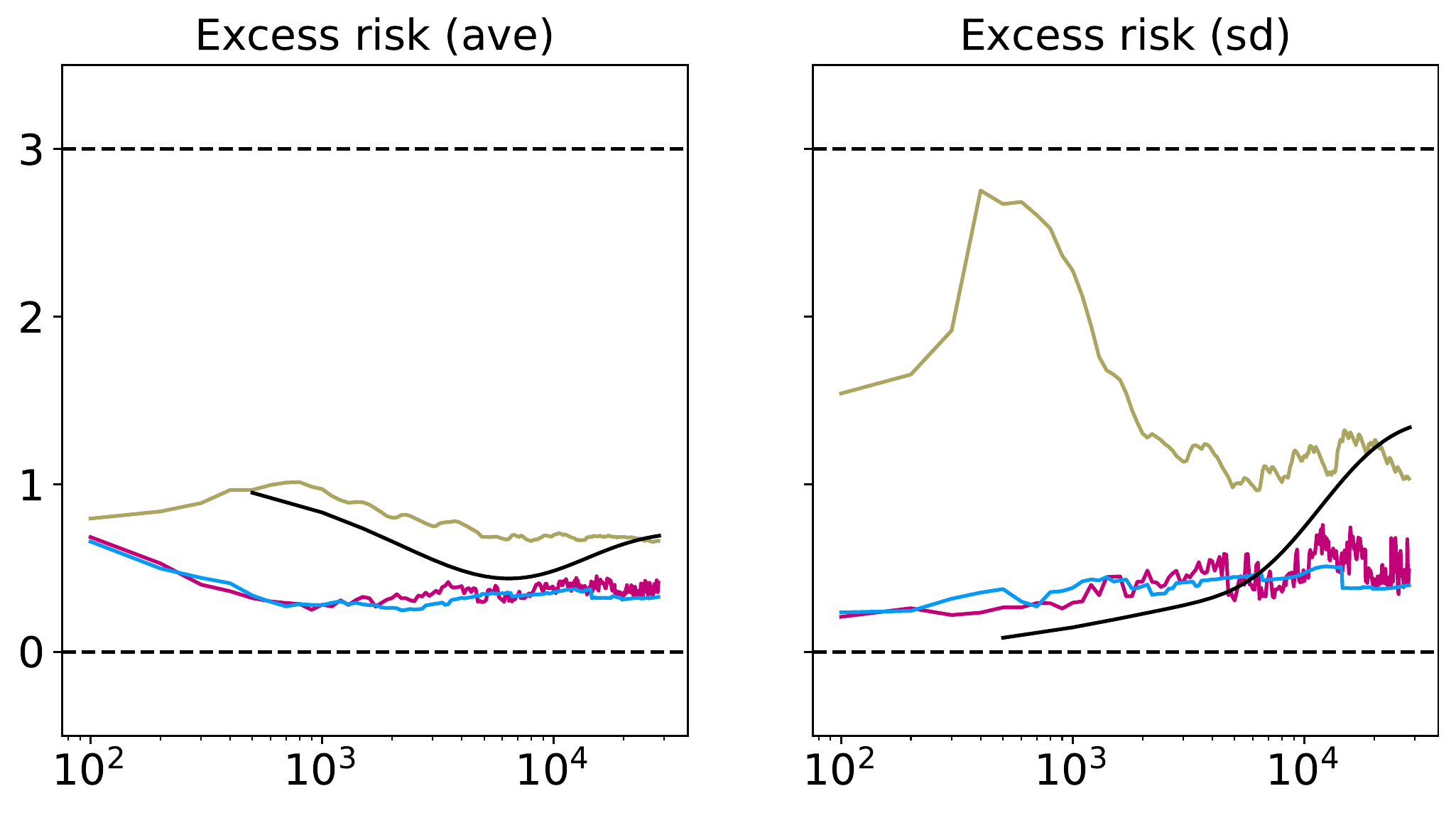}
\caption{Analogous results to Figure \ref{fig:baseline_check_normal_nonsc}, for the case of log-Normal noise.}
\label{fig:baseline_check_lognormal_nonsc}
\vspace{0.25cm}
\includegraphics[width=0.4\textwidth]{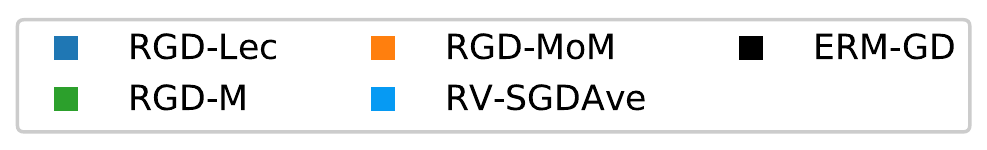}\\
\includegraphics[width=0.5\textwidth]{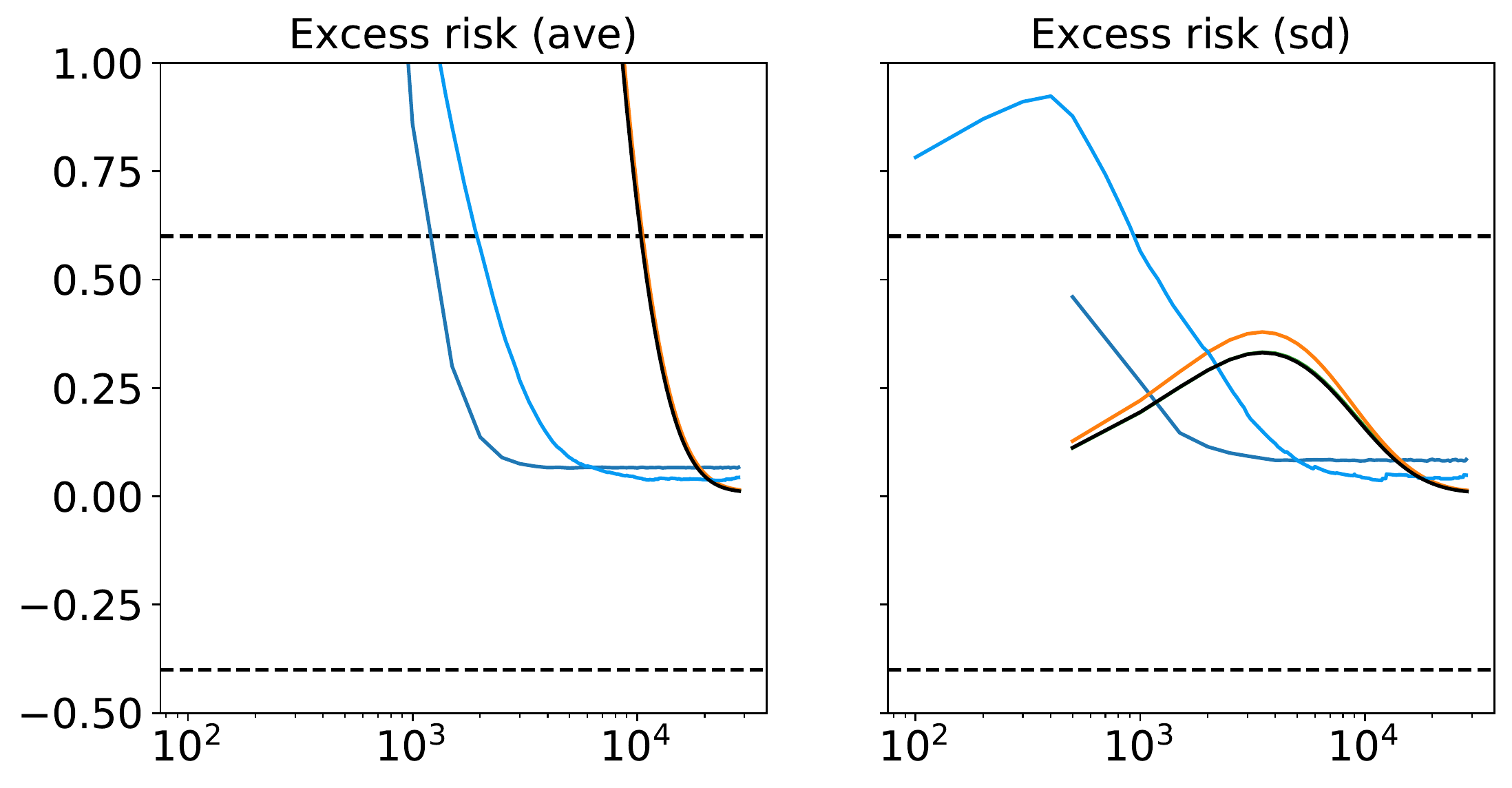}\,\includegraphics[width=0.5\textwidth]{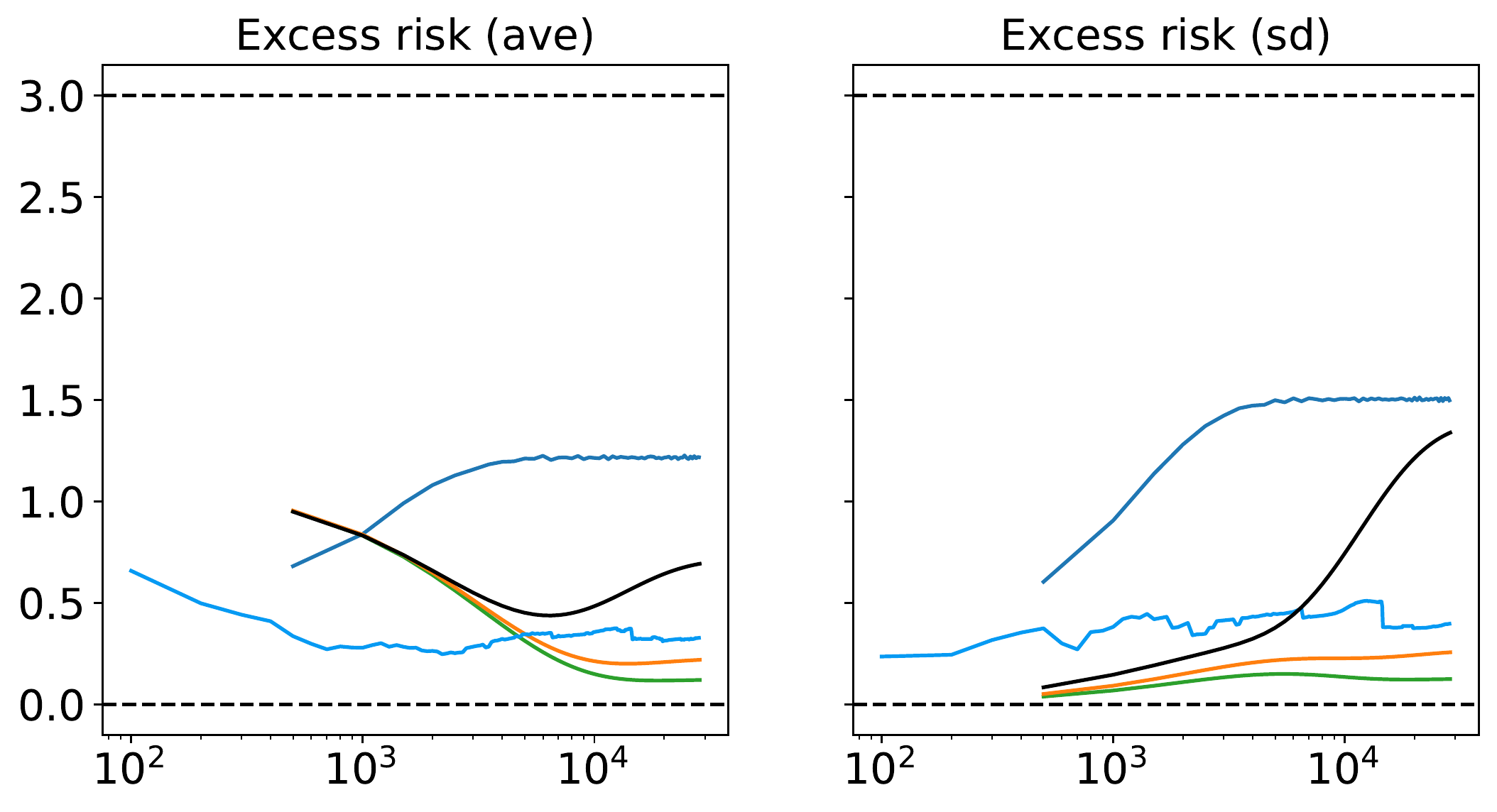}
\caption{Comparison with robust GD methods. Left: Normal case. Right: log-Normal case.}
\label{fig:proof_of_concept_nonsc}
\end{figure}

\paragraph{(E2) How does relative performance change in high dimensions without strong convexity?}

Next we look at how the competing learning algorithms perform as the number of parameters to determine increases, with $\risk_{\ddist}$ having very weak convexity in many directions. More precisely, the matrix $\Sigma$ is diagonal, and half the diagonal elements are no greater than $10^{-4}$, implying a tiny upper bound on the strong convexity parameter of $\risk_{\ddist}$. Under this setting, we look at how increasing $d$ over the range $2 \leq d \leq 1024$, with fixed sample size $n=2500$ impacts algorithm performance. We run $250$ independent trials, and for each trial record performance achieved by each method once it has spent its budget, again measured in gradient computations. Batch methods are given a large budget of $100n$. In contrast, with the previous experiments, here we only let \texttt{RV-SGDAve} (Algorithm \ref{algo:DandC_valid}) take one pass over the data for initialization, and one pass for learning, so a budget of just $2n$. This aligns more precisely with the setting of Theorem \ref{thm:smooth_SGDave_roboost}. Noise distribution settings are as previously introduced. In Figure \ref{fig:ERROR_DMU_nonsc}, we give box-plots of the final excess risk achieved by each method for different $d$ sizes.

\begin{figure}[t]
\centering
\includegraphics[width=0.5\textwidth]{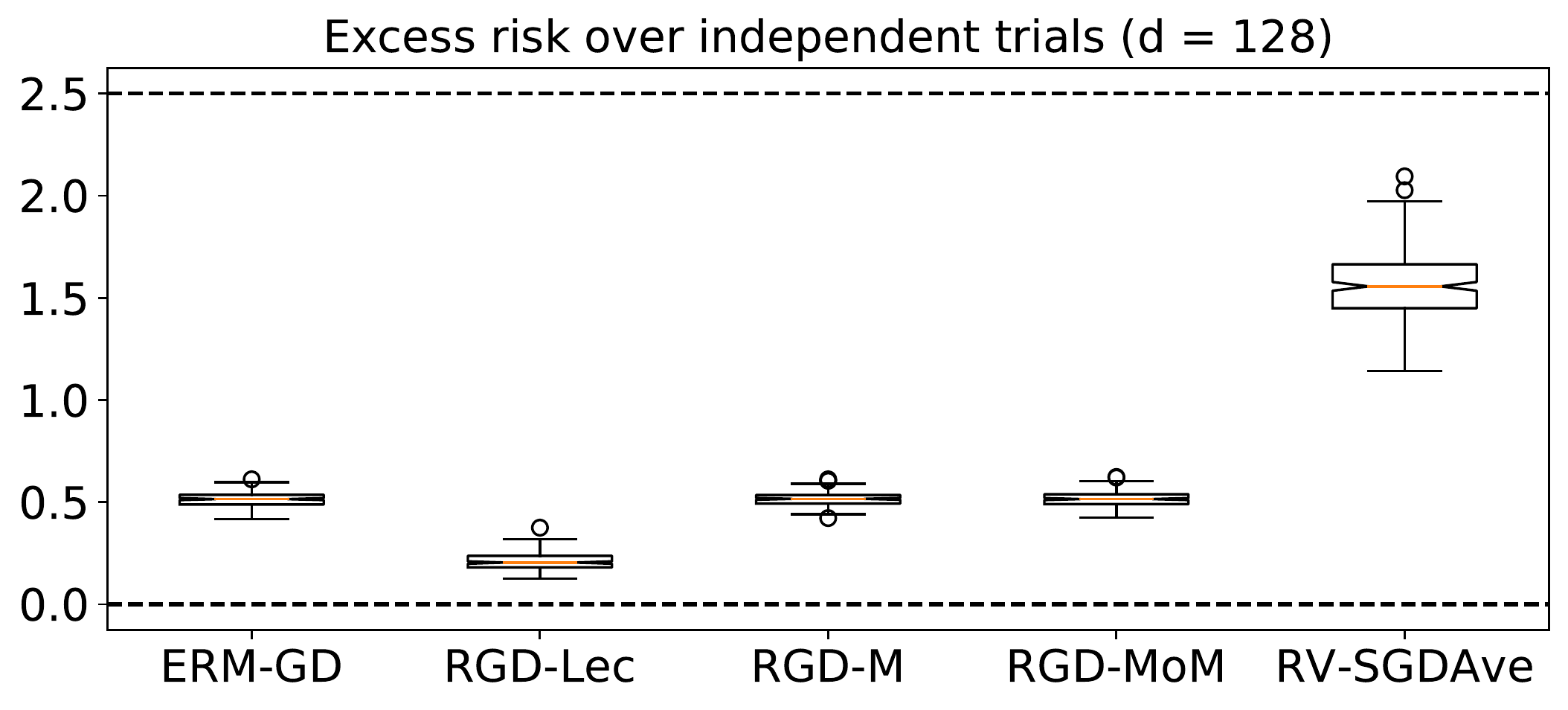}\,\includegraphics[width=0.5\textwidth]{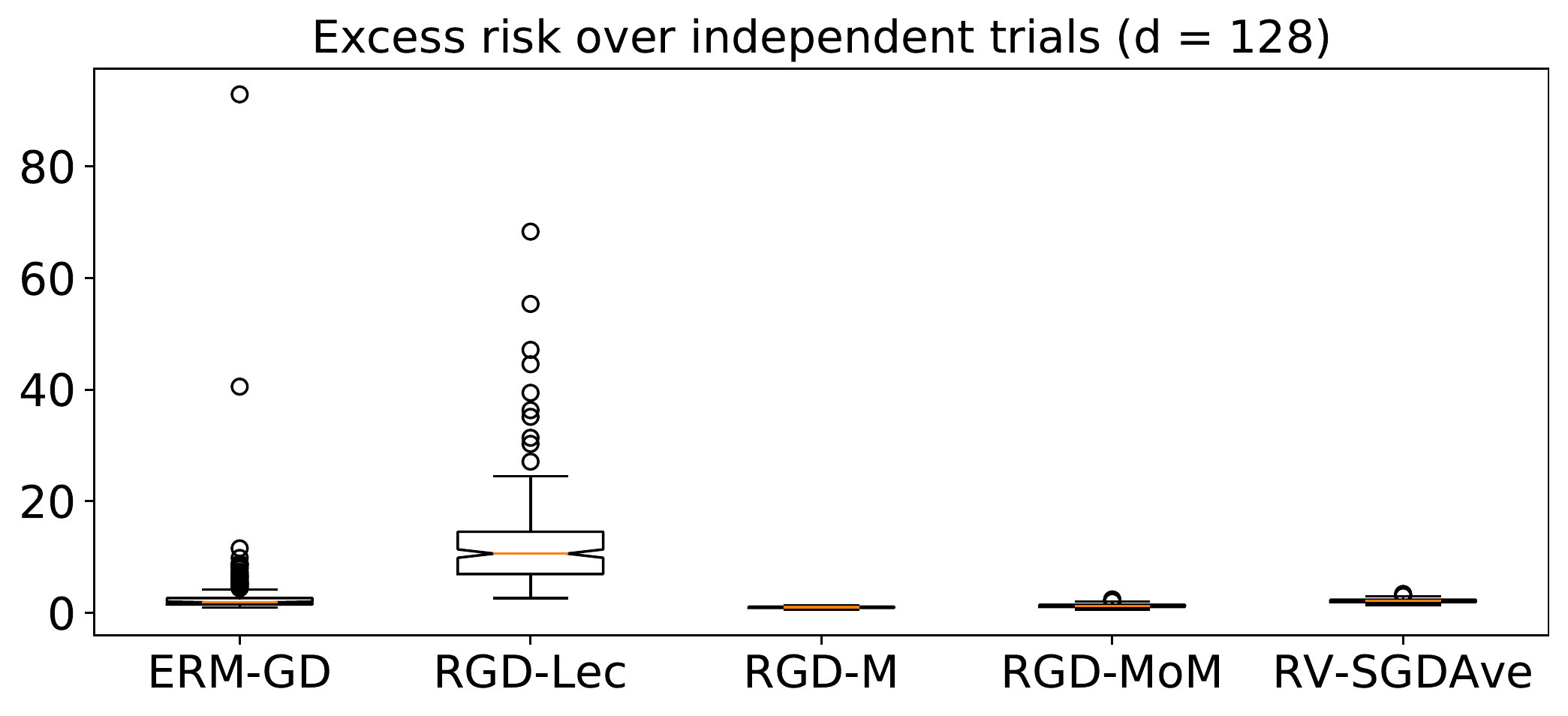}\\
\includegraphics[width=0.5\textwidth]{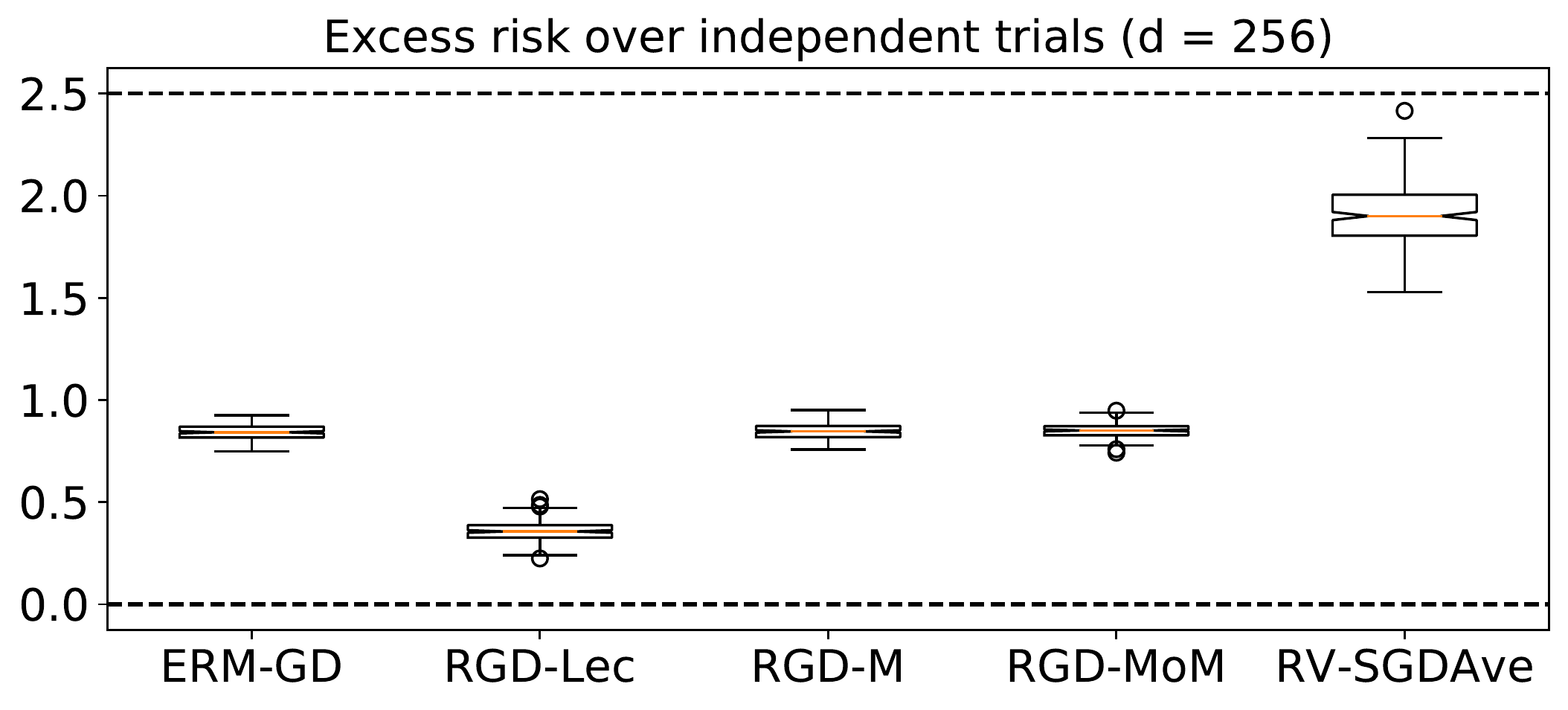}\,\includegraphics[width=0.5\textwidth]{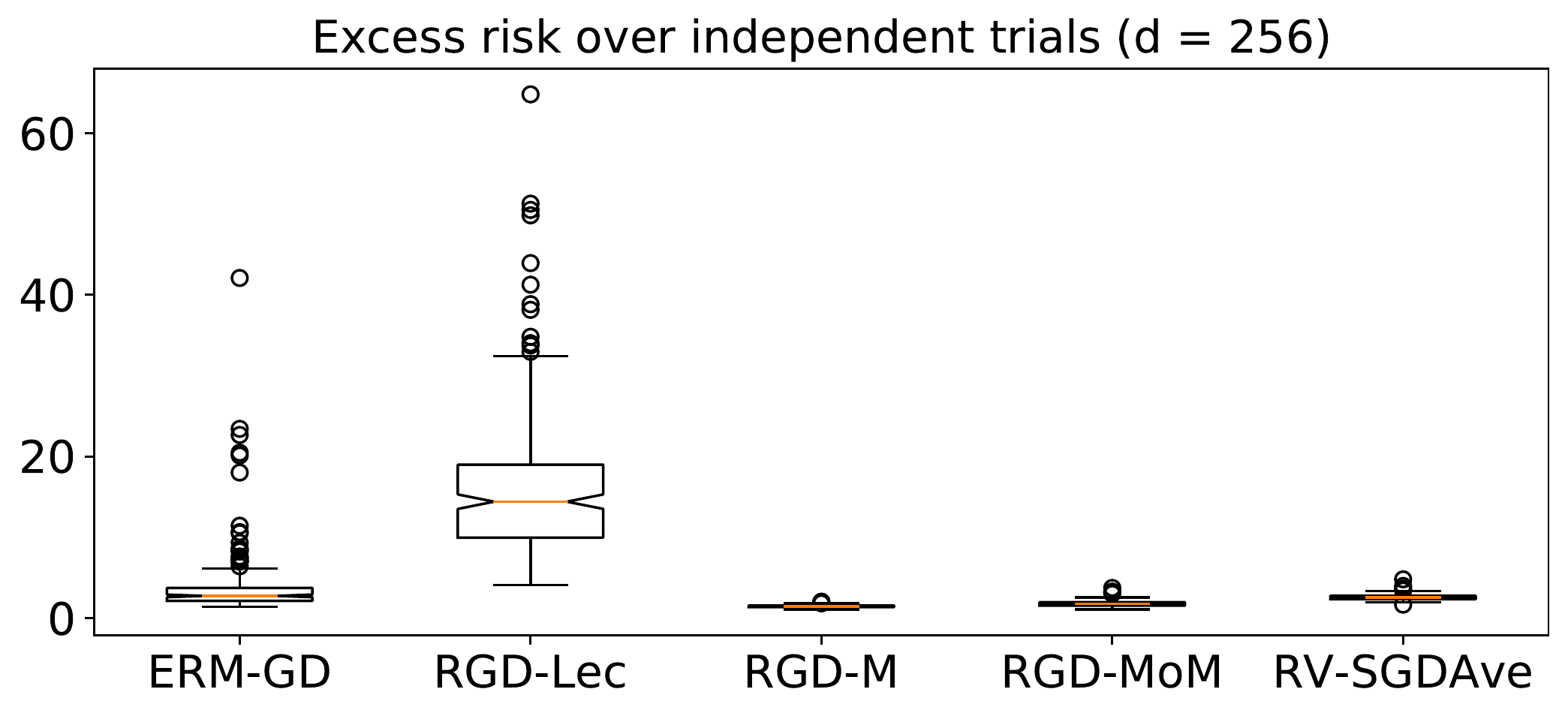}\\
\includegraphics[width=0.5\textwidth]{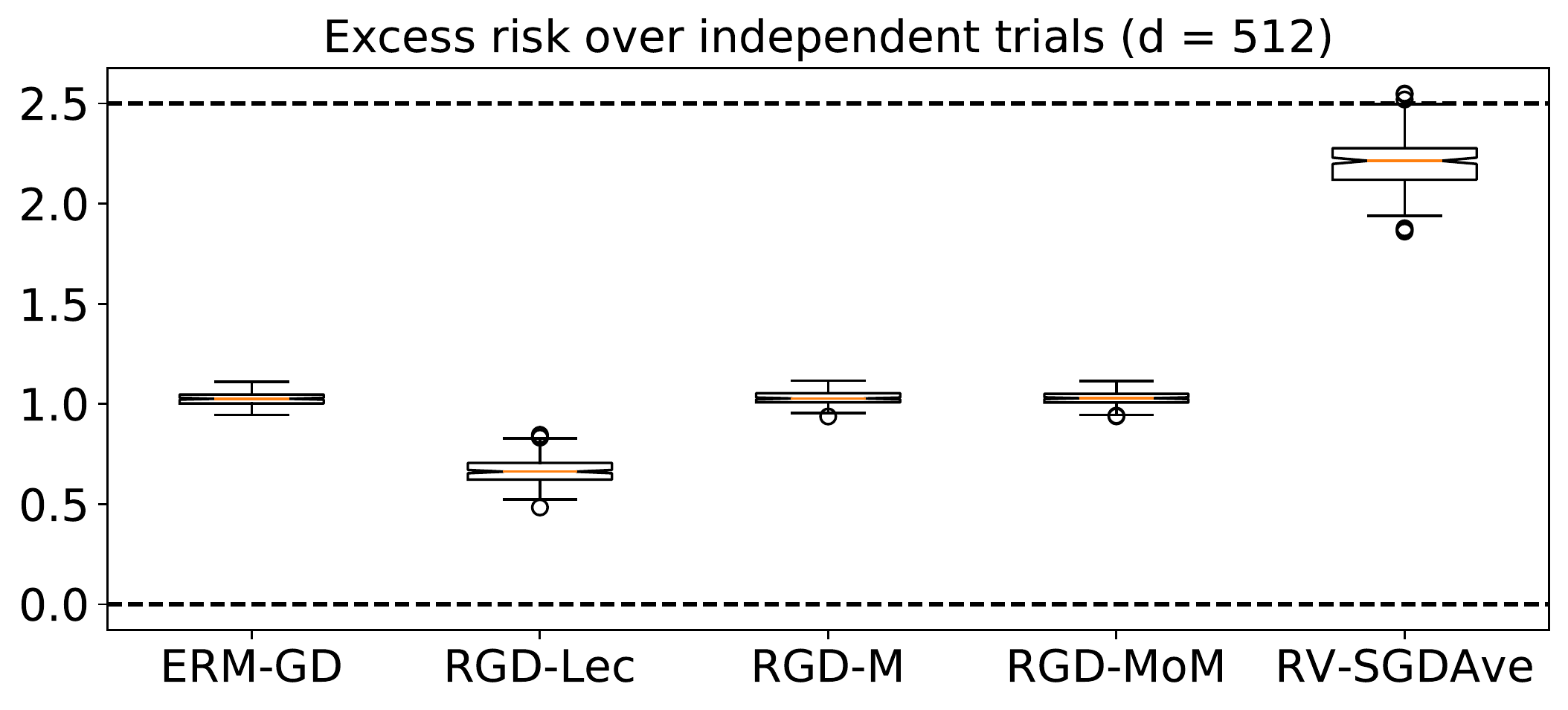}\,\includegraphics[width=0.5\textwidth]{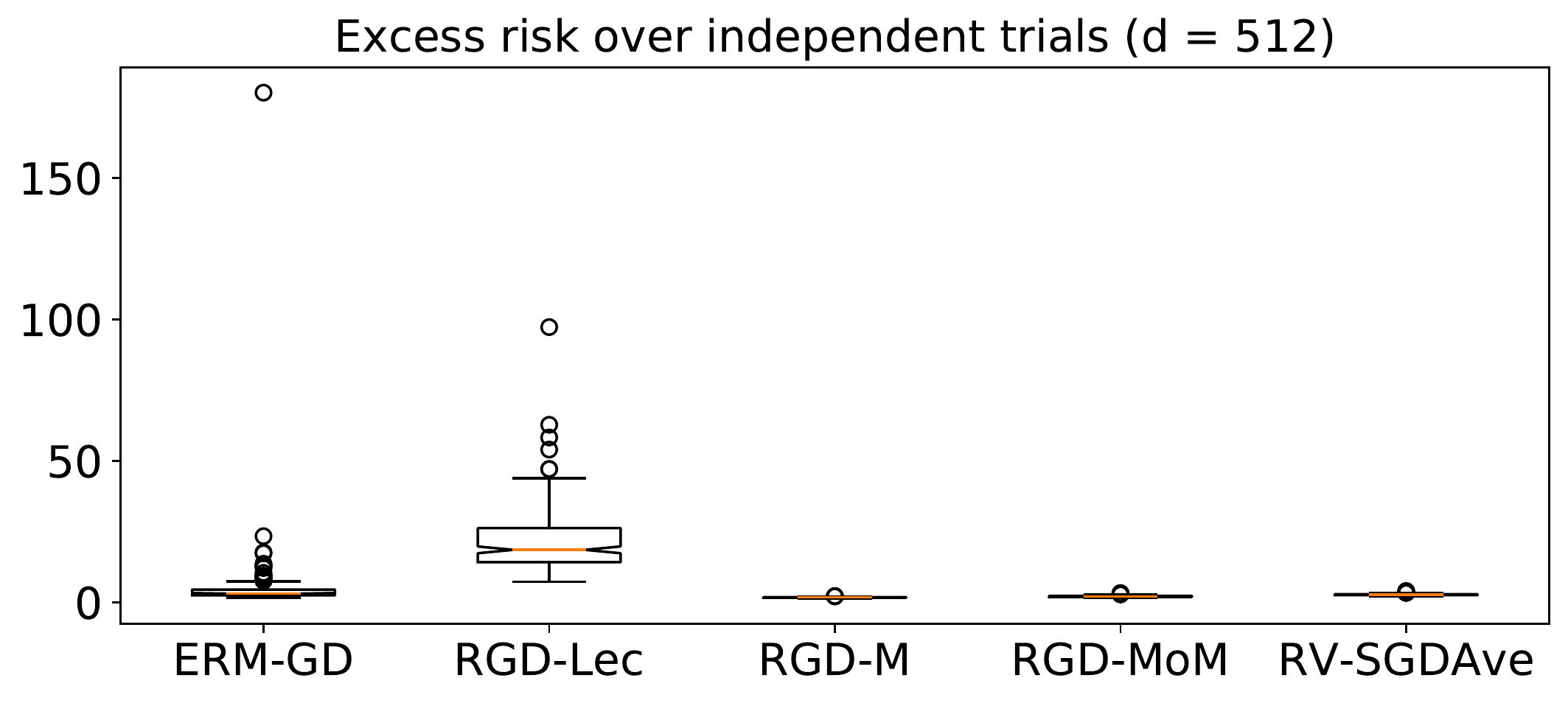}\\
\includegraphics[width=0.5\textwidth]{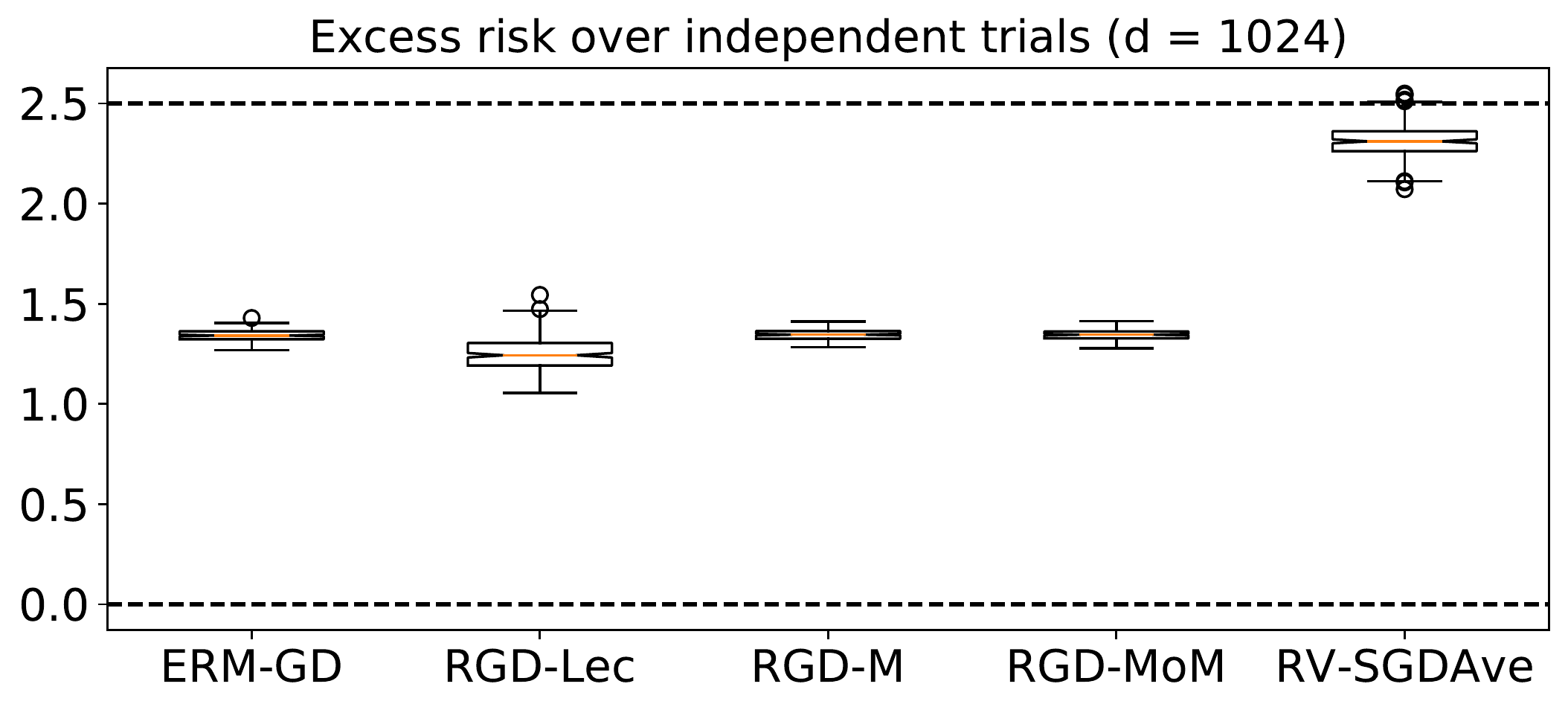}\,\includegraphics[width=0.5\textwidth]{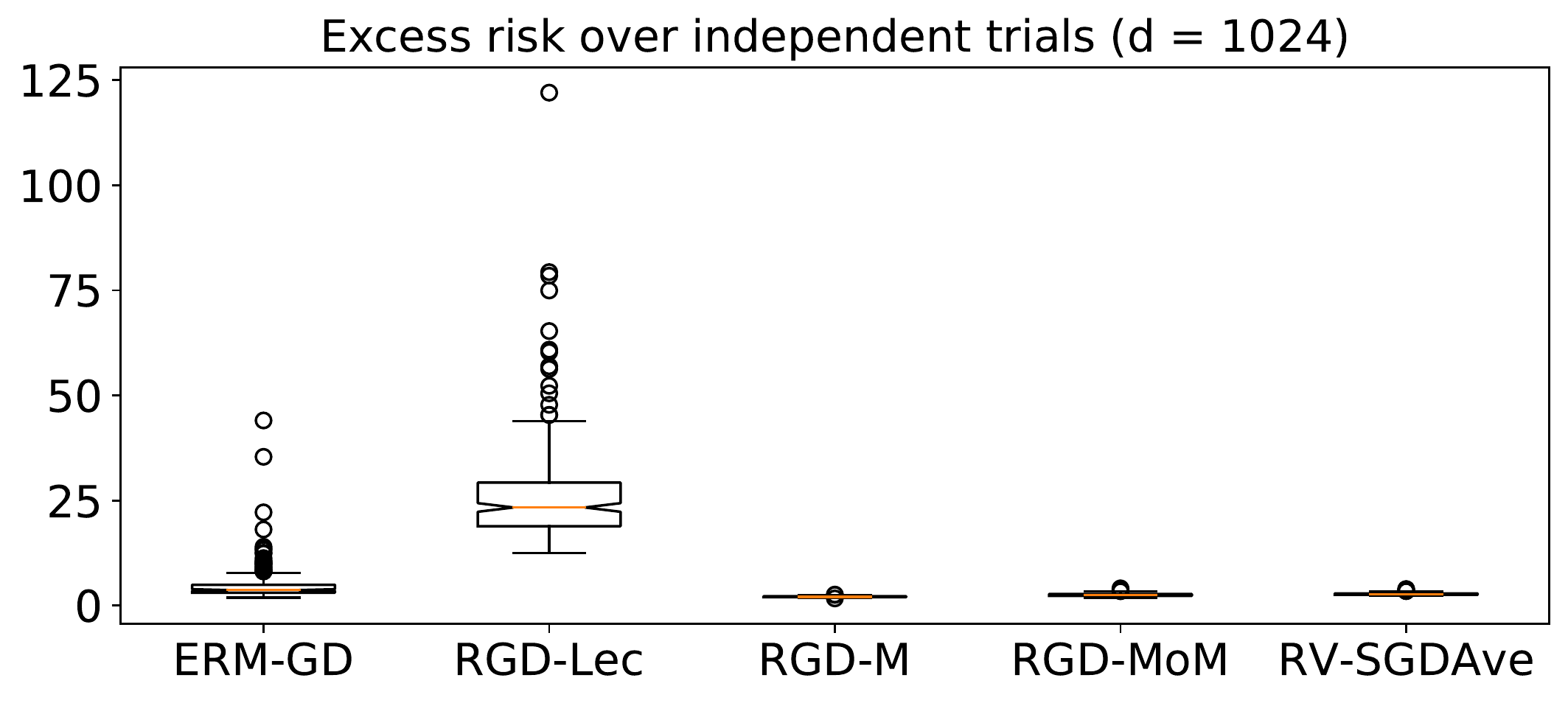}
\caption{Excess risk for many-pass batch methods and single-pass \texttt{RV-SGDAve}. Dimension settings shown are $d \in \{128, 256, 512, 1024\}$. Left column: Normal noise (dashed horizontal rule is fixed to show small relative changes). Right column: log-Normal noise.}
\label{fig:ERROR_DMU_nonsc}
\end{figure}

\paragraph{(E3) How do actual computation times compare as $n$ and/or $d$ grow?}

Tests here are done in the same way as the previous sub-section. First, $n$ and $d$ move together, with $n=4000d$ and $2 \leq d \leq 64$. Second, $n=2500$ is fixed, and dimension ranges over $2 \leq d \leq 1024$ as in (E2). Budget constraints used for stopping rules are exactly as described in (E2). We run $250$ independent trials, and compute the median times for each method. We remark that in comparing the log-Normal versus Normal cases, there is virtually no difference between the computation times for any method, and thus to save space we simply show times for the log-Normal case; these median times for both experimental settings are shown in Figure \ref{fig:lognormal_times_nonsc}.

\begin{figure}[t]
\centering
\includegraphics[width=0.4\textwidth]{nonsc/proof_of_concept_POC_legend}\\
\includegraphics[width=0.5\textwidth]{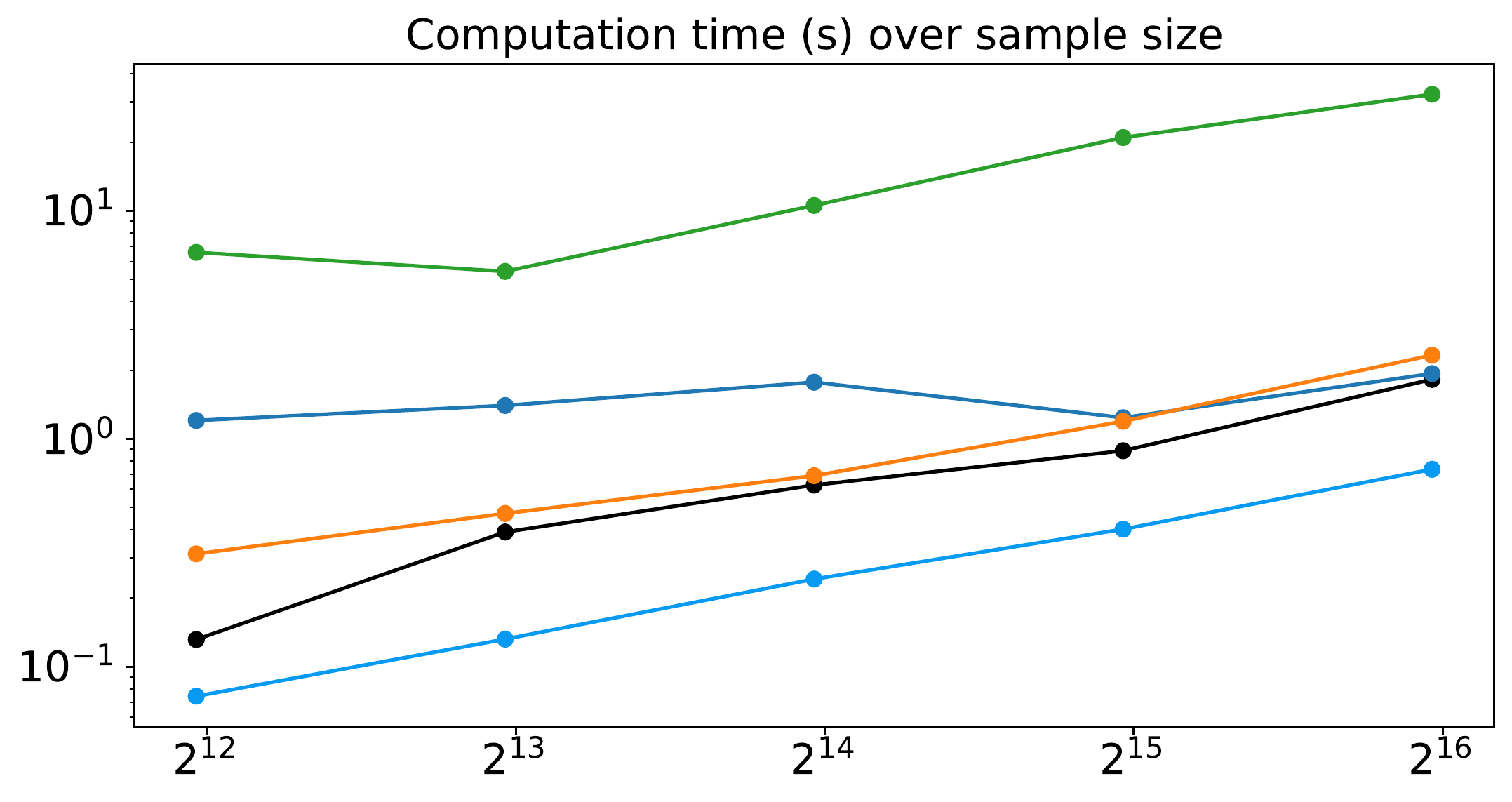}\,\includegraphics[width=0.5\textwidth]{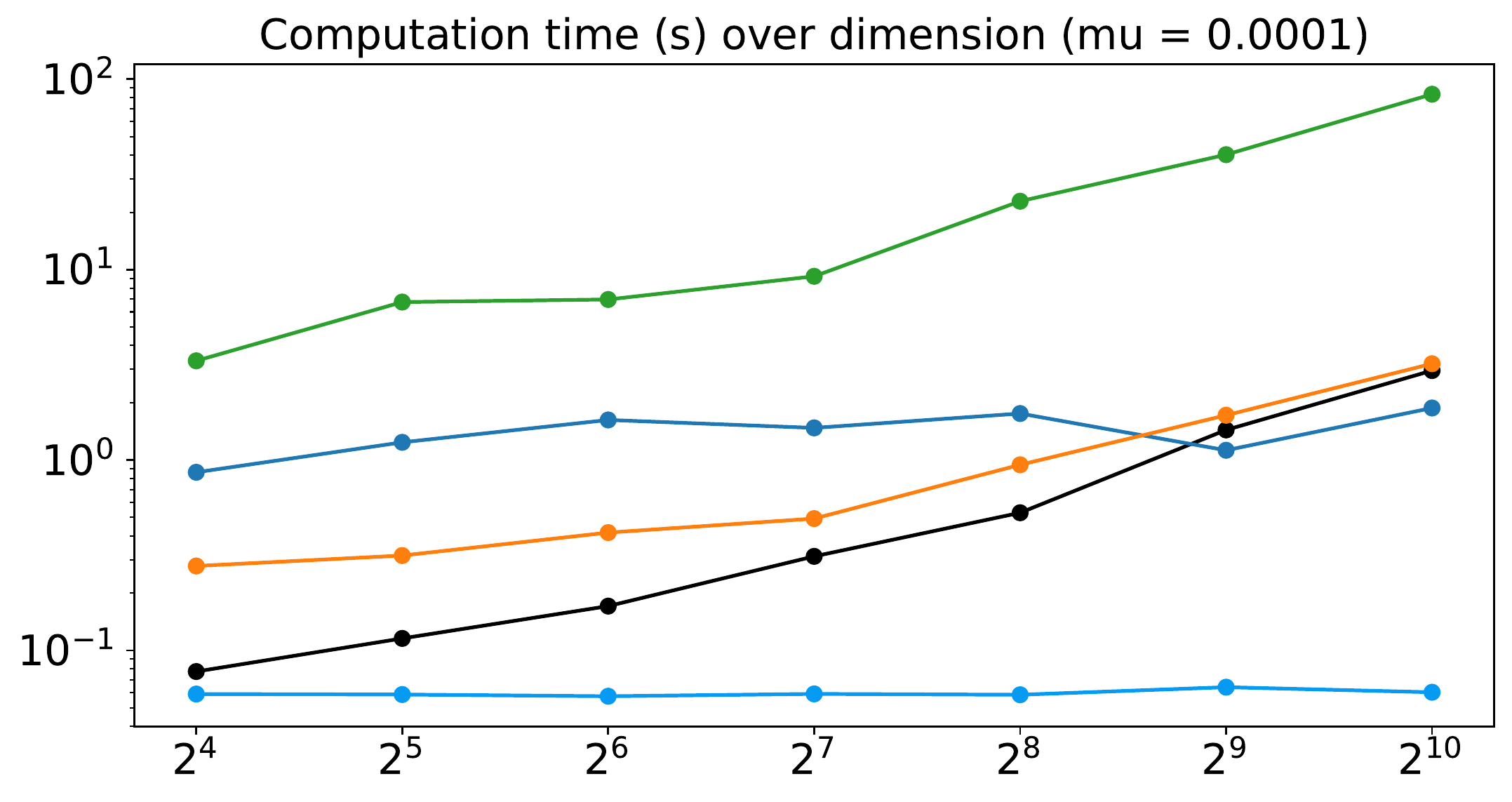}
\caption{Median computation times (log scale, base $10$) as a function of $n$  and $d$ (right; log scale, base $2$). Left: time as a function of $n$ (log scale, base $2$), with $n$ and $d$ growing together. Right: time as a function of $d$ (log scale, base $2$), with $n$ fixed.}
\label{fig:lognormal_times_nonsc}
\end{figure}

\paragraph{(E4) Can robust validation be replaced by cross-validation?}

Finally, it is natural to ask whether the procedure of Algorithm \ref{algo:DandC_valid} could be replaced by a heuristic cross-validation procedure that uses all the data for learning, doubling the effective sample size available to each sub-process. More precisely, say that instead of splitting the $n$-sized sample into $\Z_{n/2}$ and $\Z_{n/2}^{\prime}$ as done by \texttt{RV-SGD} (Algorithm \ref{algo:DandC_valid}), we simply use a full $n$-sized sample $\Z_{n}$, partition into $k$ subsets $\II_{1},\ldots,\II_{k}$, obtaining $k$ independent candidates $\overbar{w}^{(1)},\ldots,\overbar{w}^{(k)}$, now with double the sample size compared with \texttt{RV-SGD}. One might be intuitively inclined to do a cross-validation type of selection, where for each $j \in [k]$, the validation score returned by $\valid$ is computed for each $\overbar{w}^{(j)}$ using the data $Z_{i}$ indexed by $i \in [n] \setminus \II_{j}$, and the winning index $\star$ is selected to be the minimizer of this cross-validation error. Such heuristics break the assumptions used in the theoretical analysis of Algorithm \ref{algo:DandC_valid}, and it is interesting to see how this plays out in practice. Thus, we have re-implemented both \texttt{RV-SGD} and \texttt{RV-SGDAve} in this fashion, respectively denoted \texttt{RV-SGD-CV} and \texttt{RV-SGDAve-CV}. Error trajectories for the same experimental setting as (E1) for all these methods are compared in Figure \ref{fig:CVtest_nonsc}.

\begin{figure}[t]
\centering
\includegraphics[width=0.5\textwidth]{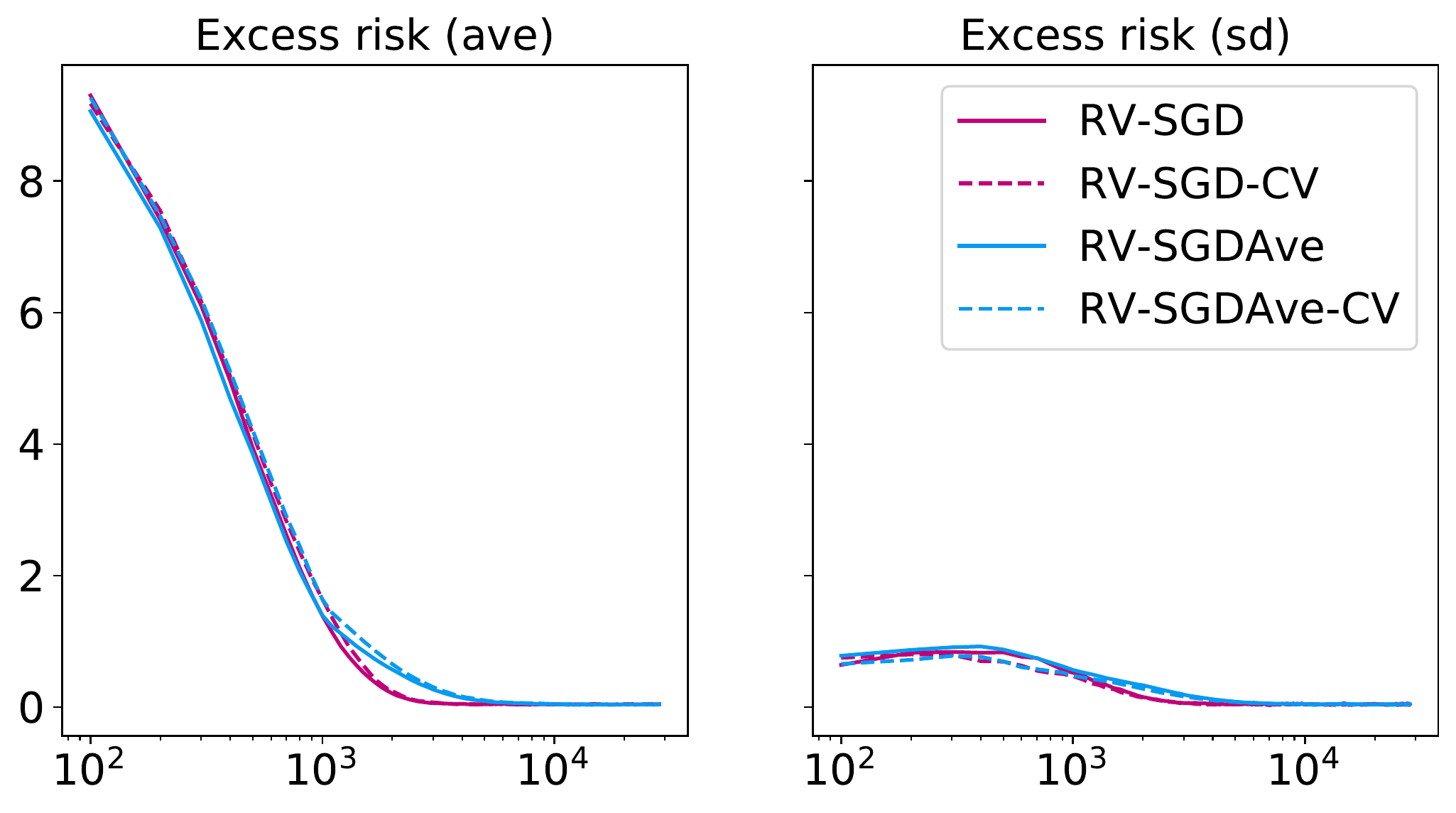}\,\includegraphics[width=0.5\textwidth]{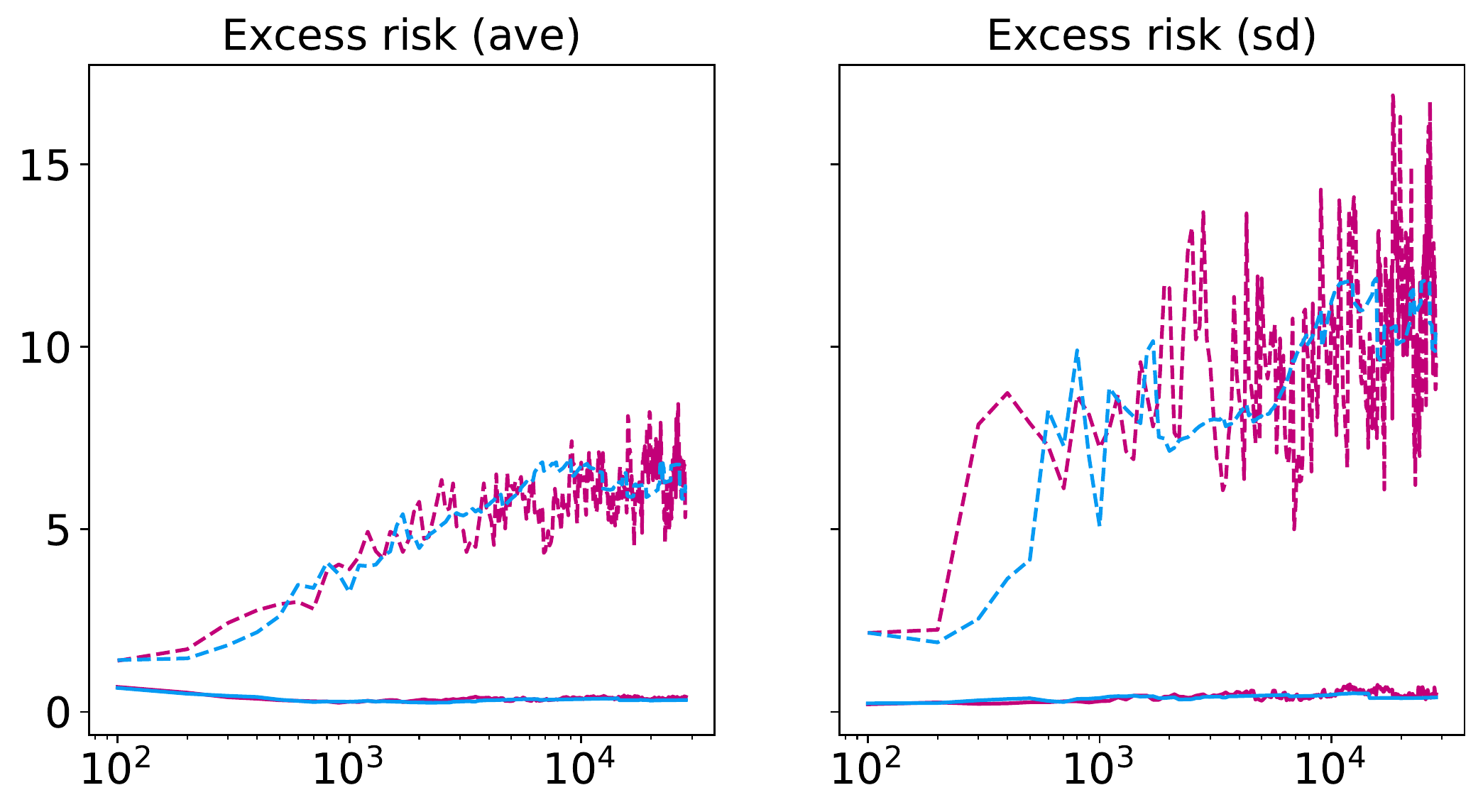}
\caption{The negative impact of trying to modify Algorithm \ref{algo:DandC_valid} to use a cross validation heuristic. Left: Normal noise. Right: log-Normal noise.}
\label{fig:CVtest_nonsc}
\end{figure}

\paragraph{Discussion of results}

From the initial proof-of-concept tests with results given in Figures \ref{fig:baseline_check_normal_nonsc}--\ref{fig:proof_of_concept_nonsc}, we see how even very noisy sub-processes can be ironed out easily using the simple robust validation sub-routine included in Algorithm \ref{algo:DandC_valid}, and that even running the algorithm for much longer than a single pass over the data, risk which is comparable to benchmark RGD methods can be realized at a much smaller cost, with comparable variance across trials, and that this holds under both sub-Gaussian and heavy-tailed data, without any modifications to the procedure being run. A particularly lucid improvement in the cost-performance tradeoff is evident from Figures \ref{fig:ERROR_DMU_nonsc}--\ref{fig:lognormal_times_nonsc}, since near-identical performance can be achieved at a small fraction of the computational cost. Note that under Normal noise, running Algorithm \ref{algo:DandC_valid} for just a single pass leaves room for improvement performance-wise, but as we saw in the low-dimension case, in practice this can be remedied by taking additional passes over the data. Finally, regarding the question of whether or not Algorithm \ref{algo:DandC_valid} can be replaced with a cross-validation heuristic, the answer is clear (Figure \ref{fig:CVtest_nonsc}): while the results are comparable under well-behaved data (the Normal noise case here), when heavy tails are a possibility (e.g., the log-Normal case), the naive cross-validation method fails to get even near the performance of Algorithm \ref{algo:DandC_valid}.

\subsection{Applications to real data}\label{sec:empirical_realdata}

In this sub-section, we consider applications of the proposed learning algorithms to real-world benchmark datasets that are frequently used in the machine learning literature. While there are innumerable avenues of analysis that could shed additional light on the behaviour and utility of the proposed procedures, here we constrain our focus to two main questions:
\begin{itemize}
\item[\textbf{(E1)}] Under a convex setting, how does Algorithm \ref{algo:DandC_SGD} compare with a single-process benchmark over multiple epochs?
\item[\textbf{(E2)}] How does introducing non-linearity into the model impact the performance of Algorithm \ref{algo:DandC_SGD} compared with Algorithm \ref{algo:DandC_valid}?
\end{itemize}
Real-world data is a valuable resource, and spreading it too thinly across cheap sub-processes is naturally bound to result in a performance decrease. From the perspective of the practioner, arguably the simplest and most natural question to ask here is whether it is actually worth the effort of going through the split-train-integrate process utilized in Algorithms \ref{algo:DandC_SGD} and \ref{algo:DandC_valid}. That is, why not reduce the sub-process noise to the limit by giving all the data to a single procedure? This is addressed by \textbf{(E1)}. Equally fundamental is the fact that convexity is quite critical to the $\merge$ sub-routines used in Algorithm \ref{algo:DandC_SGD}, but plays no role at all in the $\valid$ sub-routines used in Algorithm \ref{algo:DandC_valid}. When we introduce non-linearity into our models, we can shrink the ``bias'' due to model error, at the cost of having a more complicated space to search for a solution, typically without guarantees of convexity. It is \textbf{(E2)} that looks at how the complexity of this search space impacts the two proposed algorithms.

\paragraph{Experimental setup}

Summarized at a high level, the experiments we run here involve multiple randomized trials for each benchmark dataset, where for each trial we run the algorithms of interest for multiple passes over the data (or ``epochs''), recording performance at the end of each epoch. To provide additional insights beyond the simulations done in previous sections, instead of regression tasks here we focus on classification tasks. Our implementation relies upon PyTorch (ver.~1.6.0)\footnote{Documentation for \texttt{torch}: \url{https://pytorch.org/docs/1.6.0/index.html}.} to construct model and optimizer objects, essentially entirely using standard objects defined in \texttt{torch.nn}. More precisely, in the convex case we effectively do logistic regression, by passing the raw inputs through a single linear layer (\texttt{nn.Linear}) followed by a log-softmax transformation (\texttt{F.log\_softmax}), where the number of outputs equals the number of classes, and the loss function is the usual negative log likelihood (\texttt{F.nll}). In the non-convex case, the loss is the same, but now the inputs are passed through multiple intermediate layers, with a non-linear transformation. Each intermediate layer has $10$ units, and is passed through a rectified linear unit transformation (\texttt{F.relu}). Once again, the number of outputs in the final layer equals the number of classes, and this output is passed through the log-softmax transform. Recalling that the sub-process run in our proposed procedures is $\SGD$ defined in (\ref{eqn:sgd_defn}), this is implemented using the optimizer object \texttt{optim.SGD} off the shelf, up to modifications made to the step size and mini-batch size, to be discussed shortly.

When we say each trial is ``randomized,'' there are two key elements that are randomly determined each time: first is the initialization of all model parameters (we use the default random initialization of \texttt{nn.Module} objects), and second is the data used. Regarding the latter, say the full dataset has $m$ data points. Then for each trial, we randomly shuffle the order of all $m$ points, before splitting the data into three subsets, for training (\text{tr}), testing (\text{te}), and validation (\text{val}), i.e., $[m] = \II_{\text{tr}} \cup \II_{\text{te}} \cup \II_{\text{val}}$. For all datasets and all trials, we make this dataset such that $\lfloor |\II_{\text{tr}}|/m \rfloor = 0.8$ and $\lfloor |\II_{\text{val}}|/|\II_{\text{tr}}| \rfloor = 0.1$. The remaining points are allocated to $\II_{\text{te}}$. All results given in the figures to follow are statistics (e.g., sample mean, standard deviation, etc.) taken over all the random trials just described.

\paragraph{Description of benchmark datasets}

We use six datasets, identified respectively by the keywords: \texttt{adult},\footnote{\url{https://archive.ics.uci.edu/ml/datasets/Adult}} \texttt{cifar10},\footnote{\url{https://www.cs.toronto.edu/~kriz/cifar.html}} \texttt{cod\_rna},\footnote{\url{https://www.csie.ntu.edu.tw/~cjlin/libsvmtools/datasets/binary.html}} \texttt{emnist\_balanced},\footnote{\url{https://www.nist.gov/itl/products-and-services/emnist-dataset}} \texttt{fashion\_mnist},\footnote{\url{https://github.com/zalandoresearch/fashion-mnist}} and \texttt{mnist}.\footnote{\url{http://yann.lecun.com/exdb/mnist/}} See Table \ref{table:datasets} for a summary. Further background on all datasets is available at the URLs provided in the footnotes of this page. For all datasets, we normalize the inputs in a feature-wise fashion to take values on the unit interval. Dataset size reflects the size after removal of instances with missing values, where applicable. For all datasets with categorical features, the ``number of input features'' given in Table \ref{table:datasets} represents the number of features after doing a one-hot encoding of all such features. 

\begin{table}[t!]
\begin{center}
\begin{tabular}{|l|l|l|l|}
\hline
Dataset & Size & Number of input features & Number of classes \\
\hline\hline
\texttt{adult} & 45,222 & 105 & 2\\
\hline
\texttt{cifar10} & 60,000 & 3072 & 10\\
\hline
\texttt{cod\_rna} & 331,152 & 8 & 2\\
\hline
\texttt{emnist\_balanced} & 131,600 & 784 & 47\\
\hline
\texttt{fashion\_mnist} & 70,000 & 784 & 10\\
\hline
\texttt{mnist} & 70,000 & 784 & 10\\
\hline
\end{tabular}
\end{center}
\vspace{-0.5cm}
\caption{A summary of the benchmark datasets used for performance evaluation.}
\label{table:datasets}
\end{table}

\paragraph{(E1) Under a convex risk, how does Algorithm \ref{algo:DandC_SGD} compare with a single-process benchmark?}

For this point of inquiry, we essentially consider two options: one is running a single benchmark process of $\SGD$ on all the non-test data (i.e., all data indexed by $\II_{\text{tr}} \cup \II_{\text{val}}$), and the other is to disjointly split just the training data (indexed by $\II_{\text{tr}}$) into $k > 1$ subsets, running an $\SGD$ sub-process on each of the $k$ subsets. The step size parameter $\alpha$ (called \texttt{lr} in PyTorch) is the only hyperparameter we need to set manually. In the former case, which we use as a benchmark and denote by \texttt{bench}, we simply use the output of $\SGD$ itself. In the latter case, we use the $k$ sub-processes to fuel Algorithm \ref{algo:DandC_SGD}, once again denoted \texttt{DC-SGD}, where we have again implemented $\merge$ using the geometric median (Algorithm \ref{algo:merge_geomed}). For all these tests, we have fixed $k=20$, which in consideration of Theorem \ref{thm:sc_smooth_SGDlast_roboost}, amounts to asking for $91\%$ confidence intervals, since $\lceil 8\log(1/0.09) \rceil = 20$. Noting that the dimensionality and sample size of the different datasets in Table \ref{table:datasets} is quite varied, this ensures that with a fixed $k$ setting, formal guarantees with varying degrees of ``tightness'' can be evaluated. Furthermore, for both settings, we use mini-batch sizes of $8$, and run each procedure for $15$ epochs, shuffling the training data before each new epoch (test data is untouched). The number of random trials is $50$.

In order to make a fair comparison between the benchmark and our proposed routine, it is important to consider multiple step size settings. With enough fine-tuning, it is possible to achieve both very good and very poor performance with either procedure being compared, but our chief interest is with how \emph{sensitive} the procedures are to changes in step size. We start with baseline step size $\alpha_{\text{base}}$ set for each dataset at follows: $\alpha_{\text{base}} = 0.0025$ for \texttt{cifar10}, $\alpha_{\text{base}} = 0.05$ for \texttt{adult}, $\alpha_{\text{base}} = 0.15$ for \texttt{cod\_rna}, and $\alpha_{\text{base}} = 0.025$ for the rest. These baseline settings were set simply such that datasets with larger input dimensionality were given smaller step sizes, the specific base values were selected based on representative values found in the literature. We then tested both \texttt{bench} and \texttt{DC-SGD}, using step sizes of $\alpha_{\text{base}} \times 2^{p}$ taken over all settings of $p=0,1,2,3$. In Figure \ref{fig:losses_lr}, we give representative results for each dataset, showing the mean and standard deviation (taken over all trials) of the average test loss achieved after the final training epoch, viewed as a function of the step size coefficient $2^{p}$. Overall, it is clear that the proposed procedure can achieve performance as good or better than the costly benchmark using a much larger step size, and at all step sizes enjoys uniformly smaller variance.

A critical merit to the proposed approach is that the procedure can be easily run in parallel, having split the data across multiple cores to run each of the $k$ sub-processes. In Figure \ref{fig:acc_percore}, we plot the mean and standard deviation of the test accuracy as a function of the per-core data cost over the entire learning process, rather than just the final step. That is, the horizontal axis here denotes the number of data points that each core has processed over time. For the case of \texttt{bench}, since there is just one process, it uses all data indexed by $\II_{\text{tr}} \cup \II_{\text{val}}$ at each epoch, whereas in the case of \texttt{DC-SGD}, each core only uses $|\II_{\text{tr}}|/k$ points. For each of the two methods being compared, we are comparing the \emph{strongest} settings, i.e., the results shown here use the step size setting that led to the highest average test accuracy, chosen from the four settings just discussed. For reference, we have also plotted the mean training accuracy. An immediate take-away is that \texttt{DC-SGD} can achieve highly competitive performance with an order of magnitude less time, without paying a price in terms of variance, and while maintaining a superior generalization error, i.e., a smaller gap between the training and testing accuracies. In addition, while these are the best settings for each method, recalling the sensitivities to setting step sizes too large (shown previously in Figure \ref{fig:losses_lr}), our proposed method looks to be more robust.

\begin{figure}[t]
\centering
\includegraphics[width=0.4\textwidth]{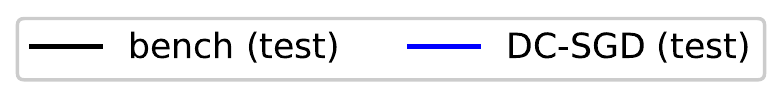}\\
\includegraphics[width=0.45\textwidth]{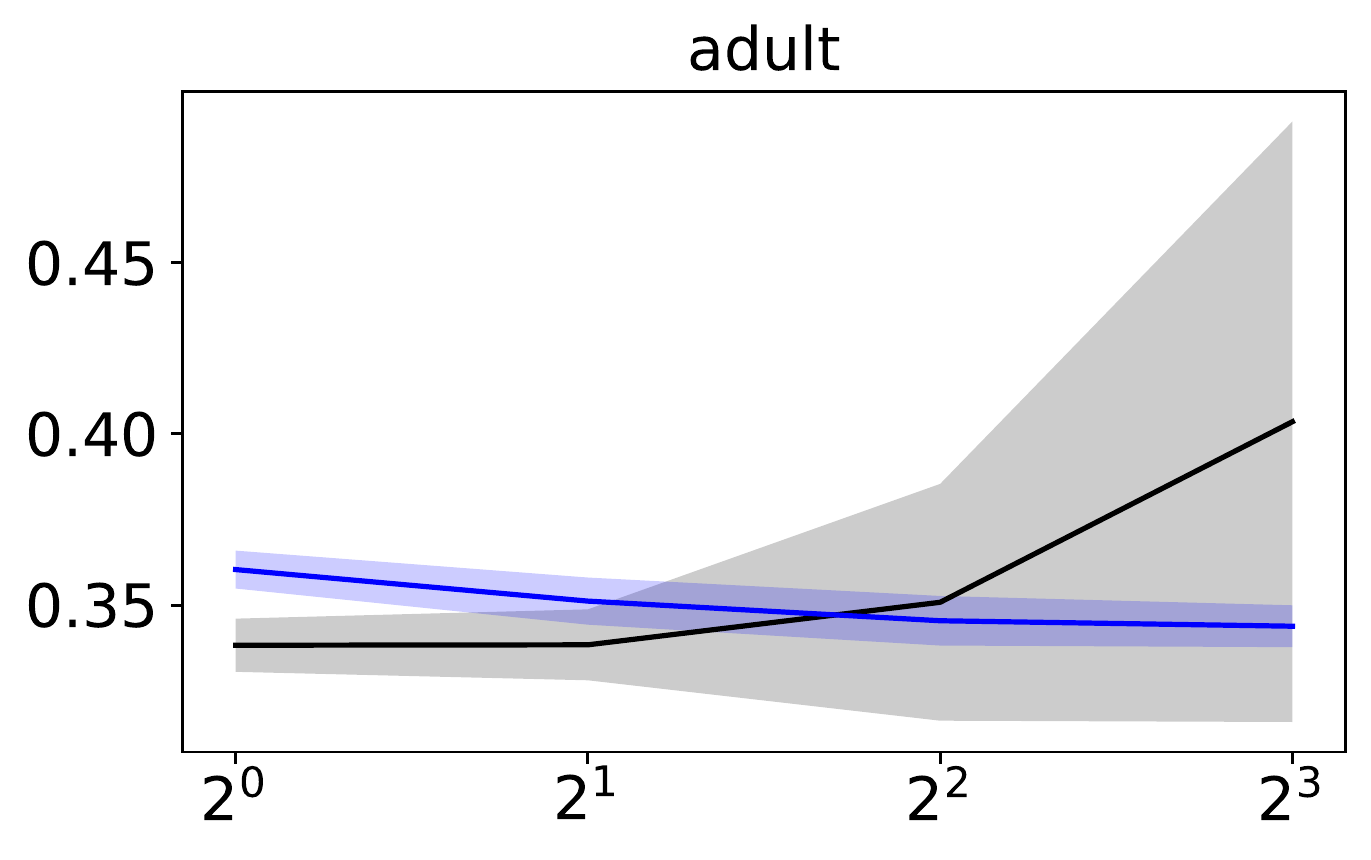}\,\includegraphics[width=0.45\textwidth]{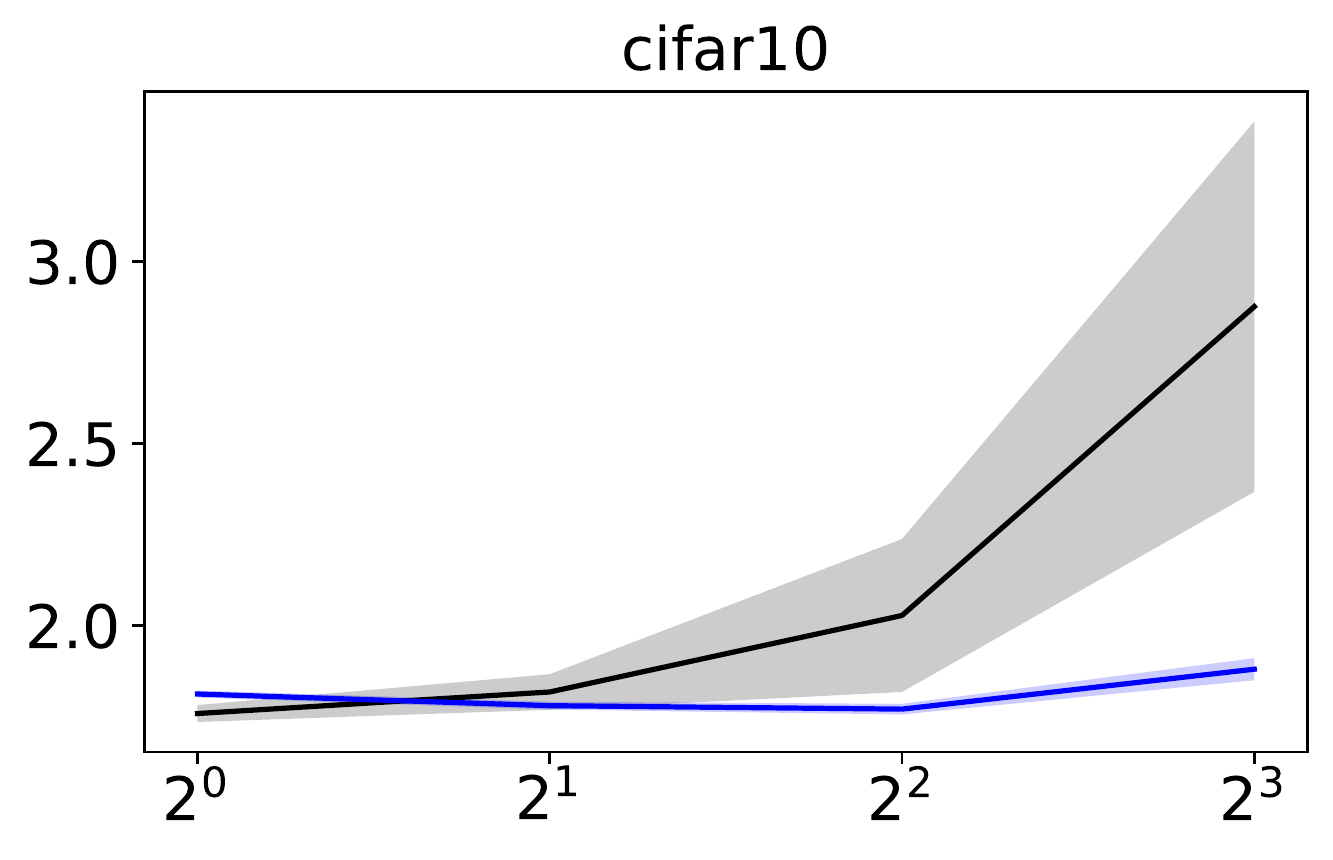}\\
\includegraphics[width=0.45\textwidth]{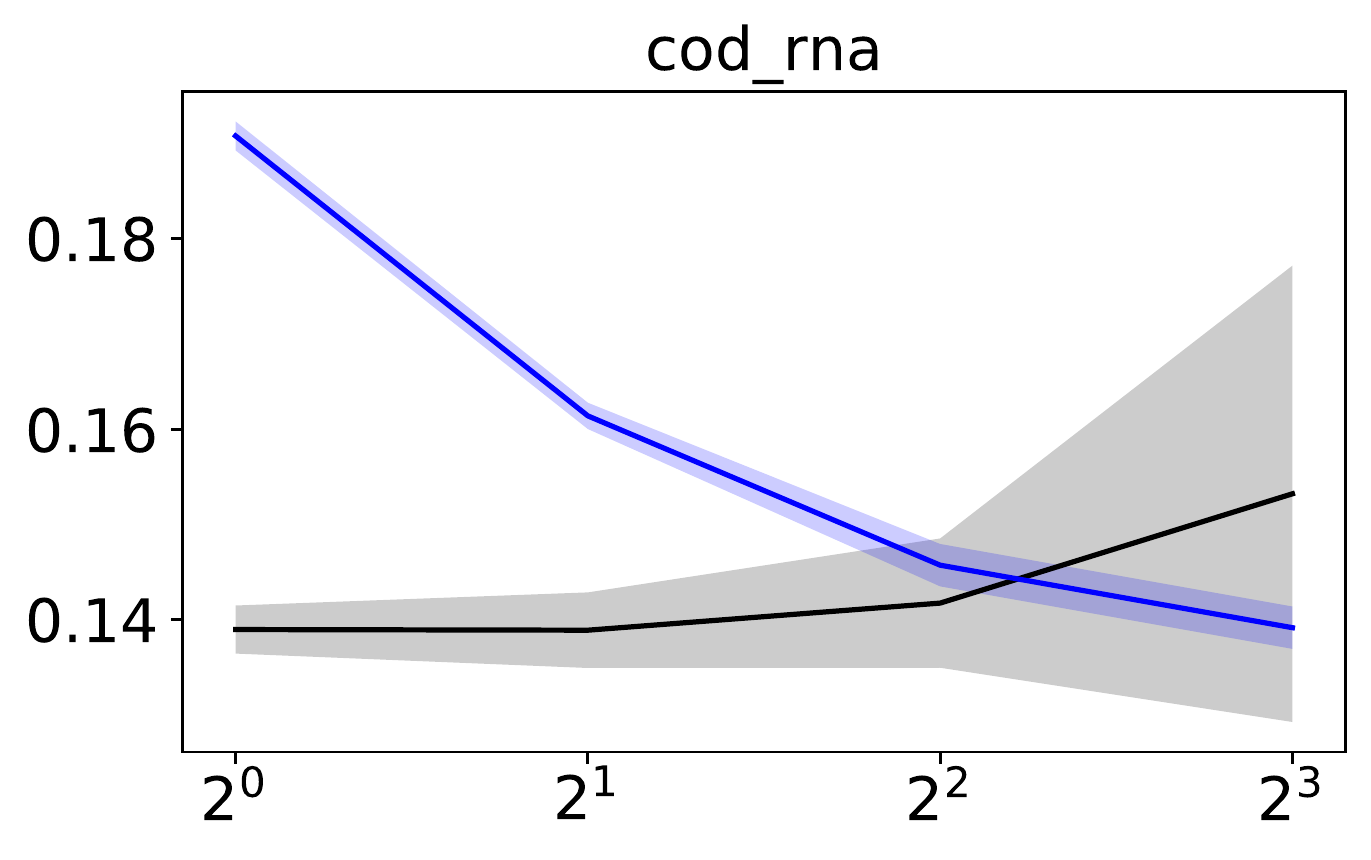}\,\includegraphics[width=0.45\textwidth]{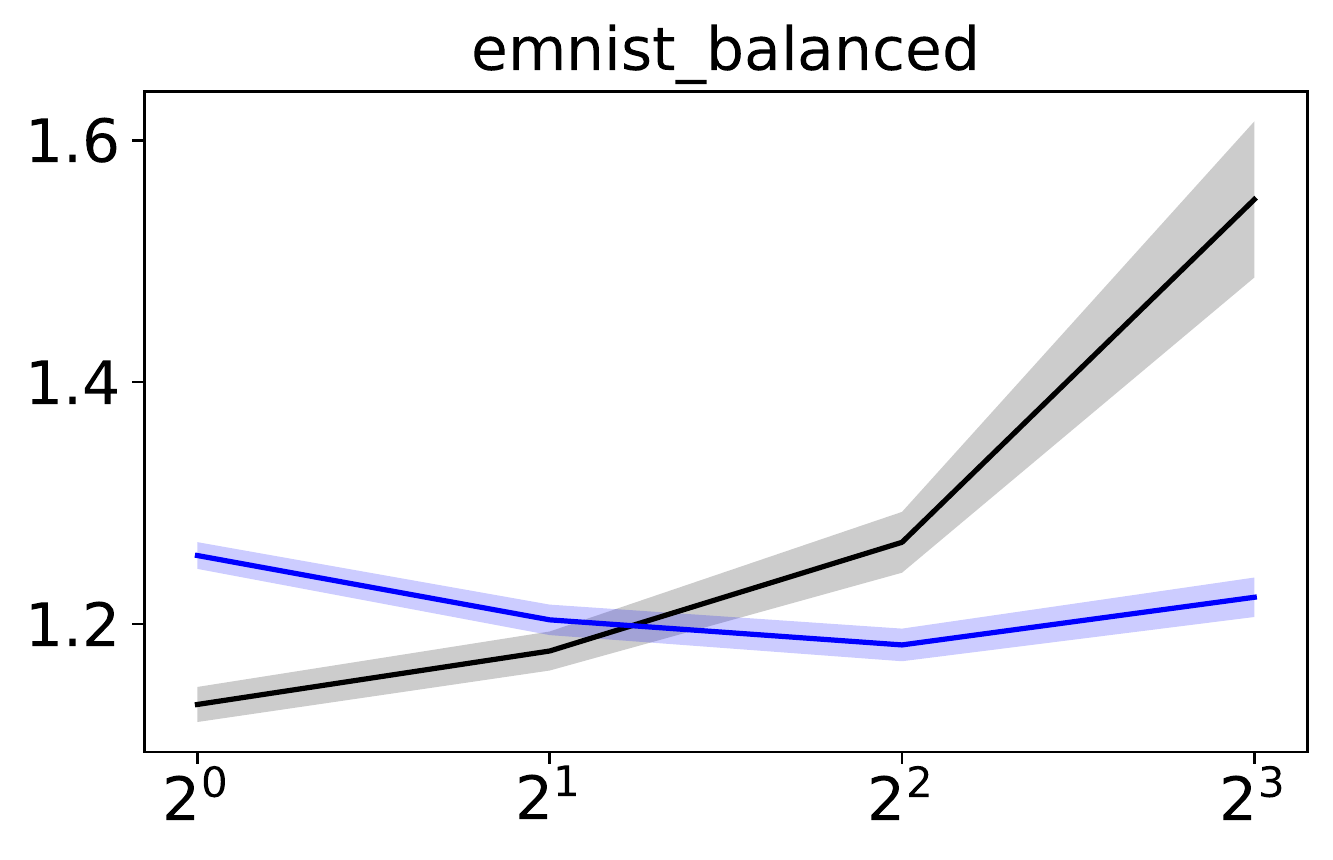}\\
\includegraphics[width=0.45\textwidth]{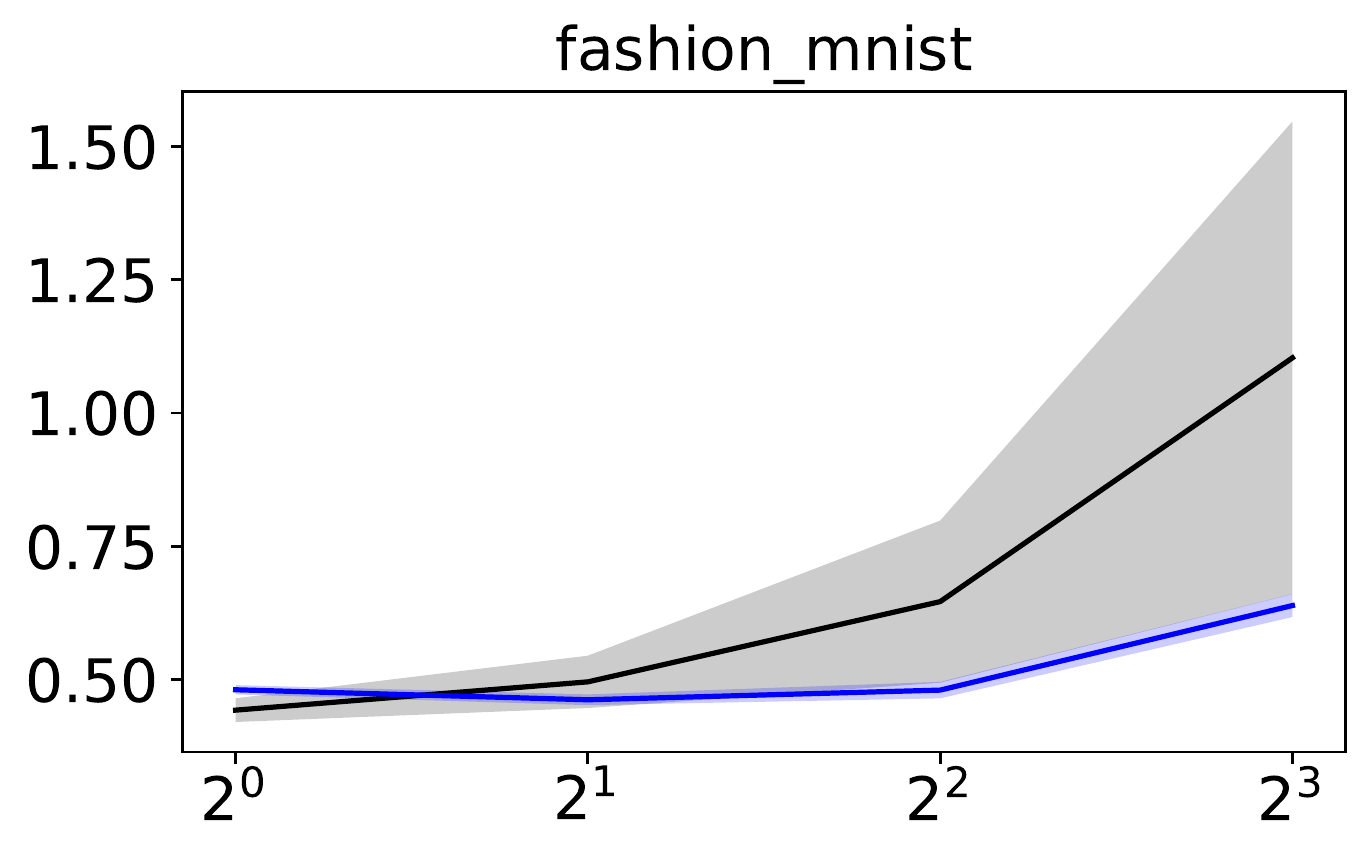}\,\includegraphics[width=0.45\textwidth]{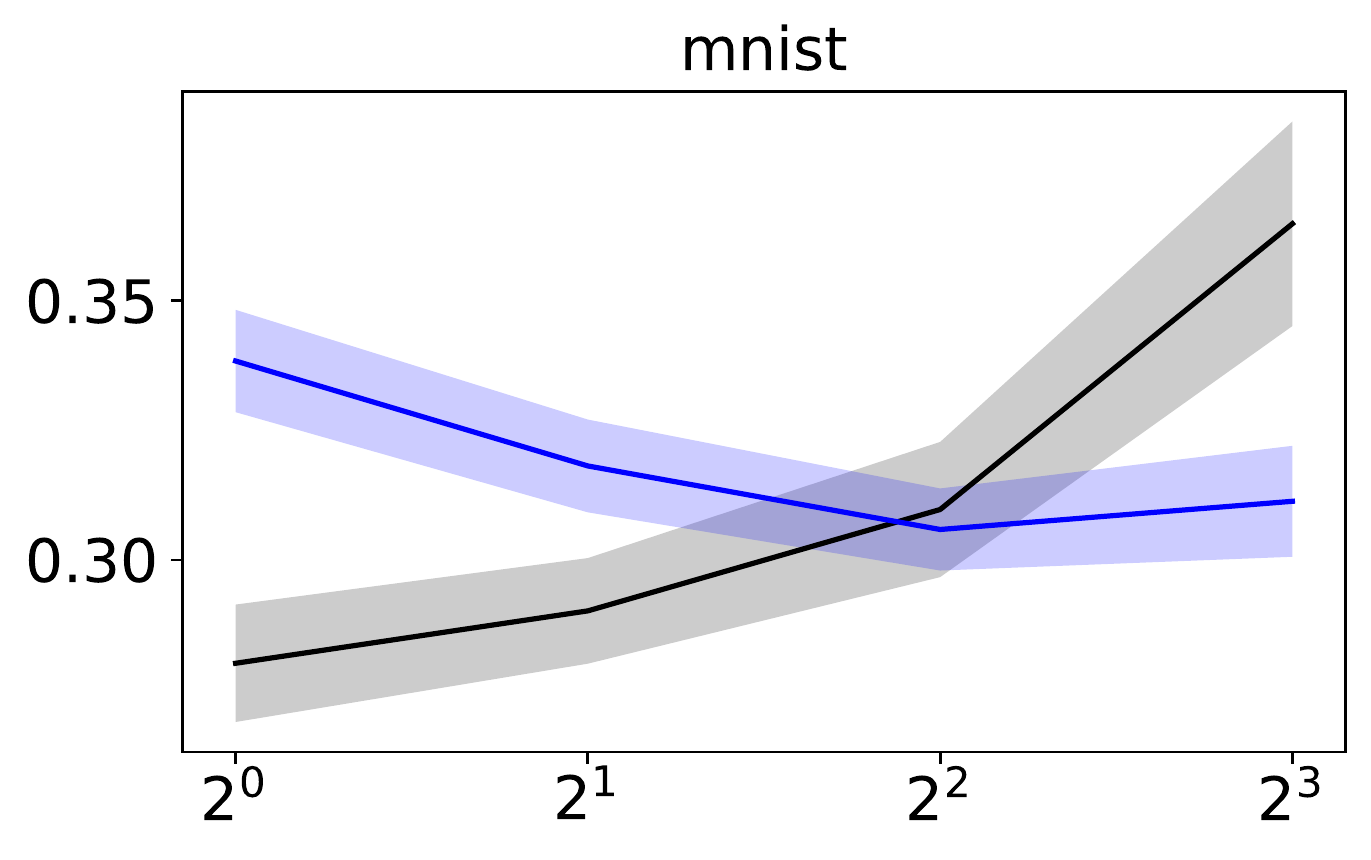}\\
\caption{For each dataset, we plot the mean (over all trials) of the average test loss achieved at the final epoch, as a function of the step size coefficient. The shaded area is the mean $\pm$ standard deviation.}
\label{fig:losses_lr}
\end{figure}

\begin{figure}[t]
\centering
\includegraphics[width=0.4\textwidth]{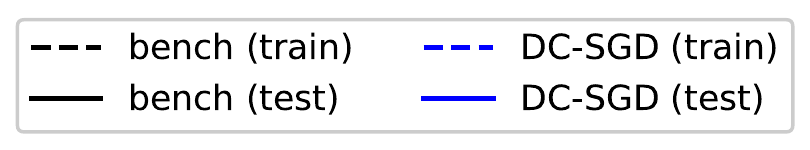}\\
\includegraphics[width=0.45\textwidth]{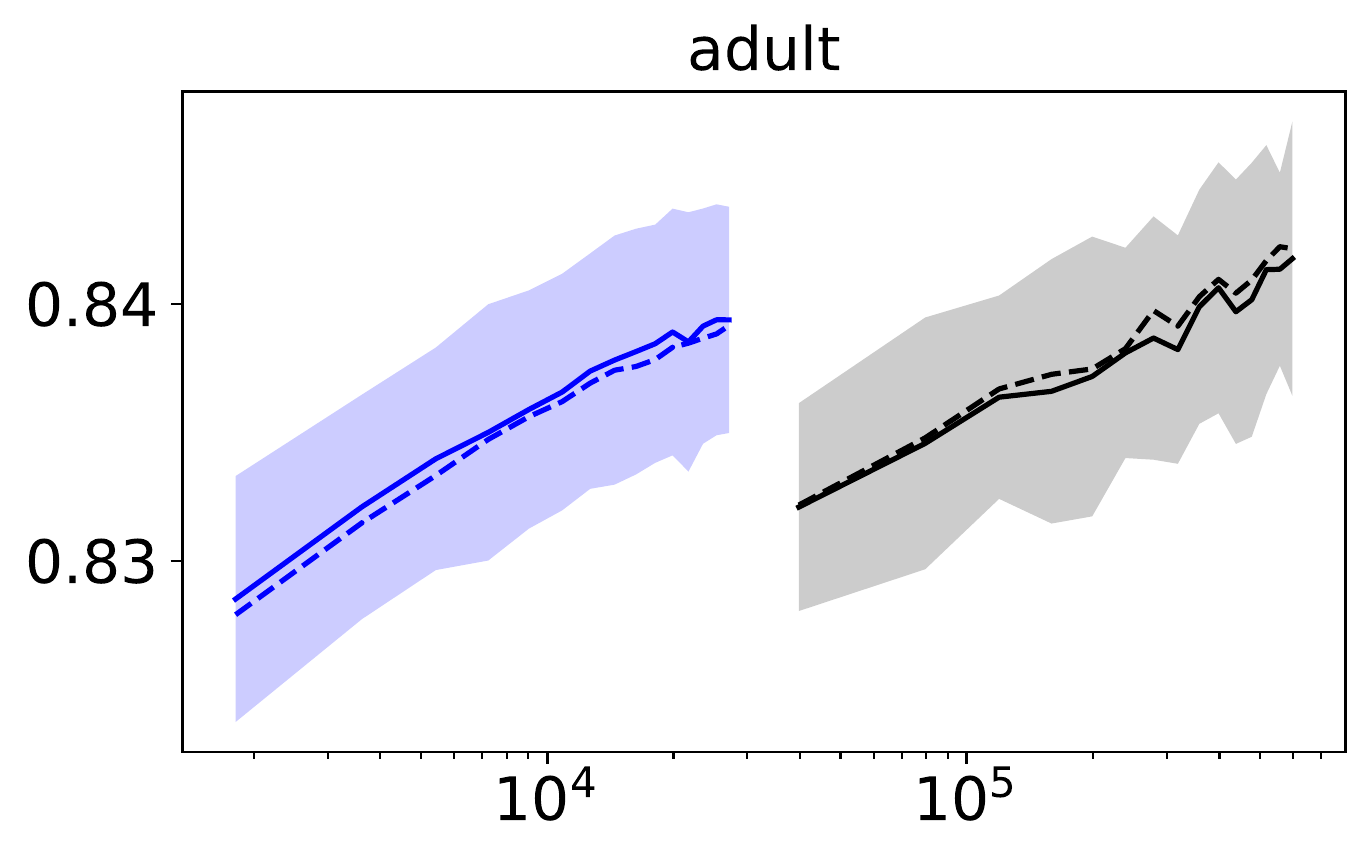}\,\includegraphics[width=0.45\textwidth]{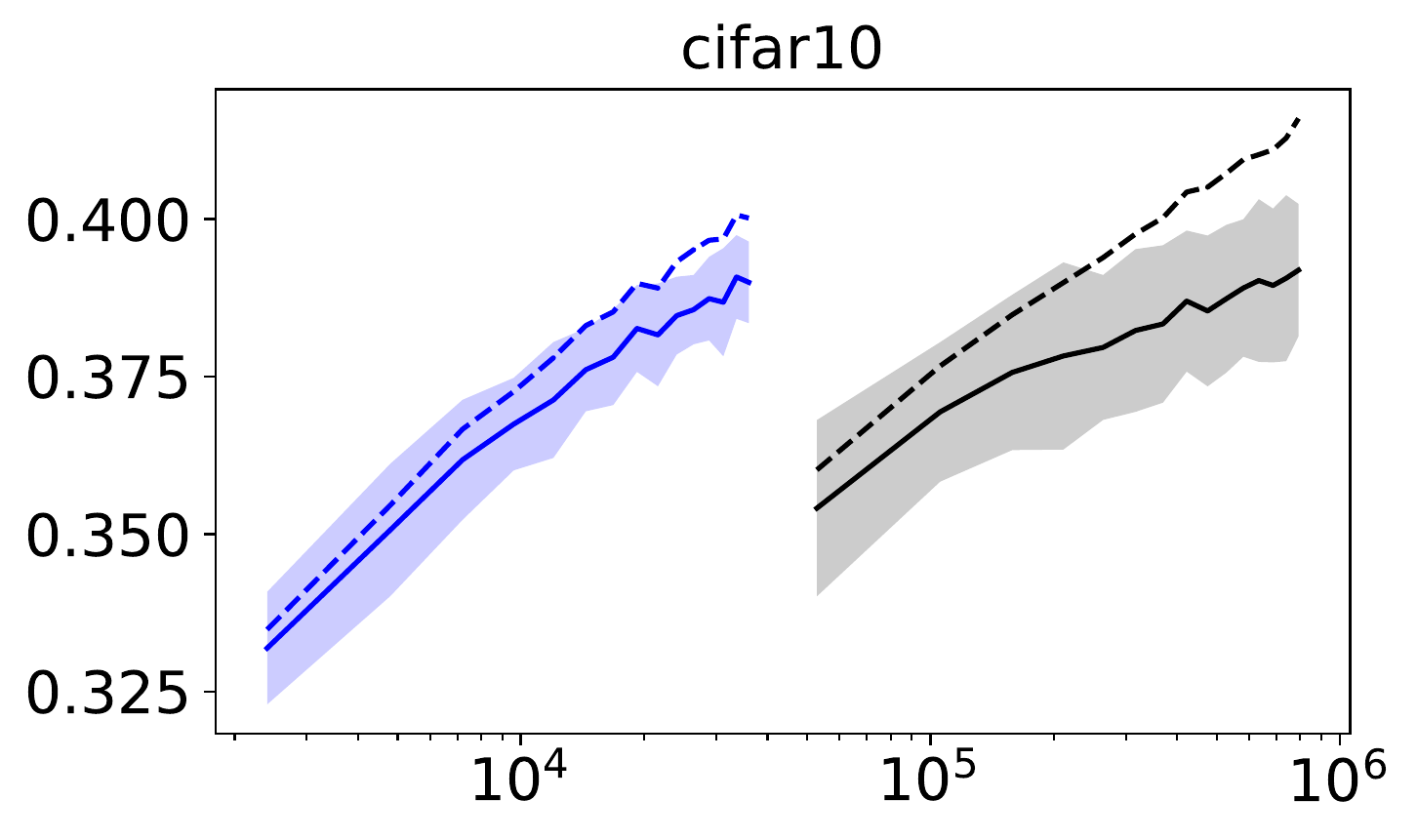}\\
\includegraphics[width=0.45\textwidth]{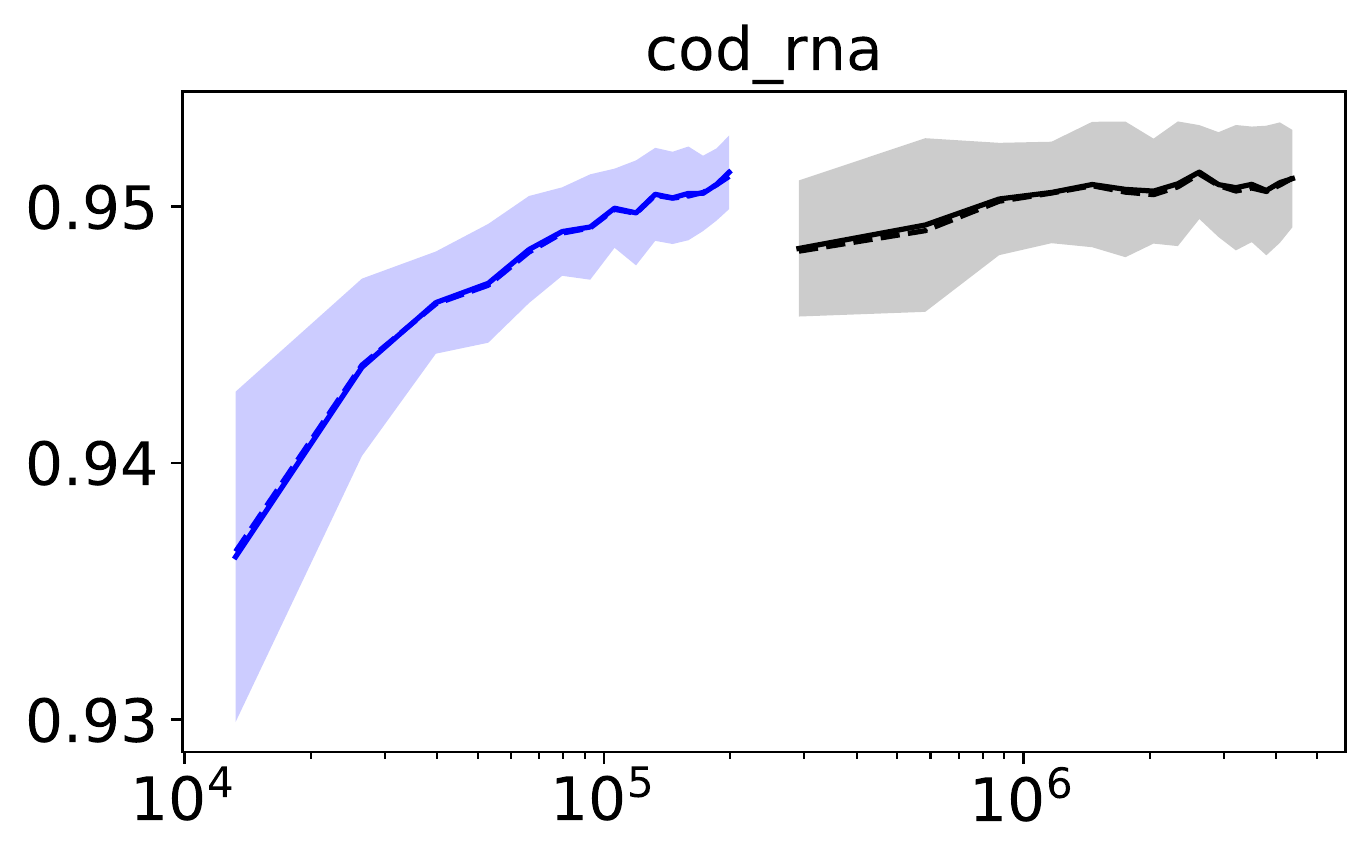}\,\includegraphics[width=0.45\textwidth]{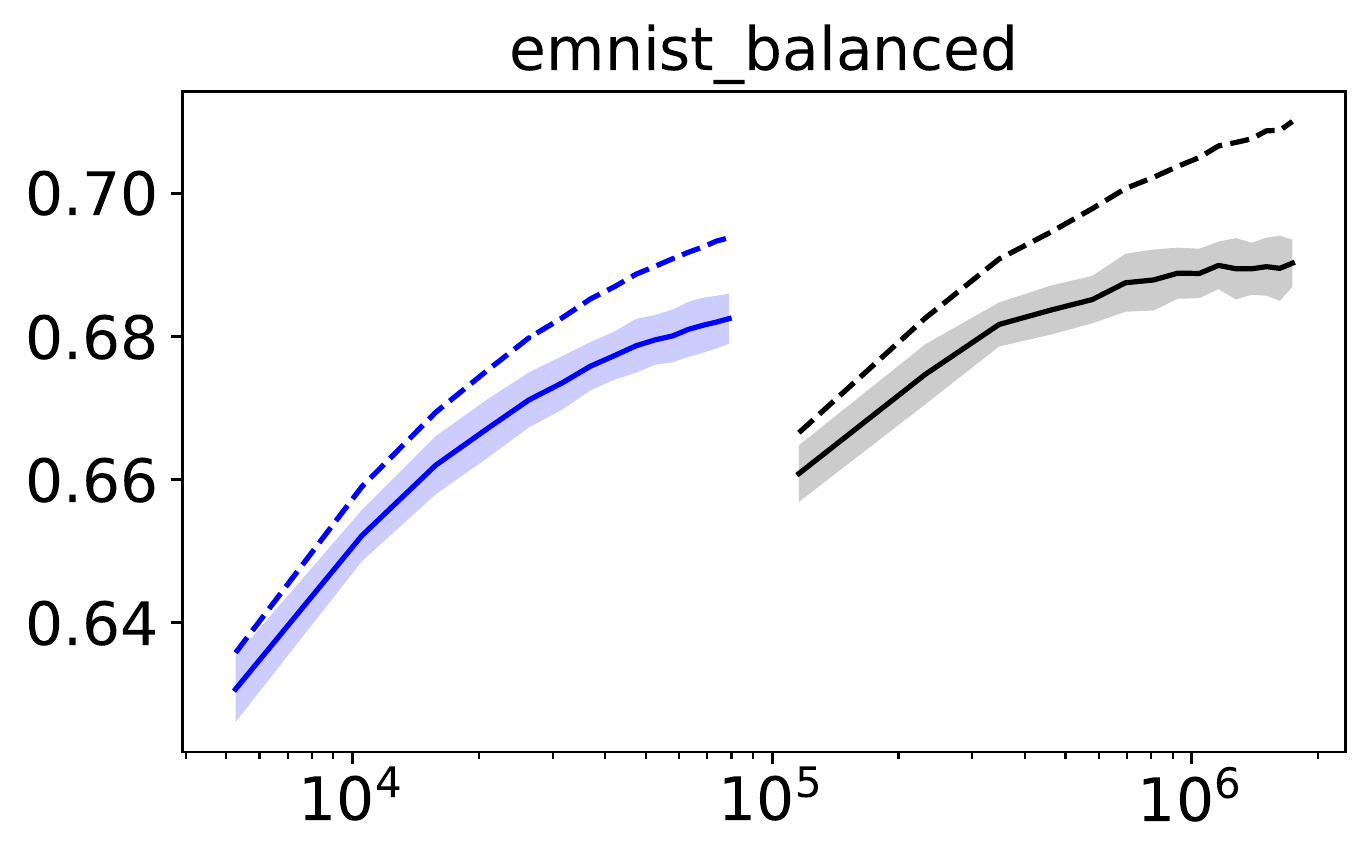}\\
\includegraphics[width=0.45\textwidth]{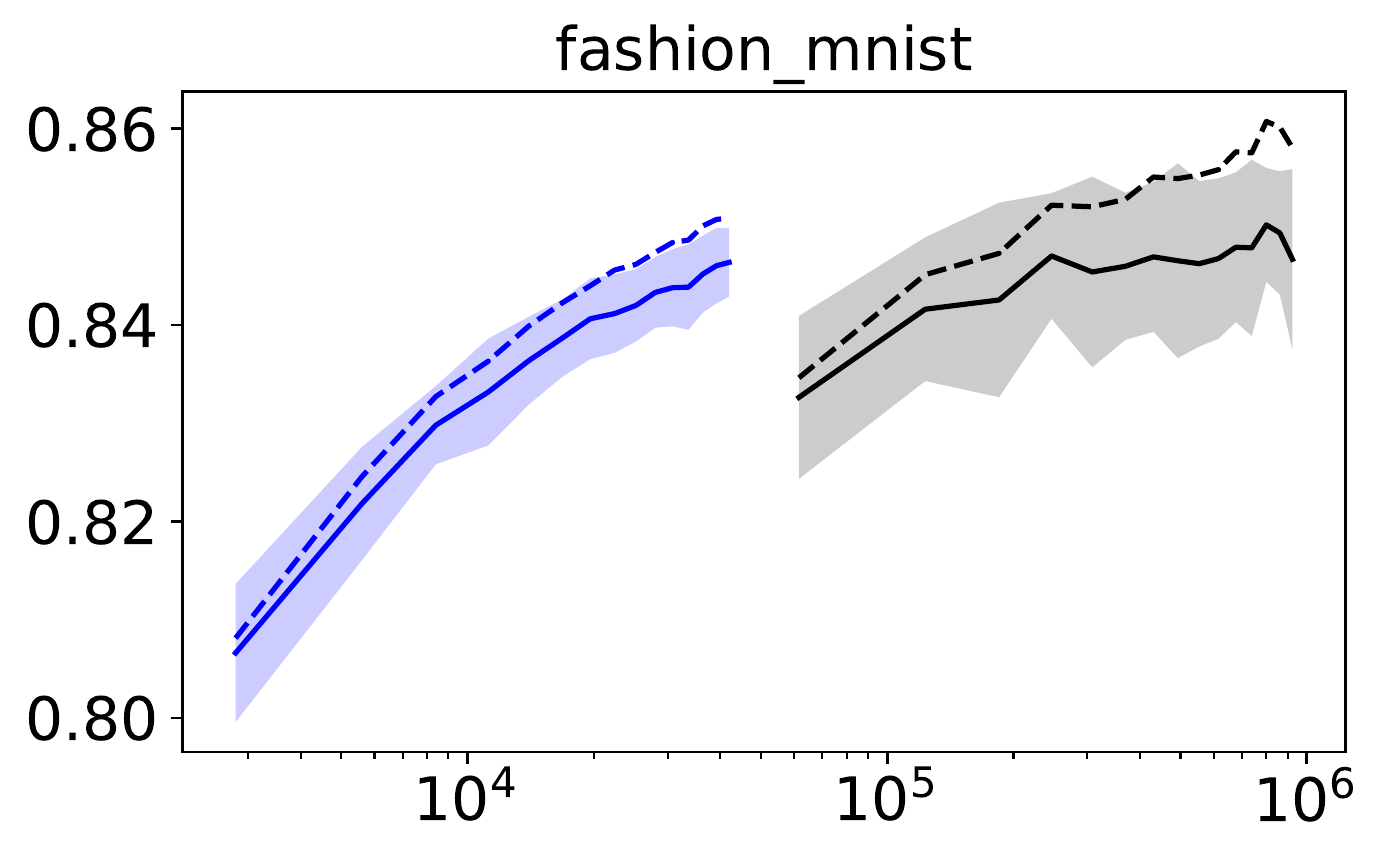}\,\includegraphics[width=0.45\textwidth]{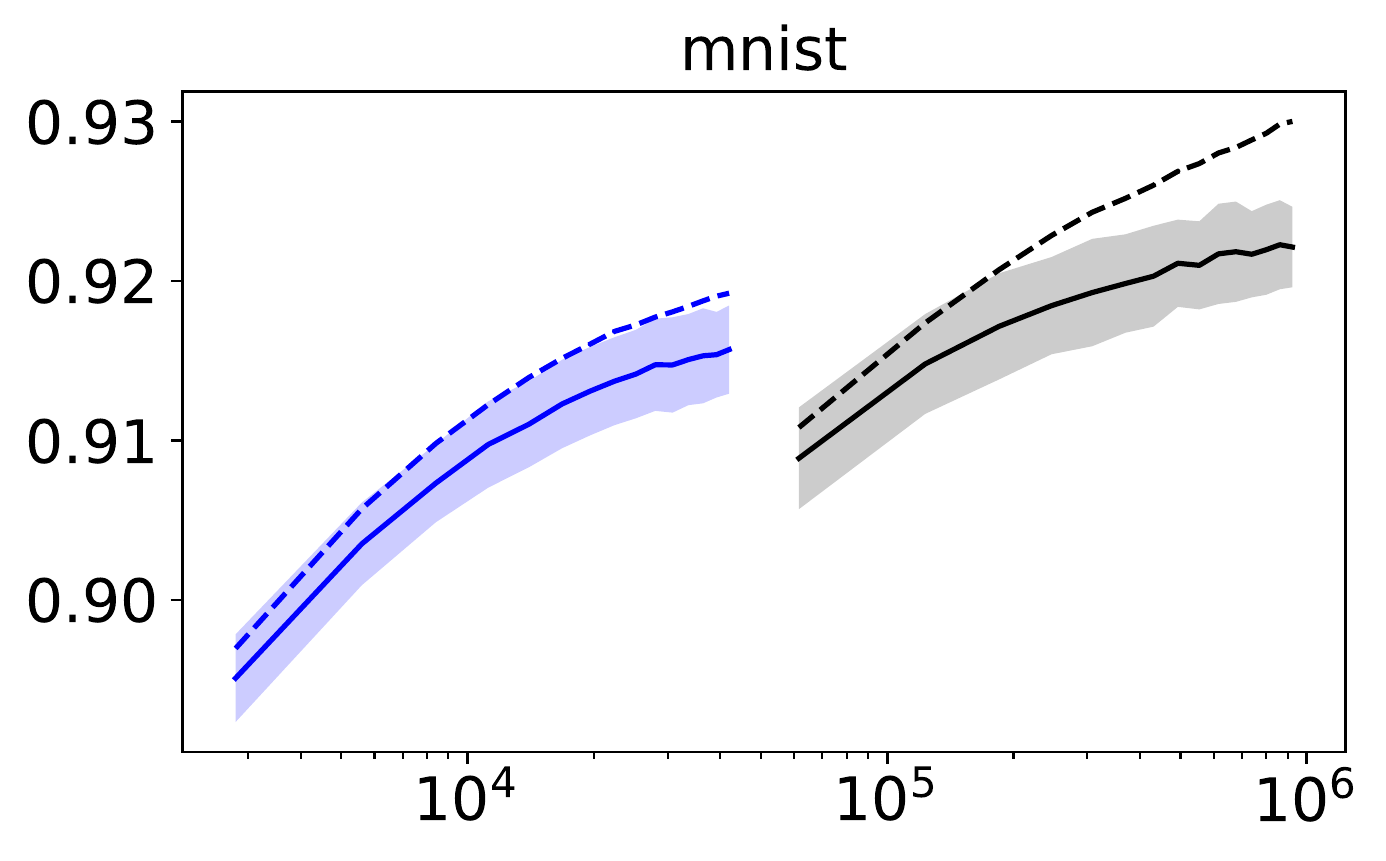}\\
\caption{For each dataset, we plot the mean of the training and test accuracy achieved at each epoch, as a function of the number of data points processed per sub-process. The shaded area is the mean $\pm$ standard deviation for the \emph{test} accuracy.}
\label{fig:acc_percore}
\end{figure}

\paragraph{(E2) How does introducing non-linearity into the model impact the performance of Algorithm \ref{algo:DandC_SGD} compared with Algorithm \ref{algo:DandC_valid}?}

For this point of inquiry, the basic setup is essentially the same as described for \textbf{(E1)}, except that we test non-linear models, which spoils the convexity of the underlying optimization task, and we also run Algorithm \ref{algo:DandC_valid}, denoted \texttt{RV-SGDAve}, in addition to \texttt{bench} and \texttt{DC-SGD}. As done in previous sub-sections, in implementing \texttt{RV-SGDAve} we use a Catoni-type M-estimator for $\valid$ (Algorithm \ref{algo:valid_cat}), and we pass it the validation data indexed by $\II_{\text{val}}$, recalling that \texttt{bench} receives both the training and validation sets together in a single batch, for fairness. Here we run $25$ independent trials for each model and each dataset. As described in the experimental setup exposition, the non-linear models we use are feed-forward neural networks, with the number of hidden layers ranging over $0,1,2,3$, respectively denoted \texttt{logistic}, \texttt{FF\_L1}, \texttt{FF\_L2}, and \texttt{FF\_L3}. In Figure \ref{fig:acc_bar}, we plot the mean (over all trials) of the test accuracy achieved by each method, for each model setting, under their respective strongest step size setting, just as with the previous figure. The immediate take-away is essentially exactly what we would expect; even with an extremely simple non-linearity introduced (moving from \texttt{logistic} $\to$ \texttt{FF\_L1}), the distance-based method of \texttt{DC-SGD} (Algorithm \ref{algo:DandC_SGD}) is critically impacted, falling almost immediately down to no better than chance level, whereas \texttt{RV-SGDAve} (Algorithm \ref{algo:DandC_valid}) is not severely impacted at all. While in an absolute sense the values achieved by \texttt{RV-SGDAve} here are slightly below the costly \texttt{bench}, recall that the per-core costs of \texttt{RV-SGDAve} are identical to \texttt{DV-SGD}, and thus the performance achieved here is done at a small fraction of the cost of \texttt{bench}.

\begin{figure}[t]
\centering
\includegraphics[width=0.4\textwidth]{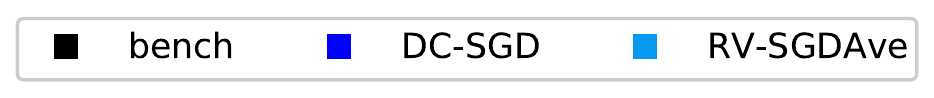}\\
\includegraphics[width=0.45\textwidth]{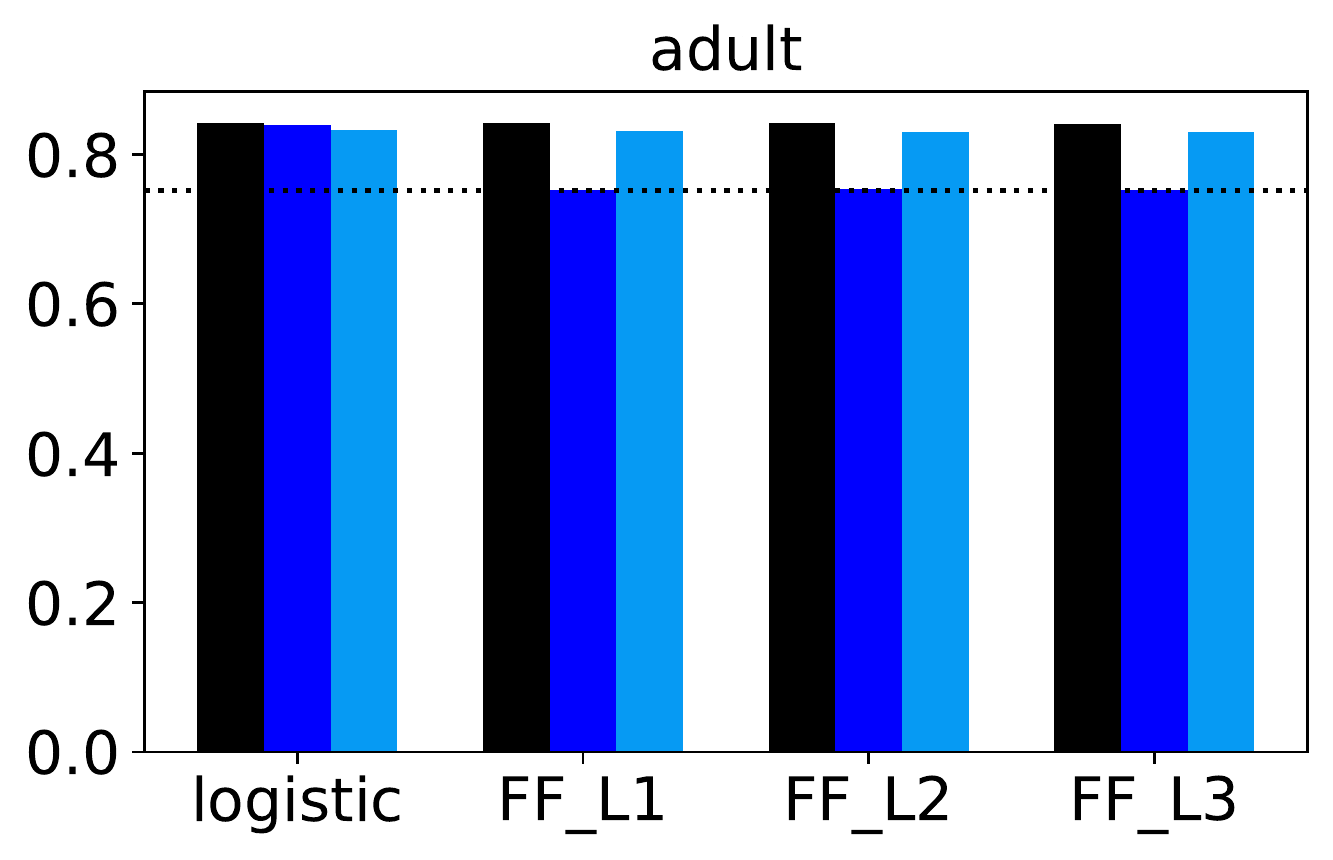}\,\includegraphics[width=0.45\textwidth]{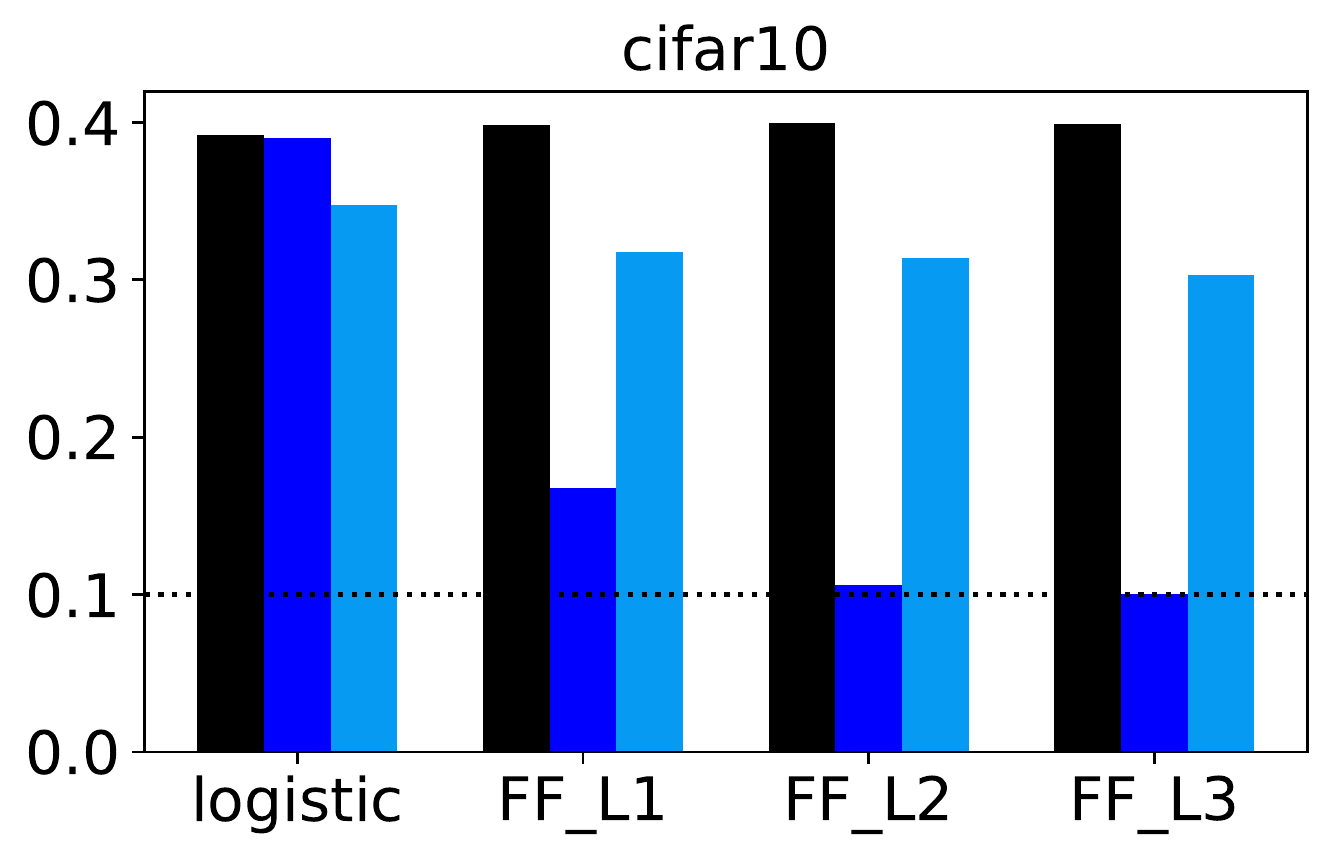}\\
\includegraphics[width=0.45\textwidth]{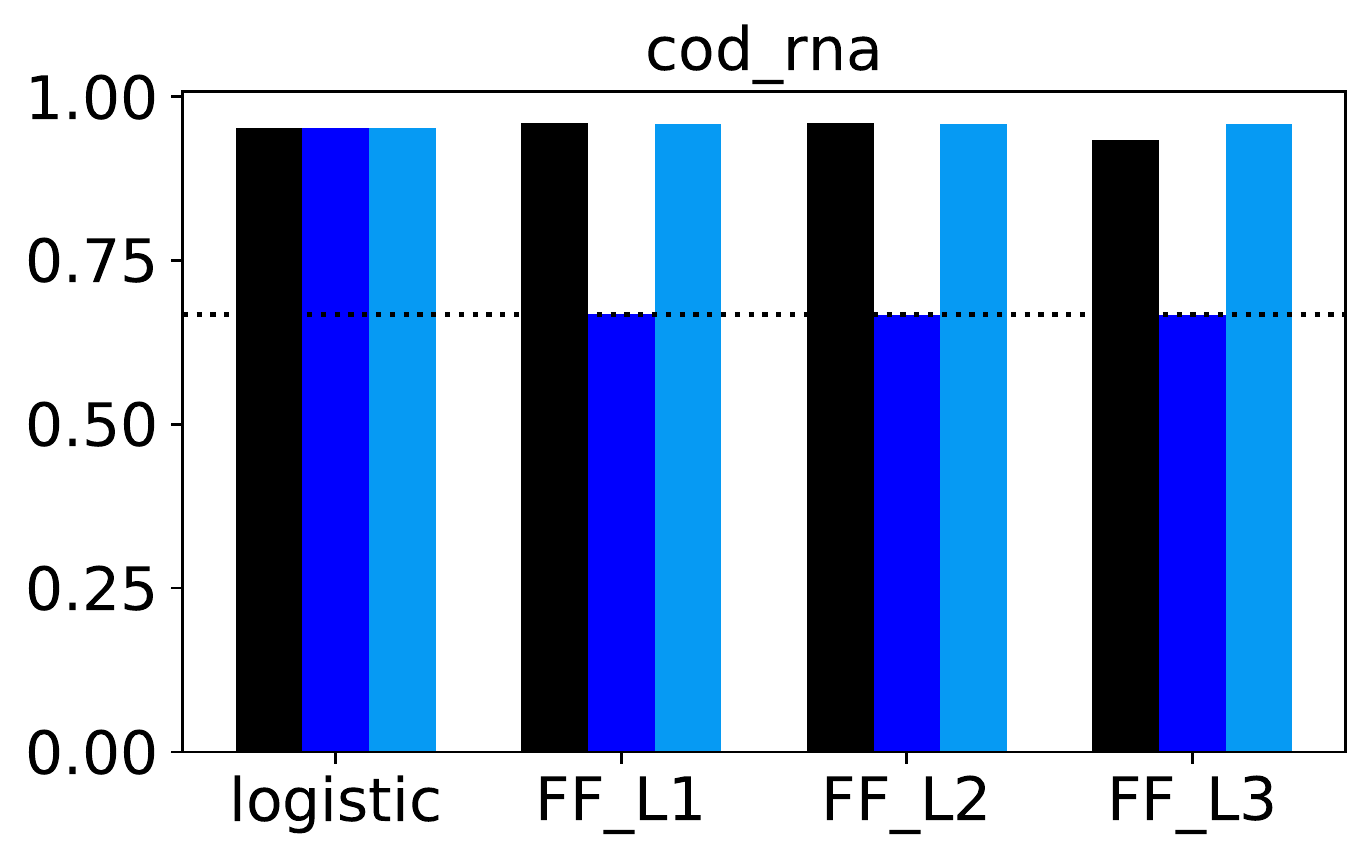}\,\includegraphics[width=0.45\textwidth]{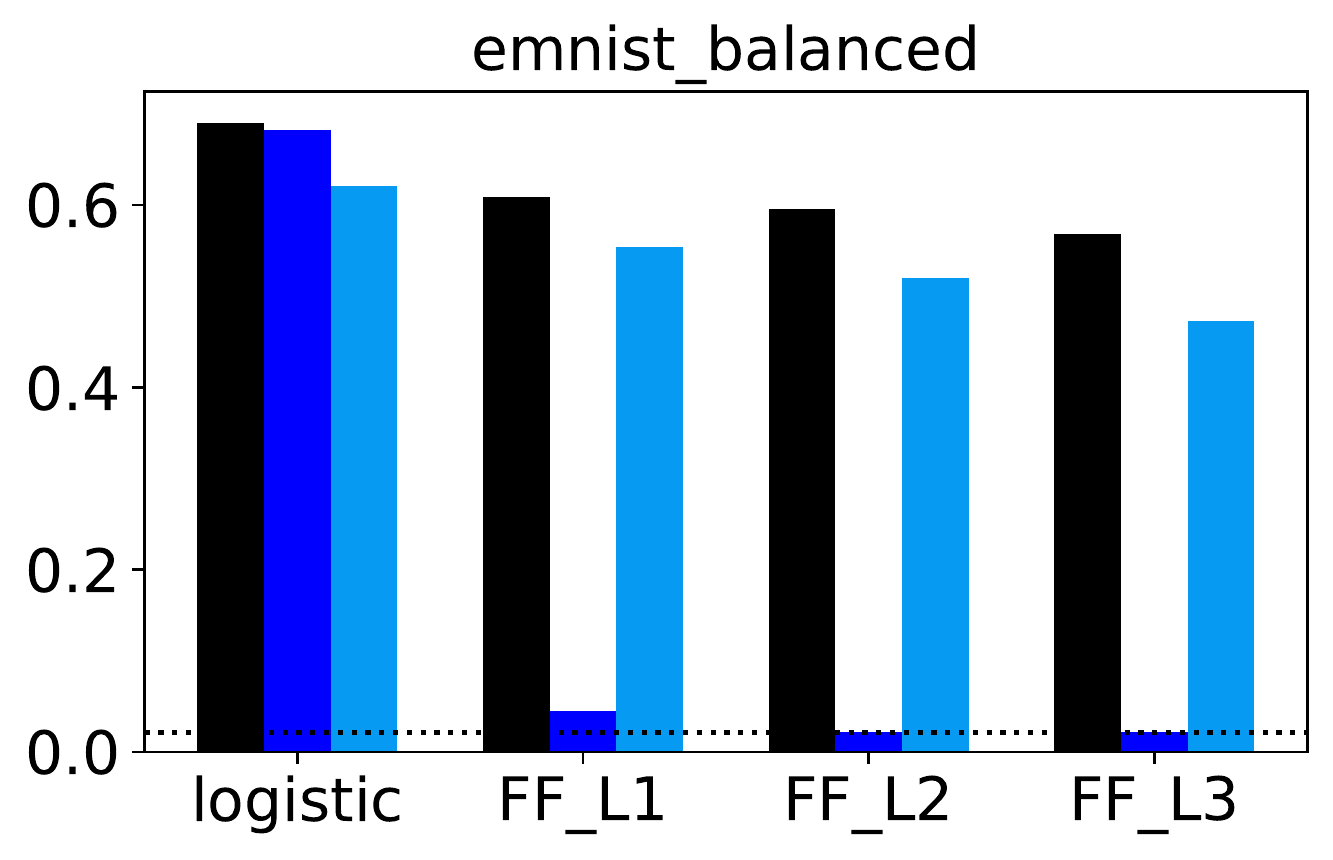}\\
\includegraphics[width=0.45\textwidth]{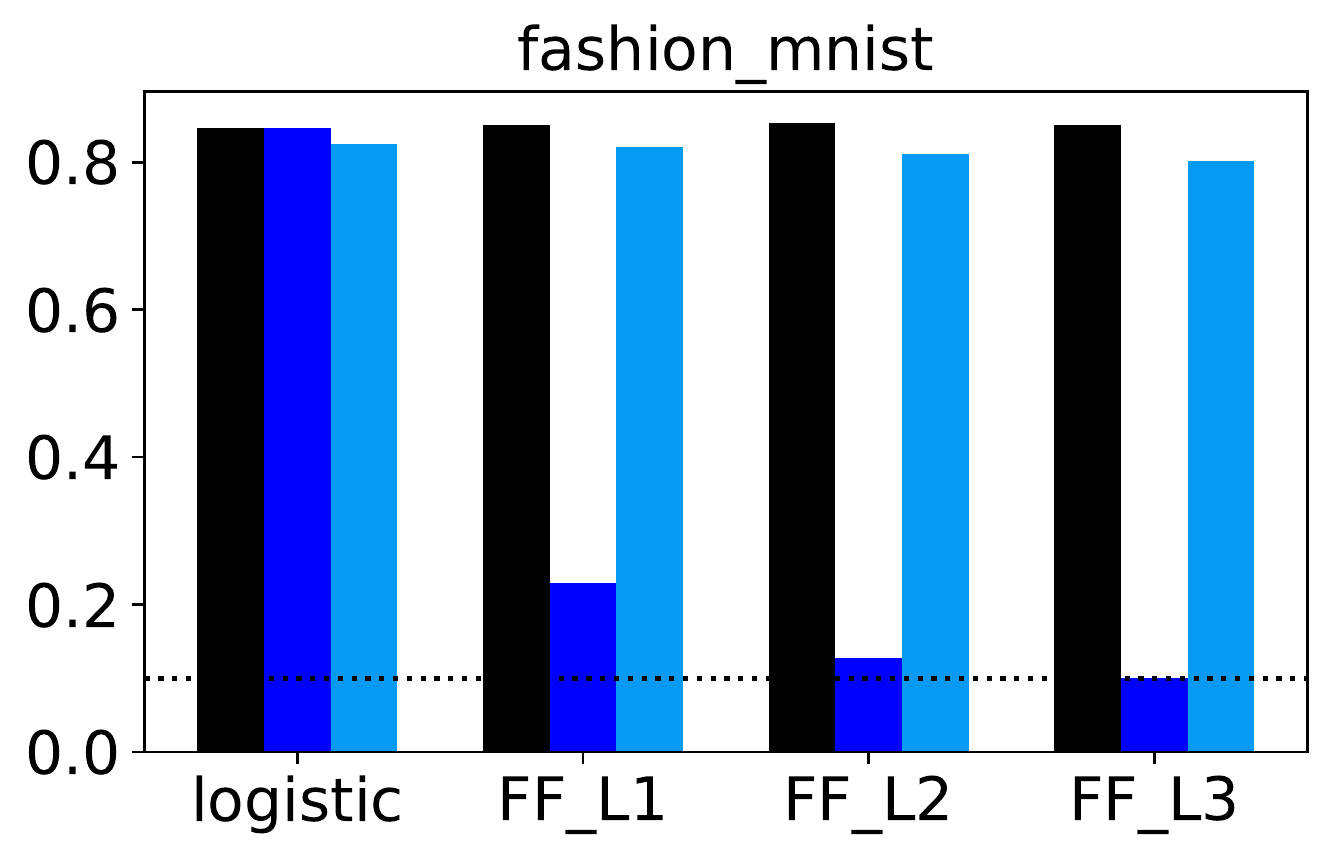}\,\includegraphics[width=0.45\textwidth]{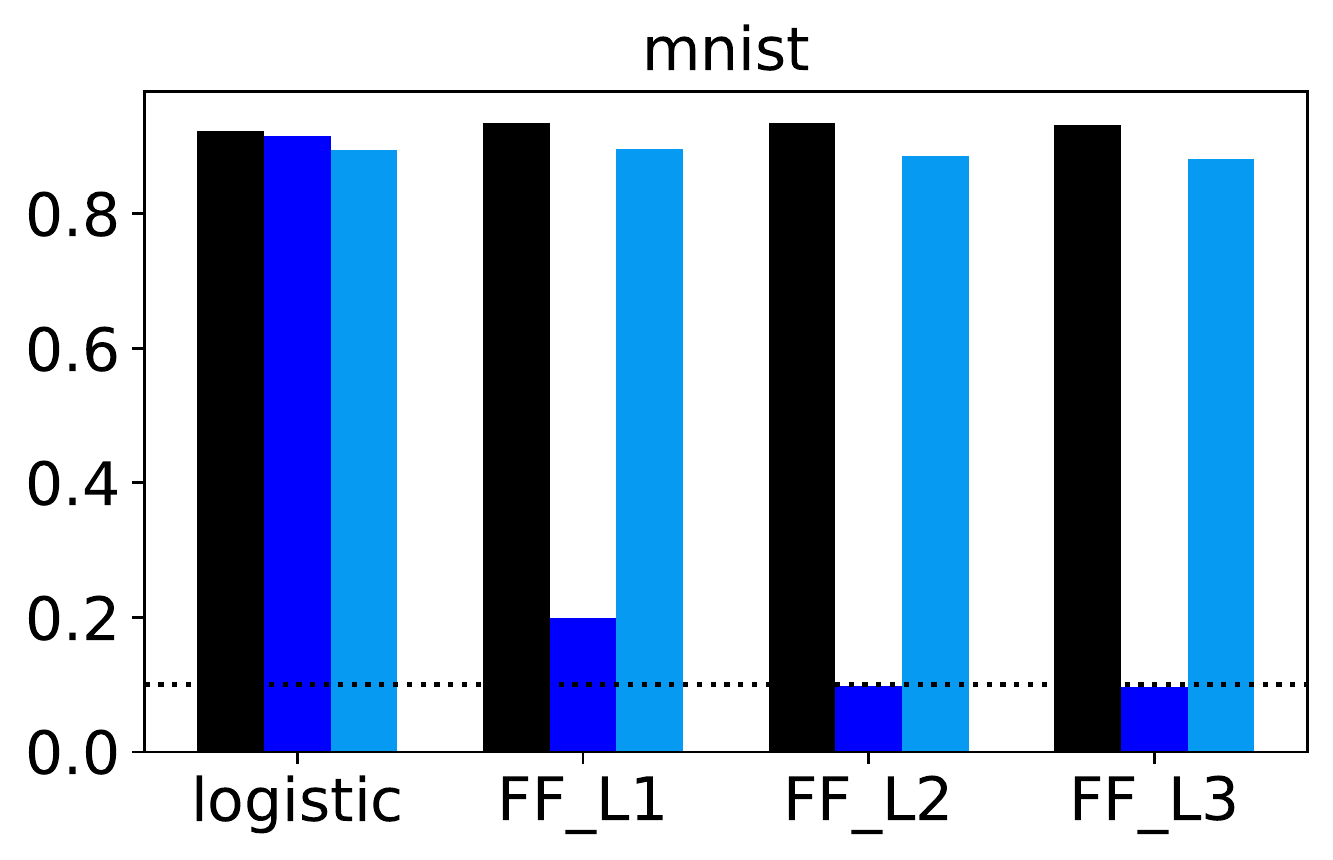}\\
\caption{For each dataset, we plot the mean of the test accuracy achieved after the final epoch for each method and each model. The dotted line represents ``chance level,'' and is set to the fraction of the dataset that belongs to the majority class.}
\label{fig:acc_bar}
\end{figure}

\clearpage

\section{Future directions}

As discussed in section \ref{sec:context_contribs}, this paper presents evidence, both formal and empirical, that general-purpose learning algorithms following the archetype drawn out in Algorithms \ref{algo:DandC_SGD} and \ref{algo:DandC_valid} should be able to improve significantly on the cost-performance tradeoff achieved by current state of the art robust gradient descent methods under potentially heavy-tailed losses and gradients. Furthermore, the latter algorithm allows us to achieve competitive efficiency and sability, without restrictive assumptions of strong convexity. Clearly, this work represents only a first step in this direction. Extending the theory to other algorithms besides vanilla SGD is a straightforward exercise; less straightforward is when we start considering a stage-wise strategy, when partition size $k$ can change from stage to stage. Extending results to allow for multiple passes is also of natural interest; the work of \citet{lin2017a} does this for the squared error, but without heavy tails. We only covered the $\ell_{2}$ norm case here, but extensions to cover other geometries (via stochastic mirror descent for example) are also of interest. In particular, if one considers a stochastic mirror descent type of generalization to the proposed algorithm, it would be interesting to compare the robust validation approach taken here with say the truncation-based approach studied recently by \citet{juditsky2019a}, and how the performance of the respective methods changes under different constraints on prior knowledge of the underlying data-generating distribution.

\appendix

\section{Technical supplement}

\subsection{Helper facts}

\begin{lem}\label{lem:lipschitz_dual_bound_characterization}
Let $f: \VV \to \RR$ and $\VV$ be convex. Then, $f$ is $\parasm$-Lipschitz with respect to norm $\|\cdot\|$ if and only if $\|u\|_{\star} \leq \parasm$ for all $u \in \partial f(v)$ and $v \in \VV$, where the dual norm $\|\cdot\|_{\star}$ is defined by $\|u\|_{\star} \defeq \sup \{\langle u,v \rangle: v \in \VV, \|v\| \leq 1\}$.
\end{lem}
\begin{proof}
See \citet[Lem.~2.6]{shalev2012a} for a proof.
\end{proof}

\paragraph{Strong convexity}

Let function $f: \VV \to \RR$ be continuous on closed convex set $C \subseteq \dom(f)$. We say that $f$ is \term{$\parasc$-strongly convex} on $C$ with respect to norm $\|\cdot\|$ if there exists $0 < \parasc < \infty$ such that for all $0 \leq \alpha \leq 1$ and $u, v \in C$ we have
\begin{align}\label{eqn:defn_sc_general}
f(\alpha u + (1-\alpha)v) \leq \alpha f(u) + (1-\alpha)f(v) - \frac{\parasc}{2} \alpha(1-\alpha)\|u-v\|^{2}.
\end{align}
The definition given in (\ref{eqn:defn_sc_general}) is intuitively appealing, as the third term on the right-hand side specifies how large the distance is from the graph of $f$ to the chord between $(u,f(u))$ and $(v,f(v))$. All else equal, when taking $\alpha$ over $[0,1]$, a larger $\parasc$ value means that the graph dips farther down below the chord. Furthermore, this definition is technically appealing in that it does not require $f$ to be differentiable. When $f$ is differentiable, then the more familiar characterizing property is that
\begin{align}\label{eqn:defn_sc_differentiable}
f(u)-f(v) \geq \langle \nabla f(v),u-v\rangle + \frac{\parasc}{2}\|u-v\|^{2}.
\end{align}
for all $u,v \in \VV$. Among other useful properties, (\ref{eqn:defn_sc_general}) implies that $f$ has a unique minimum $u^{\ast}$, and that for any $u \in \VV$,
\begin{align}\label{eqn:sc_prop_0}
f(u)-f(u^{\ast}) \geq \frac{\parasc}{2}\|u-u^{\ast}\|^{2}.
\end{align}
For a proof, see \citet[Lem.~13.5]{shalev2014a}. When $f$ is differentiable, then the following property characterizes $\parasc$-strong convexity \citep[Thm.~2.1.9]{nesterov2004ConvOpt}:
\begin{align}\label{eqn:sc_prop_1}
\langle u-v, \nabla f(u) - \nabla f(v) \rangle \geq \parasc\|u-v\|^{2}, \qquad \forall \, u,v \in \VV.
\end{align}

However, we shall be interested in settings where the objective may be $\parasc$-strongly convex, but not differentiable, and more general results of this nature are useful. For example, the standard defining property (\ref{eqn:defn_sc_differentiable}) holds in an analogous fashion for all sub-gradients, as the following lemma shows.
\begin{lem}\label{lem:sc_subgrad_key_ineq}
Let function $f$ be $\parasc$-strongly convex in the sense that it satisfies (\ref{eqn:defn_sc_general}). Then, for any $u,v \in C$, we have that
\begin{align*}
f(u)-f(v) \geq \langle g_{v}, u-v \rangle + \frac{\parasc}{2}\|u-v\|^{2}, \qquad \forall \, g_{v} \in \partial f(v).
\end{align*}
\end{lem}
\begin{proof}
See \citet[Lem.~13]{shalev2007PhD} for a concise proof.
\end{proof}
\noindent This fact is useful, since for arbitrary $u$ and $v$ we can immediately obtain bounds
\begin{align*}
\langle g_{v}, u-v \rangle & \leq f(u)-f(v) - \frac{\parasc}{2}\|u-v\|^{2}\\
\langle g_{u}, v-u \rangle & \leq f(v)-f(u) - \frac{\parasc}{2}\|u-v\|^{2}.
\end{align*}
Adding up both sides of these inequalities, the sums of the left-hand terms are smaller than the sums of the right-hand terms, immediately yielding
\begin{align}\label{eqn:sc_prop_1_subgrad}
\langle g_{u}-g_{v}, u-v \rangle \geq \parasc \|u-v\|^{2}, \qquad \forall \, g_{u} \in \partial f(u), g_{v} \in \partial f(v).
\end{align}
Clearly, when $f$ is differentiable, the inequality (\ref{eqn:sc_prop_1_subgrad}) reduces to that of (\ref{eqn:sc_prop_1}).

\paragraph{Smoothness}

Assuming $f:\VV \to \RR$ is differentiable, we say that $f$ is \term{$\parasm$-smooth} in norm $\|\cdot\|$ if its gradients are $\parasm$-Lipschitz continuous in the same norm, that is
\begin{align}\label{eqn:defn_smooth_general}
\|\nabla f(u) - \nabla f(v)\| \leq \parasm \|u-v\|
\end{align}
for all $u,v \in \VV$. \citet[Thm.~2.1.5]{nesterov2004ConvOpt} gives many useful characterizations of $\parasm$-smoothness. In particular, we shall utilize the fact that for all $u,v \in \VV$, we have
\begin{align}\label{eqn:smoothness_property}
0 \leq f(u)-f(v)-\langle \nabla f(v),u-v \rangle \leq \frac{\parasm}{2}\|u-v\|^{2}.
\end{align}

\subsection{Additional proofs for section \ref{sec:theory_sc}}\label{sec:theory_sc_appendix}

\begin{proof}[Proof of Lemma \ref{lem:merge_requirement}]
For $\geomed$, the result for any metric space is due to \citet[Thm.~28]{hsu2016a}, extending the key technical result of \citet[Lem.~2.1]{minsker2015a}.

For $\smball$, the desired result is easily shown to hold for any pseudometric $\|\cdot\|$ which satisfies the triangle inequality, as follows. For notational simplicity, write $U=\{u_{1},\ldots,u_{k}\}$, and for any $u \in \RR^{d}$, write $\diameter(u;\alpha) = \diameter(u;\alpha,U,\|\cdot\|)$. Furthermore, denote the ball centered at $u$ with radius $r > 0$ by $B(u;r) \defeq \{v: \|u-v\| \leq r\}$. In running $\smball$, the procedure looks at $\diameter_{j} = \diameter(u_{j};\alpha)$ for each $j \in [k]$ and chooses $u_{\star}$ with the smallest radius. Fixing any $u_{i} \in U \cap B(u;\diameter(u;\alpha))$, then for any choice of $u_{j} \in U \cap B(u;\diameter(u;\alpha))$, by definition of this intersection we have
\begin{align*}
\|u_{i} - u_{j}\| \leq \|u_{i} - u\| + \|u - u_{j}\| \leq 2\diameter(u;\alpha).
\end{align*}
Since $|U \cap B(u;\diameter(u;\alpha))| \geq |U|(\alpha+1/2)$, it immediately follows that there are at least $|U|(\alpha+1/2)$ points in $U$ that are $2\diameter(u;\alpha)$-close to this $u_{i}$, which implies $\diameter_{i} \leq 2\diameter(u;\alpha)$. By optimality, we have $\diameter_{\star} \leq \diameter_{i} \leq 2\diameter(u;\alpha)$. While $u_{\star}$ need not be $\diameter(u;\alpha)$-close to $u$, note that by the pigeonhole principle, $U \cap B(u;\diameter(u;\alpha)) \cap B(u_{\star};2\diameter(u;\alpha)) \neq \emptyset$. Let $u^{\prime}$ be any point in this intersection. Then
\begin{align*}
\|u_{\star}-u^{\prime}\| \leq \|u_{\star}-u\| + \|u-u^{\prime}\| \leq 2\diameter(u;\alpha) + \diameter(u;\alpha) = 3\diameter(u;\alpha).
\end{align*}
We thus conclude that $\smball$ satisfies (\ref{eqn:merge_requirement}) with $c_{\alpha} = 3$ for all $\alpha$ values.

Finally, for $\median$, one simply observes that it is a special case of the geometric median on a metric space equipped with the $\ell_{1}$ norm. To see this, note that the objective in the definition of $\geomed$ for the $\ell_{1}$ case is
\begin{align*}
\sum_{j=1}^{k} \|v - u_{j}\|_{1} = \sum_{j=1}^{k} \sum_{l=1}^{d} |v_{l} - u_{j,l}| = \sum_{l=1}^{d} \sum_{j=1}^{k} |v_{l} - u_{j,l}|.
\end{align*}
That is, it can be written as a sum of $d$ sums of absolute deviations. Since for each $j \in [k]$, we have that the sum $\sum_{j=1}^{k} |v_{l} - u_{j,l}|$ is minimized at $v_{l}=\widehat{u}_{j} \defeq \med\{u_{1,j},\ldots,u_{k,j}\}$, it follows that the vector of coordinate-wise medians $(\widehat{u}_{1},\ldots,\widehat{u}_{d})$ minimizes the original objective as a function of $v$. Finally, since $\|u\|_{1} \leq \sqrt{d}\|u\|_{2}$ for any $u \in \RR^{d}$, it follows that $\Delta(u;\|\cdot\|_{1}) \leq \sqrt{d}\Delta(u;\gamma,\|\cdot\|_{2})$.
\end{proof}

\begin{proof}[Proof of Lemma \ref{lem:boost_basic_prop}]
This basic statistical principle is well known; see \citet[Thm.~3.1]{minsker2015a} and \citet[Lem.~10]{hsu2016a} for similar results. For readability we write $a_{i}=a_{i}(\varepsilon)$, and denote by $a$ an independent copy of the $a_{i}$. Basic manipulations given us
\begin{align*}
\prr\left\{ \sum_{i=1}^{k} a_{i} > k\left(\frac{1}{2}+\gamma\right) \right\} & = \prr\left\{ \sum_{i=1}^{k} (a_{i}-\exx{}a) > k\left(\frac{1}{2}+\gamma-\exx{}a\right) \right\}\\
& = \prr\left\{ \sum_{i=1}^{k} (a_{i}-\exx{}a) > k\left(\gamma+\delta_{\varepsilon}-\frac{1}{2}\right) \right\}\\
& = 1 - \prr\left\{ -\sum_{i=1}^{k} (a_{i}-\exx{}a) > k\left(\frac{1}{2}-\gamma-\delta_{\varepsilon}\right) \right\}\\
& \geq 1 - \exp\left(-\frac{2k^{2}\left(\frac{1}{2}-\gamma-\delta_{\varepsilon}\right)^{2}}{\sum_{i=1}^{k}(0-1)^{2}}\right)
\end{align*}
from which the desired result follows. The final equality makes use of a one-sided Hoeffding inequality \citep[Thm.~2.8]{boucheron2013a}, noting that $-(a_{i}-\exx{}a)=((-1)a_{i}-\exx{}(-1)a_{i})$ and $-a_{i} \in [-1,0]$.
\end{proof}

\begin{proof}[Proof of Lemma \ref{lem:boost_conf_sc}]
This result follows quite directly from the facts laid out prior to its statement. Using property (\ref{eqn:smoothness_property}) of smooth functions, and Lemma \ref{lem:merge_requirement}, it immediately follows that
\begin{align}
\nonumber
\risk_{\ddist}(\what_{\textsc{new}})-\risk_{\ddist}^{\ast} & \leq \frac{\parasm_{1}}{2}\|\what_{\textsc{new}}-w^{\ast}\|^{2}\\
\label{eqn:boost_conf_sc_1}
& \leq \frac{\parasm_{1}c_{\gamma}^{2}}{2}\left( \diameter(w^{\ast};\alpha,\what_{\textsc{old}}^{(1)},\ldots,\what_{\textsc{old}}^{(k)}) \right)^{2}
\end{align}
for any choice of $0 < \gamma < 1/2$. It remains for us to control the radius $\diameter(w^{\ast};\gamma)$ on a high-probability event. To do this, we must control the distance between the base candidates and the minimum $w^{\ast}$. Without strict convexity (implied by strong convexity), one can never say in general that all of the candidates $\what_{\textsc{old}}^{(j)}$ are close to the \emph{same} point, despite achieving small excess risk. Fortunately, under $\parasc$-strong convexity, via (\ref{eqn:sc_prop_0}), for any $\delta_{0} \in (0,1)$ and each $j \in [k]$, we can say that
\begin{align*}
\frac{\parasc}{2}\|\what_{\textsc{old}}^{(j)} - w^{\ast}\|^{2} \leq \risk_{\ddist}(\what_{\textsc{old}}^{(j)})-\risk_{\ddist}^{\ast} \leq \frac{\varepsilon_{\ddist}(\lfloor n/k \rfloor)}{\delta_{0}}
\end{align*}
with probability no less than $1-\delta_{0}$. Note that the sample size $\lfloor n/k \rfloor$ comes from partitioning the $n$-sized sample and feeding equal-sized subsets to obtain each of the $k$ candidates. Cleaning this up, for each $j \in [k]$, we have
\begin{align}\label{eqn:boost_conf_sc_2}
\prr\left\{  \|\what_{\textsc{old}}^{(j)} - w^{\ast}\| > \sqrt{\frac{2\varepsilon_{\ddist}(\lfloor n/k \rfloor)}{\parasc \delta_{0}}} \right\} \leq \delta_{0}.
\end{align}
Recalling Lemma \ref{lem:boost_basic_prop}, direct application using the $k$ events defined (for $j=1,\ldots,k$) by the left-hand side of (\ref{eqn:boost_conf_sc_2}), we have that
\begin{align*}
\prr\left\{ \diameter(w^{\ast};\gamma) \leq \sqrt{\frac{2\varepsilon_{\ddist}(\lfloor n/k \rfloor)}{\parasc \delta_{0}}} \right\} \geq 1 - \exp\left( -2k\left(\gamma+\delta_{0}-\frac{1}{2}\right)^{2} \right),
\end{align*}
as long as $0 < \gamma < 1/2 - \delta_{0}$. As a concrete example, set $\delta_{0} = 1/4$. From this, we obtain
\begin{align*}
\prr\left\{ \diameter(w^{\ast};\gamma) \leq \sqrt{\frac{8}{\parasc} \varepsilon_{\ddist}\left(\left\lfloor \frac{n}{k_{s}} \right\rfloor\right)} \right\} \geq 1 - \ct{e}^{-s},
\end{align*}
and the number of partitions is $k_{s} = \lceil 8s/(1-\gamma)^{2} \rceil$. Plugging this into (\ref{eqn:boost_conf_sc_1}), we have
\begin{align*}
\risk_{\ddist}(\what)-\risk_{\ddist}^{\ast} \leq \frac{4\parasm_{1}c_{\gamma}^{2}}{\parasc}\varepsilon_{\ddist}\left(\left\lfloor \frac{n}{k_{s}} \right\rfloor\right)
\end{align*}
with probability no less than $1-\ct{e}^{-s}$. To obtain a $1-\delta$ probability bound, simply set $s = \log(\delta^{-1})$. The desired result assumes $k_{s}$ divides $n$ for simplicity.
\end{proof}

\begin{thm}[Strongly convex, Lipschitz, smooth; last iterate]\label{thm:learn_sc_lip_smooth_SGDlast}
Let $\text{\ref{asmp:lip_loss}}(\parasm_{0})$, $\text{\ref{asmp:sc_risk}}(\parasc)$, and $\text{\ref{asmp:sm_risk}}(\parasm_{1})$ hold in the $\ell_{2}$ norm. Use $\SGD$ as specified by (\ref{eqn:sgd_defn}), with update directions $G_{t} = \nabla\loss(\what_{t};Z_{t})$ at each step $t=0,1,\ldots,n-1$, with step size $\alpha_{t} = 1/(\parasc \max\{1,t\})$. With probability at least $1-\delta$, we have
\begin{align*}
\risk_{\ddist}(\what_{n})-\risk_{\ddist}^{\ast} \leq \frac{\parasm_{1}}{n} \left(\frac{\parasm_{0}}{\parasc}\right)^{2} \left(\frac{1}{\delta}\right).
\end{align*}
\end{thm}
\begin{proof}[Proof of Theorem \ref{thm:learn_sc_lip_smooth_SGDlast}]
We begin with the well-known inequality
\begin{align}
\label{eqn:learn_sc_lip_smooth_SGDlast_1}
\exx\|\what_{t+1}-\wstar\|^{2}_{2} \leq \exx\|\what_{t}-\wstar\|^{2}_{2} + \alpha_{t}^{2}\exx\|G_{t}\|^{2}_{2} - 2 \alpha_{t} \exx\langle \what_{t}-\wstar,\nabla\risk_{\ddist}(\what_{t})\rangle,
\end{align}
valid for any $t=0,1,\ldots,n-1$; see for example \citet[Sec.~2.1]{nemirovski2009a} or \citet[Thm.~2]{rakhlin2012a}. To control the second term on the right-hand side, using Lipschitz continuity and Lemma \ref{lem:lipschitz_dual_bound_characterization},
\begin{align}\label{eqn:learn_sc_lip_smooth_SGDlast_2}
\|G_{t}\|^{2}_{2} = \|\nabla\loss(\what_{t};Z_{t})\|^{2}_{2} \leq \parasm_{0}^{2}.
\end{align}
This implies $\exx\|G_{t}\|_{2}^{2} \leq \parasm_{0}^{2}$. To deal with the inner product, by first-order optimality conditions \citep[Prop.~3.1.4]{bertsekas2015ConvexOpt}, we have
\begin{align*}
\langle w-\wstar, \nabla\risk_{\ddist}(\wstar) \rangle \geq 0, \qquad \forall w \in \WW.
\end{align*}
As such, it follows that
\begin{align}
\nonumber
\langle \what_{t}-\wstar, \nabla\risk_{\ddist}(\what_{t}) \rangle & \geq \langle \what_{t}-\wstar, \nabla\risk_{\ddist}(\what_{t}) \rangle - \langle \what_{t}-\wstar, \nabla\risk_{\ddist}(\wstar) \rangle\\
\nonumber
& = \langle \what_{t}-\wstar, \nabla\risk_{\ddist}(\what_{t})-\nabla\risk_{\ddist}(\wstar) \rangle\\
\label{eqn:learn_sc_lip_smooth_SGDlast_3}
& \geq \parasc\|\what_{t}-\wstar\|^{2}_{2}
\end{align}
where the final inequality follows from the strong convexity of $\risk_{\ddist}$, namely property (\ref{eqn:sc_prop_1}). Taking (\ref{eqn:learn_sc_lip_smooth_SGDlast_2}) and (\ref{eqn:learn_sc_lip_smooth_SGDlast_3}) together, we can control the expected squared deviations as
\begin{align}
\nonumber
\exx \|\what_{t+1}-\wstar\|^{2}_{2} & \leq \exx \|\what_{t}-\wstar\|^{2}_{2} + \alpha_{t}^{2}\parasm_{0}^{2} - 2\parasc\alpha_{t} \exx\|\what_{t}-\wstar\|^{2}_{2}\\
\nonumber
& = \left(1-2\parasc\alpha_{t}\right)\exx\|\what_{t}-\wstar\|^{2}_{2} + \alpha_{t}^{2}\parasm_{0}^{2}\\
\label{eqn:learn_sc_lip_smooth_SGDlast_4}
& = \left(1-\frac{2}{\max\{1,t\}}\right)\exx\|\what_{t}-\wstar\|^{2}_{2} + \frac{\parasm_{0}^{2}}{\parasc^{2}\max\{1,t^{2}\}},
\end{align}
again emphasizing that expectation is with respect to the full sequence $(Z_{0},\ldots,Z_{n-1})$. This recursive inequality holds for any $t=0,1,\ldots,n-1$, and a simple induction argument leads to the desired result. For completeness, we spell this out explicitly. For the first two steps, note that by (\ref{eqn:learn_sc_lip_smooth_SGDlast_4}),
\begin{align*}
\exx \|\what_{1}-\wstar\|^{2}_{2} & \leq \left(1-\frac{2}{1}\right)\|\what_{0}-\wstar\|^{2}_{2} + \frac{\parasm_{0}^{2}}{\parasc^{2}} \leq \frac{\parasm_{0}^{2}}{\parasc^{2}}\\
\exx \|\what_{2}-\wstar\|^{2}_{2} & \leq \left(1-\frac{2}{1}\right)\exx\|\what_{1}-\wstar\|^{2}_{2} + \frac{\parasm_{0}^{2}}{\parasc^{2}} \leq \frac{\parasm_{0}^{2}}{\parasc^{2}}.
\end{align*}
In light of (\ref{eqn:learn_sc_lip_smooth_SGDlast_4}), we would like an upper bound on $\exx \|\what_{t}-\wstar\|^{2}_{2}$ that scales with $1/t$. For $\what_{1}$ we trivially have this, but for $\what_{2}$, we must pay the price of an extra factor of $2$; that is, the above two inequalities imply
\begin{align}
\label{eqn:learn_sc_lip_smooth_SGDlast_5}
\exx \|\what_{t}-\wstar\|^{2}_{2} \leq \frac{2}{t} \left(\frac{\parasm_{0}}{\parasc}\right)^{2}, \qquad t=1,2.
\end{align}
This follows immediately from direct computation. To proceed with an induction argument, simply assume that $\exx \|\what_{t}-\wstar\|^{2}_{2} \leq 2\parasm_{0}^{2}/(\parasc^{2}t)$ for some $t > 1$. Using (\ref{eqn:learn_sc_lip_smooth_SGDlast_4}), it follows that
\begin{align}
\nonumber
\exx \|\what_{t+1}-\wstar\|^{2}_{2} & \leq \left(1-\frac{2}{t}\right)\exx\|\what_{t}-\wstar\|^{2}_{2} + \frac{\parasm_{0}^{2}}{\parasc^{2}t^{2}}\\
\nonumber
& \leq \left(1-\frac{2}{t}\right)\frac{2\parasm_{0}^{2}}{\parasc^{2}t} + \frac{\parasm_{0}^{2}}{\parasc^{2}t^{2}}\\
\nonumber
& = \frac{\parasm_{0}^{2}}{\parasc^{2}} \left( \frac{2}{t}-\frac{3}{t^{2}} \right)\\
\label{eqn:learn_sc_lip_smooth_SGDlast_6}
& \leq \frac{2\parasm_{0}^{2}}{\parasc^{2}(t+1)},
\end{align}
where the final inequality follows from simple algebra. As such, by induction using (\ref{eqn:learn_sc_lip_smooth_SGDlast_5}), we may conclude that for any $T=1,\ldots,n$, we have
\begin{align}
\label{eqn:learn_sc_lip_smooth_SGDlast_7}
\exx \|\what_{T}-\wstar\|^{2}_{2} \leq \frac{2}{T} \left(\frac{\parasm_{0}}{\parasc}\right)^{2}.
\end{align}
All that remains is to link the control of the iterates to the control of risk value. This is where the smoothness assumption, and the assumption that $\wstar \in \inter(\WW)$ come in handy. If $\wstar$ is in the interior of $\WW$, it follows that $0 \in \partial\risk_{\ddist}(\wstar)$ \citep[Thm.~3.1.15]{nesterov2004ConvOpt}, and by differentiability we thus have $\nabla\risk_{\ddist}(\wstar) = 0$. Then by the key property (\ref{eqn:smoothness_property}) of smooth functions, it follows that for any $w \in \WW$ we have
\begin{align*}
\risk_{\ddist}(w)-\risk_{\ddist}^{\ast} & \leq \langle \nabla\risk_{\ddist}(\wstar),w-\wstar\rangle + \frac{\parasm_{1}}{2}\|w-\wstar\|^{2}_{2}\\
& = \frac{\parasm_{1}}{2}\|w-\wstar\|^{2}_{2}
\end{align*}
for any $w \in \WW$. Setting $w = \what_{n}$ and leveraging (\ref{eqn:learn_sc_lip_smooth_SGDlast_7}), we have
\begin{align*}
\exx\left[\risk_{\ddist}(\what_{n})-\risk_{\ddist}^{\ast}\right] \leq \frac{\parasm_{1}}{n} \left(\frac{\parasm_{0}}{\parasc}\right)^{2}.
\end{align*}
A direct application of Markov's inequality yields the desired result.
\end{proof}

\begin{thm}[Strongly convex, smooth; last iterate]\label{thm:learn_sc_smooth_SGDlast}
Let $\text{\ref{asmp:sc_risk}}(\parasc)$ and $\text{\ref{asmp:sm_loss}}(\parasm_{1})$ hold in the $\ell_{2}$ norm. Use $\SGD$ as specified by (\ref{eqn:sgd_defn}), with update directions $G_{t} = \nabla\loss(\what_{t};Z_{t})$ at each step $t=0,1,\ldots,n-1$, with initial step size $\alpha_{0}=1/(2\parasm_{1})$, and subsequent step sizes $\alpha_{t} = a/(\parasc n + b)$ for $t > 0$, with $b = 2a\parasm_{1}$, and $a>0$ set such that $\alpha_{t} \leq \alpha_{0}$ for all $t$. Taking sample size $n$ large enough that
\begin{align*}
n \geq M^{\ast} \defeq \frac{4\parasm_{1}}{\parasc}\left(\max\left\{ \frac{\parasm_{1}\parasc\|\what_{0}-\wstar\|^{2}}{\exx_{\ddist}\|G(\wstar;Z)\|^{2}}, 1\right\} - 1 \right),
\end{align*}
then with probability at least $1-\delta$, we have
\begin{align*}
\risk_{\ddist}(\what_{n})-\risk_{\ddist}^{\ast} \leq \frac{\exx_{\ddist}\|G(\wstar;Z)\|^{2}}{n-M^{\ast}+b} \left(\frac{1}{\delta}\right) \left(\frac{2a^{2}\parasm_{1}}{\parasc}\right).
\end{align*}
\end{thm}
\begin{proof}
See \citet{nguyen2018a} for a detailed proof. In particular, they show (their Lemma 1) that a $\parasm_{1}$-smoothness assumption on the \emph{losses} (i.e., assumption $\text{\ref{asmp:sm_loss}}(\parasm_{1})$) is sufficient to control the squared gradient norm as
\begin{align}\label{eqn:learn_sc_smooth_SGDlast_1}
\exx_{\ddist}\|G(w;Z)\|^{2} \leq A\left(\risk_{\ddist}(w)-\risk_{\ddist}^{\ast}\right) + B,
\end{align}
with constants $A=4\parasm_{1}$ and $B=2\exx_{\ddist}\|G(\wstar;Z)\|^{2}$. The general form (\ref{eqn:learn_sc_smooth_SGDlast_1}) is used as an \emph{assumption} in influential work by \citet{bottou2016a} for the convergence analysis of SGD, from which the general result can be extracted.
\end{proof}

\subsection{Additional proofs for section \ref{sec:theory_nonsc}}\label{sec:theory_nonsc_appendix}

\begin{proof}[Proof of Lemma \ref{lem:valid_basic_prop}]
For the median-of-means estimator $\mom$, see \citet[Sec.~4.1]{devroye2016a}, or \citet{hsu2016a} for a lucid proof. For the M-estimator $\cat$, simply apply \citet[Prop.~2.4]{catoni2012a}. For the truncated mean estimator, see \citet[Thm.~6]{lugosi2019b}.
\end{proof}

In the proof of Theorem \ref{thm:smooth_SGDave_roboost}, one key underlying result we rely on has to do with convergence rates of averaged SGD for smooth objectives, recalling that $\parasm$-smoothness of a function $f$ is defined via (\ref{eqn:defn_smooth_general}). The fact cited directly in the main text is summarized in the following theorem; it can be extracted readily from well-known properties of (stochastic) mirror descent, a family of algorithms dating back to \citet{nemirovsky1983a}.
\begin{thm}[Convex and smooth case; averaged]\label{thm:learn_conv_smooth_SGDave}
Let $\risk_{\ddist}$ be $\parasm_{1}$-smooth in the $\ell_{2}$ norm. Furthermore, assume that $\exx_{\ddist}\|G(w;Z)-\nabla\risk_{\ddist}(w)\|^{2}_{2} \leq \sigma_{G,\ddist}^{2} < \infty$ for all $w \in \WW$. Run $\SGD$ (\ref{eqn:sgd_defn}) with step size $\alpha_{t} = 1/(\parasm_{1}+(1/c_{n}))$ for $n$ iterations, setting $c_{n} = \diameter/\sqrt{\sigma^{2}n}$, and take the average as
\begin{align*}
\what_{[n]} \defeq \frac{1}{n} \sum_{t=1}^{n} \what_{t-1}.
\end{align*}
We then have with probability no less than $1-\delta$ that
\begin{align*}
\risk_{\ddist}(\what_{[n]}) - \risk_{\ddist}^{\ast} \leq \frac{\diameter}{\delta}\left( \frac{\diameter\parasm_{1}}{2n} + \frac{\sigma_{G,\ddist}}{\sqrt{n}} \right).
\end{align*}
\end{thm}
\begin{proof}[Proof of Theorem \ref{thm:learn_conv_smooth_SGDave}]
To begin, we establish some extra terms and notation related to mirror descent. For any differentiable convex function $f: \VV \to \RR$, define the \term{Bregman divergence} induced by $f$ as
\begin{align}
D_{f}(u,v) \defeq f(u)-f(v) - \langle \nabla f(v), u-v \rangle.
\end{align}
In mirror descent, one utilizes Bregman divergences of a particular class of convex functions, often called ``mirror maps.'' Let $\WW_{0} \subseteq \RR^{d}$ be an open convex set containing including $\WW$ within its closure, and also let $\WW \cap \WW_{0} \neq \emptyset$. We denote arbitrary mirror maps on $\WW_{0}$ by $\Phi:\WW_{0} \to \RR$. Strictly speaking, to call $\Phi$ a \term{mirror map} on $\WW_{0}$ it is sufficient if $\Phi$ is strictly convex, differentiable, and that its gradient takes on all values (i.e., $\{\nabla\Phi(u): u \in \WW_{0}\} = \RR^{d}$) and diverges on the boundary of $\WW_{0}$; see \citet[Sec.~4]{bubeck2015a} and the references therein for more details. Bregman divergences induced by mirror maps, namely $D_{\Phi}: \WW_{0} \to \RR$, play an important role in mirror descent when constructing a projection map that takes up between the primal space $\WW$, and the space where we can leverage gradient information. The generic mirror descent procedure is as follows. Initializing at arbitrary $\what_{0} \in \WW \cap \WW_{0}$, we update as
\begin{align}\label{eqn:smd_defn}
\what_{t+1} = \argmin_{u \in \WW \cap \WW_{0}} \, \left[ \alpha_{t} \langle u, G(\what_{t};Z_{t}) \rangle + D_{\Phi}(u,\what_{t}) \right],
\end{align}
where the random gradient vector is such that $\exx_{\ddist}G(w;Z) \in \partial \risk_{\ddist}(w)$ for all $w \in \WW$, just as discussed after equation (\ref{eqn:sgd_defn}). The following result is useful \citep[Thm.~6.3]{bubeck2015a}:
\begin{lem}\label{lem:learn_conv_smooth_SMDave}
Assume $\exx_{\ddist}\|G(w;Z)-\nabla\risk_{\ddist}(w)\|_{\ast}^{2} \leq \sigma_{G,\ddist}^{2}$ for all $w \in \WW$, and that $\risk_{\ddist}$ is $\parasm_{1}$-smooth in norm $\|\cdot\|$. Write $r^{2} = \sup \{ \Phi(w)-\Phi(\what_{0}): w \in \WW \cap \WW_{0} \}$. Run stochastic mirror descent (\ref{eqn:smd_defn}) for $n$ iterations, using any mirror map $\Phi$ that is $1$-strongly convex on $\WW \cap \WW_{0}$ in norm $\|\cdot\|$, with step sizes $\alpha_{t} = 1 / (\parasm_{1} + 1/c_{n})$, using $c_{n} = \sqrt{2r^{2}/(n\sigma_{G,\ddist}^{2})}$. Under this setting, we have
\begin{align*}
\exx \left[ \risk_{\ddist}\left(\frac{1}{n}\sum_{i=1}^{n} \what_{i}\right) - \risk_{\ddist}^{\ast} \right] \leq \frac{r^{2}\parasm_{1}}{n} + \sqrt{\frac{2r^{2}\sigma_{G,\ddist}^{2}}{n}},
\end{align*}
where expectation is taken with respect to the entire sequence $(Z_{1},\ldots,Z_{n})$.
\end{lem}
\noindent In order to use Lemma \ref{lem:learn_conv_smooth_SMDave}, it is sufficient to show that $\SGD$ (\ref{eqn:sgd_defn}) is a special case of (\ref{eqn:smd_defn}). Letting $\WW_{0} = \RR^{d}$, and setting $\Phi(u) = \|u\|^{2}_{2}/2$, note that this is a valid mirror map, and strong convexity follows from noting that the Hessian of $\Phi$ in this special case is the identity matrix. The resulting Bregman divergence is $D_{\Phi}(u,v) = \|u-v\|^{2}_{2}/2$. Noting that for any $u, w \in \WW$ we have
\begin{align*}
\langle G(w;Z), u-w \rangle + \frac{1}{2\alpha}\|u-w\|_{2}^{2} = \frac{1}{2\alpha}\|u - (w-\alpha \, G(w;Z))\|_{2}^{2} - \frac{\alpha}{2}\|G(w;Z)\|_{2}^{2},
\end{align*}
it follows that the left-hand side over $\WW$ is minimized by setting $u = \proj_{\WW}(w-\alpha\,G(w;Z))$. Using this fact, it follows that
\begin{align*}
\what_{t+1} & = \argmin_{u \in \WW} \, \left[ \alpha_{t} \langle u, G(\what_{t};Z_{t}) \rangle + \frac{1}{2}\|u-\what_{t}\|_{2}^{2} \right]\\
& = \argmin_{u \in \WW} \, \left[ \langle u-\what_{t}, G(\what_{t};Z_{t}) \rangle + \frac{1}{2\alpha_{t}}\|u-\what_{t}\|_{2}^{2} \right]\\
& = \proj_{\WW} \left( \what_{t} - \alpha_{t} \, G(\what_{t};Z_{t}) \right),
\end{align*}
which is precisely the $\SGD$ update in (\ref{eqn:sgd_defn}). Since the dual norm $\|\cdot\|_{\ast}$ of the $\ell_{2}$ norm is once again the $\ell_{2}$ norm, all the other assumptions in Lemma \ref{lem:learn_conv_smooth_SMDave} clearly align with those in Theorem \ref{thm:learn_conv_smooth_SGDave}, which follows from a direct application of Markov's inequality to convert bounds in expectation to high-probability confidence intervals, and finally using the fact that $r^{2} \leq \diameter / \sqrt{2}$.
\end{proof}

\subsection{Additional proofs for section \ref{sec:empirical_sc}}\label{sec:empirical_sc_appendix}

\begin{proof}[Proof of quadratic form used in section \ref{sec:empirical_sc}]
Here we verify that given desired risk form of $\risk_{\ddist}(w) = \langle \Sigma w, w \rangle + \langle w, u \rangle + a$, when we define $\loss(w;Z) = (\langle w-\wstar, X \rangle + E)^{2}/2$ and assume that $X$ and $E$ are independent, we obtain $\exx_{\ddist}\loss(w;Z) = \risk_{\ddist}(w)$. First note that using the independence of $X$ and $E$, we have
\begin{align*}
\exx_{\ddist}\loss(w;Z) = \frac{1}{2}\exx_{\ddist}\left(\langle w-\wstar, X \rangle + E\right)^{2} = \frac{1}{2}\exx_{\ddist}\langle w-\wstar,X \rangle^{2} + \frac{1}{2}\exx_{\ddist}E^{2}.
\end{align*}
Then, noting that
\begin{align*}
\langle w-\wstar, X \rangle^{2} & = \langle w, X \rangle^{2} + \langle \wstar, X \rangle^{2} - 2\langle w, X \rangle\langle \wstar, X \rangle\\
& = \langle XX^{\trans}w, w \rangle + \langle XX^{\trans}\wstar, \wstar \rangle - 2\langle XX^{\trans}\wstar, w \rangle,
\end{align*}
using linearity of the integration and inner product operations, we have $\exx_{\ddist}\loss(w;Z) = \risk_{\ddist}(w)$ with $\Sigma = \exx_{\ddist}XX^{\trans}/2$, $u = -2\Sigma\wstar$, and $a = \langle \Sigma\wstar,\wstar \rangle + \exx_{\ddist}E^{2}/2$.
\end{proof}

\bibliographystyle{apalike}
\bibliography{refs/refs_arxiv}

\end{document}